\documentclass{article}

\usepackage{PRIMEarxiv}

\usepackage{ulem}
\usepackage{amsmath}
\usepackage{amsfonts}
\usepackage{amssymb}
\usepackage{amsthm}
\usepackage{makeidx}
\usepackage{color}
\usepackage{bbm}
\usepackage{mathrsfs}
\usepackage{mathtools}
\usepackage{bbm}
\usepackage{amsthm}
\usepackage{booktabs}
\usepackage{makecell}
\usepackage{caption}
\usepackage{subcaption}
\usepackage{float}
\usepackage{longtable}
\usepackage{pifont}
\usepackage{algorithm} 
\usepackage{derivative}
\usepackage{array}
\usepackage{tabularx}
\usepackage{booktabs}
\usepackage[labelfont=bf,format=plain,justification=centering,singlelinecheck=false]{caption}
\usepackage[thinc]{esdiff}
\usepackage{dcolumn}
\usepackage{tikz}
\usepackage{algpseudocode} 
\usepackage{float}
\usepackage{slashbox}

\DeclareMathOperator*{\argmax}{arg\,max}
\DeclareMathOperator*{\argmin}{arg\,min}

\newtheorem{assu}{Assumption}
\newtheorem{pro}{Proposition}
\newtheorem{defi}{Definition}
\newtheorem{coro}{Corollary}

\newtheorem{lem}{Lemma}

\usepackage[scr=rsfs,cal=boondox]{mathalfa}

\usepackage[abbrvbib, preprint]{jmlr2e}

\graphicspath{{figures/}}

\firstpageno{1}

\usepackage{diagbox}
\usepackage{makecell}

\newcommand{\bR}{\mathbb{R}}
\newcommand{\gcal}{\mathcal{g}}
\newcommand{\Gcal}{\mathcal{G}}

\newcommand{\dcal}{\mathcal{d}}
\newcommand{\mcal}{\mathcal{m}}

\newcommand{\wcal}{\mathcal{w}}
\newcommand{\norm}[1]{\Vert{#1}\Vert}

\def\rom#1{\uppercase\expandafter{\romannumeral#1\relax}}

\usepackage{amsmath,amsthm}
\newtheorem{assprime}{Assumption}

\usepackage{hyperref}
\hypersetup{
    colorlinks=true,
    linkcolor=blue,
    filecolor=magenta,      
    urlcolor=cyan,
    citecolor=blue
}

\usepackage[utf8]{inputenc} 
\usepackage[T1]{fontenc}    
\usepackage{hyperref}       
\usepackage{url}            
\usepackage{booktabs}       
\usepackage{amsfonts}       
\usepackage{nicefrac}       
\usepackage{microtype}      
\usepackage{lipsum}
\usepackage{fancyhdr}       
\usepackage{graphicx}       
\graphicspath{{media/}{figures/}}     

\title{The non-overlapping statistical approximation to overlapping group lasso}

\author{
  Mingyu \ Qi \\
  Department of Statistics \\
  University of Virginia\\
  Charlottesville, VA 22904-4135, USA\\
  \texttt{mq3sq@virginia.edu} \\
   \And
   Tianxi \ Li \\
   School of Statistics \\
    University of Minnesota, Twin Cities\\
   Minneapolis, MN 55455, USA\\
   \texttt{tianxili@umn.edu} \\
}

\begin{document}
\maketitle

\begin{abstract}
   The group lasso penalty is widely used to introduce structured sparsity in statistical learning, characterized by its ability to eliminate predefined groups of parameters automatically. However, when the groups are overlapping, solving the group lasso problem can be time-consuming in high-dimensional settings because of the non-separability induced by the groups. This difficulty has significantly limited the penalty's applicability in cutting-edge computational areas, such as gene pathway selection and graphical model estimation. This paper introduces a non-overlapping and separable penalty to efficiently approximate the overlapping group lasso penalty. The approximation substantially improves the computational efficiency in optimization, especially for large-scale and high-dimensional problems. We show that the proposed penalty is the tightest separable relaxation of the overlapping group lasso norm within the family of $\ell_{q_1}/\ell_{q_2}$ norms. Furthermore, the estimators based on our proposed norm are statistically equivalent to those derived from the overlapping group lasso in terms of estimation error, support recovery, and minimax rate, under the squared loss. The effectiveness of the method is demonstrated through extensive simulation examples and a predictive task of cancer tumors. 
 \end{abstract}

\keywords{overlapping group lasso \and separable approximation \and  computational efficiency \and statistical error bound \and support recovery \and high-dimensional regression}

\section{Introduction}

Grouping patterns of variables are commonly observed in real-world applications. For example, in regression modeling, explanatory variables might belong to different groups with the expectation that the variables are highly correlated within the groups. In this context, variable selection or model regularization should also consider the grouping patterns, and one may prefer to either include the whole group of variables in the selection or completely rule out the group. Group lasso \citep{yuan2006model} is one popular method designed for this group selection task via adding $\ell_1/\ell_2$ regularization, as a broader class for group selection \citep{glc,seol,tglf,ravikumar2009sparse,tcap,danaher2014joint,loh2014high,basu2015network,XIANG201528,campbell2017within,tank2017efficient,hvs,austin2020new,yang2020estimating}.

While the original group lasso penalty \citep{yuan2006model} focuses on regularizing disjoint parameter groups,  overlapping groups appear frequently in many applications such as tumor metastasis analysis \citep{ogl, tcap, ES, chen2012smoothing} and structured model seelction problems \citep{ mohan2014node, hmgm,yu2017learning, tarzanagh2018estimation}.  For example, in tumor metastasis analysis, scientists usually aim to select a small number of tumor-related genes. Biological theory indicates that rather than functioning in isolation, genes act in groups to perform biological functions. Hence, the gene selection is more meaningful if co-functioning groups of genes are selected together \citep{ma2010detection}. In particular, gene pathways, in the form of overlapping groups of genes, render mechanistic insights into the co-functioning pattern. Applying group lasso with these overlapping groups is then a natural way to incorporate the prior group information into tumor metastasis analysis. For another example, graphical models have been widely used to represent conditional dependency structures among variables. \citet{hmgm} developed a mixed graphical model for high-dimensional data with both continuous and discrete variables. In their model, the groups are naturally determined by groups of parameters corresponding to each edge, and these groups overlap because edges share common nodes. Selecting the graph structures under this class of models requires eliminating groups of parameters, which is achieved by the overlapping group lasso penalty.

The optimization involving the group lasso penalty with non-overlapping groups is efficient \citep{anotg,ebda, afua}. However, the overlapping group lasso problems present more complex challenges despite their convex nature. This is because the non-separability between groups intrinsically increases the problem's dimensionality compared with the non-overlapping situation, as revealed in the study of \cite{hvs}. Proposed methods for such optimization problems include the second-order cone program method SLasso \citep{svsw}, the ADMM-based methods  \citep{ADMM, ADM}, and their smoothed improvement, FoGLasso, introduced by \cite{ES}. Nevertheless, these exact solvers of the problem involve expensive gradient calculations when the overlapping becomes severe, which may limit the applicability of the overlapping group lasso penalty in many large-scale applications such as genomewide association studies  \citep{yang2010identifying,lee2012leveraging,ESS} or graphical model fitting problems \citep{hmgm}.  For instance, \citet{hmgm} showed that overlapping group lasso, though a natural choice for the problem, is infeasible even for moderate-size graphs, and they used a fast lasso approach \citep{tibshirani1996regression} to solve the graph estimation problem without theory. As we introduce later, our proposed solution includes the method of \citet{hmgm} as a special case, but our method is more general and comes with theoretical guarantees.
 
In this paper, we propose a non-overlapping approximation alternative to the overlapping group lasso penalty. The approximation is formulated as a weighted non-overlapping group lasso penalty that respects the original overlapping group patterns, making optimization significantly easier. The proposed penalty is shown to be the tightest separable relaxation of the original overlapping group lasso penalty within a broad family of penalties. Our analysis reveals that the estimator derived from our method is statistically equivalent to the original overlapping group lasso estimator in terms of estimator error and support recovery. The practical utility of our proposed method is exemplified through simulation examples and its application in a predictive task involving a breast cancer gene dataset. As a high-level summary, our major contribution to the paper is the design of a novel approximation penalty to the overlapping group lasso penalty, which enjoys substantially better computational efficiency in optimization while maintaining equivalent statistical properties as the original penalty.

The remainder of this paper is organized as follows: Section \ref{sec:methods} introduces the overlapping group lasso problem and the proposed approximation method. We also establish the optimality of the proposed penalty from the optimization perspective. Section \ref{sec:results} details the statistical properties of the penalized estimator based on the proposed penalty. Comparisons between our estimator and the original overlapping group lasso estimator are made to show that they are statistically equivalent with respect to estimation errors and variable selection performance. Empirical evaluations using simulated and real breast cancer gene expression data are presented in Sections \ref{sec:simu} and \ref{sec:data}, respectively. Finally, Section \ref{sec:disc} concludes the paper with additional discussions.
\section{Methodology}
\label{sec:methods}
\paragraph{Notation and Preliminaries.} Throughout this paper,  for an integer $z$, the notation $[z] $ is used to denote the index set $\{1, \cdots, z\} $. Given two sequences $\{a_n\} $ and $\{b_n\} $, we denote $a_n \lesssim b_n$ or $a_n = O(b_n)$ if $a_n \leqslant Cb_n $ for a sufficiently large $n $ and a universal constant $C > 0 $. We write $a_n\ll b_n$ or $a_n = o(b_n)$ if $a_n/b_n \to 0$. Furthermore, $a_n \asymp b_n $ if both $a_n \lesssim b_n $ and $a_n \gtrsim b_n $ hold. Given a set $T$,  $|T| $ represents the cardinality of $T$.  
When referring to a matrix $A $, $A_{T} $ denotes the sub-matrix consisting of columns indexed by $T $, and $A_{T,T} $ denotes the sub-matrix induced by both rows and columns indexed by $T $. Additionally, for a vector $x \in \mathbb{R}^p$, we define $\|x\|_a = \left(|x_1|^a + |x_2|^a + \ldots + |x_p|^a\right)^{\frac{1}{a}}$. Recall the operator norm definition: $\|A\|_{a, b} = \sup_{\|u\|_a \leq 1}\|A u\|_b $. When $A $ is a symmetric matrix, $\gamma_{\min}(A) $ and $\gamma_{\max}(A) $ denote its smallest and largest eigenvalues, respectively. We will introduce other notations within the text as needed. Table~\ref{tab:notations} in Appendix~\ref{sec:notations} lists all the notations in the paper.

\subsection{Overlapping Group Lasso}
Suppose in a statistical learning problem, the parameters are represented by a vector $\beta \in \mathbb{R}^p$, where $\beta_j$ denotes the $j$-th element of $\beta$. Let $G = \{G_1, \cdots, G_m\}$ be the $m$ predefined groups for the $p$ parameters, with each group $G_g$ being a subset of $[p]$, and $\cup_{g \in [m]} G_g = [p]$. For each group $G_g$, $d^G_g = |G_g|$ denoted the group size, and $d^G_{\max} = \max\limits_{g \in [m]}d^G_g$ . For any set $T \subset [p]$, $\beta_{T}$ denotes the subvector of $\beta$ indexed by $T$.  Let  $w = \{w_1, \cdots, w_m\}$ be the user-defined positive weights associated with the groups. The group lasso penalty \citep{yuan2006model} is defined as 
\begin{equation}
\label{eq:glnorm}
    \phi^G(\beta) = \sum\limits_{g \in [m]}{w_g \left\|\beta_{G_g}\right\|_2}.
\end{equation}
 We will omit $G$ in all notations when the group structure is clearly given.

In statistical estimation problems involving group selection, the group lasso norm is combined with a convex empirical loss function $L_n$, and the estimator is determined by solving the following M-estimation problem:
\begin{equation}
\label{glreg}
    \text{minimize}_{\beta \in \mathbb{R}^p} \left\{ L_n(\beta) + \lambda_n \phi(\beta) \right\}.
\end{equation}

 If the groups are disjoint, then the group lasso penalty will select and eliminate variables by groups. When the groups overlap, the above estimation enforces an  ``all-out'' pattern by simultaneously setting all variables in certain groups to be zero, thus the zero-out variables are form a union of a subset of the groups  \citep{svsw}. Such a pattern is desirable in many problems, such as graphical models, multi-task learning and gene analysis  \citep{ogl, tcap, mohan2014node, hmgm, tarzanagh2018estimation}. Another generalization of the group lasso for overlapping groups is the latent overlapping group lasso \citep{ogl, mairal2013supervised}, following an ``all-in'' pattern by keeping the nonzero patterns as a union of groups. As noted in \cite{hvs}, the decision to use an ``all-in'' or ``all-out'' strategy depends on the problem and the corresponding scientific interpretations. The comparison between these two strategies is not our objective in this paper. However, both methods suffer from computational difficulties. We focus on introducing an approximation method for the overlapping group lasso penalty \eqref{eq:glnorm} and will leave the computational improvement of the latent overlapping group lasso for future work.


Problem~\eqref{glreg} is a non-smooth convex optimization problem \citep{svsw,chen2012smoothing}, and the proximal gradient method \citep{beck2009fast,nesterov2013gradient} is one of the most general yet efficient strategies to solve it. Intuitively, proximal gradient descent minimizes the objective iteratively by applying the proximal operator of $\lambda_n\phi(\beta)$ at each step.

The proximal operator associated with group lasso penalty in \eqref{eq:glnorm} is defined as
\begin{equation}
\label{op}
    \text{prox}_{\lambda_n}(\mu)=\mathop\mathrm{argmin}_{\beta\in \mathbb{R}^p}\frac{1}{2}\| \mu - \beta \|^2 + \lambda_n \phi(\beta).
\end{equation}
whose dual problem is shown to be the following by \citet{jenatton2011proximal}:
\begin{equation}
\label{opp}
    \mathop\mathrm{minimize}_{\{\xi^g\in \mathbb{R}^p\}_{g \in [m]}}\left(\frac{1}{2}\| \mu - \sum_{g=1}^m \xi^g\|_2^2 \right), \hspace{0.5cm} \text{s.t.} \hspace{0.2cm}\|\xi^g\|_2 \leq \lambda_n w_g, \hspace{0.2cm} \text{and} \hspace{0.2cm} \xi^g_j = 0 \hspace{0.2cm}\text{if} \hspace{0.2cm} j \notin G_g.
\end{equation}
The proximal operator \eqref{op} and its dual can be computed using a block coordinate descent (BCD) algorithm, as studied by \citet{jenatton2011proximal}. We list the procedure in Algorithm~\ref{alg1} for readers' information. The convergence of Algorithm~\ref{alg1} is guaranteed by Proposition 2.7.1 of \cite{bertsekas1997nonlinear}.

 {\tiny
   \begin{algorithm}[hbt]
    \caption{BCD algorithm for the proximal operator of the overlapping group lasso}
    \label{alg1}
    \hspace*{\algorithmicindent} \textbf{Input: $\mathbf{G},$ $\left\{w_g\right\}_{g=1}^m$,  $u$, $\lambda_n$, } \\
    \hspace*{\algorithmicindent} \textbf{Requirement:}  $\mathbf{G}$, $\left\{w_g\right\}_{g=1}^m > 0$,  $\lambda_n > 0$.\\
    \hspace*{\algorithmicindent} \textbf{Initialization: Set $\left\{\xi^{g}\right\}_{g=1}^m = 0 \in \mathbb{R}^p.$ } \\
    \hspace*{\algorithmicindent} \textbf{Output: $\beta^\ast$ } 
     \begin{algorithmic}[1]
     \While {stopping criterion not reached} 
         \ForAll{$g\in \{1, \cdots,m \}$}
         \State Calculate $r^g = \mu - \sum_{h \neq g} \xi^h$.
      \If{$||r^g||_2 \leqslant \lambda_n  w_g$}
      $\xi_j^g  = \begin{cases*}
                    0 & if  $j \notin G_g$  \\
                   r_j^g & if  $j \in G_g$ 
                 \end{cases*} $
    \Else  \hspace{0.1cm} $\xi_j^g  = \begin{cases*}
                    0 & if  $j \notin G_g$  \\
                   \frac{\lambda  w_gr_j^{g}}{||r^{g}||_2} & if  $j \in G_g$ 
                 \end{cases*} $
    \EndIf
  \EndFor
	     	\EndWhile
    \State $\beta^\ast = u - \sum\limits_{g = 1}^{m}\xi^{g}.$
    \end{algorithmic}
\end{algorithm}
}

 Although additional techniques employing smoothing techniques have been developed to improve the optimization \citep{ES,chen2012smoothing},  \eqref{op} and \eqref{opp} continue to offer crucial insights into the computational bottlenecks caused by overlapping groups. Notably, the duality between \eqref{op} and \eqref{opp} reveals that the overlapping group lasso problem has an intrinsic dimension equal to a $\sum_{g \in [m]} d_g$-dimensional separable problem. When the groups have a nontrivial proportion of overlapping variables, the computation of the overlapping group lasso becomes substantially more diﬃcult, eventually prohibitive on large-scale problems. This issue significantly limits the applicability of the overlapping group lasso penalty. Next, we introduce our non-overlapping approximation to rectify this challenge.

\subsection{The Non-overlapping Approximation of the Overlapping Group Lasso}


The fundamental challenge in solving overlapping group lasso problems stems from the non-separability of the penalty. To enhance computational efficiency, our approach hinges on introducing separable operators. As a starting point, we will illustrate this concept using a toy example of an interlocking group structure as a special case. In this structure, the groups are arranged sequentially, with each group overlapping with its adjacent neighbors (Figure~\ref{fig_OG}). For simplicity, we consider a uniform weight scenario where $w_g \equiv 1$ for all groups.

\begin{figure}[H]
\centering
     \begin{subfigure}[t]{0.75\textwidth}
         \centering
         \includegraphics[width=\textwidth]{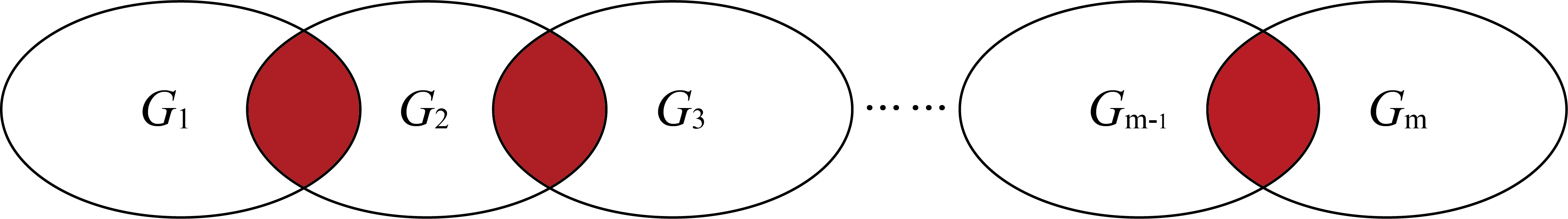}
         \caption{Interlocking group structure.}
         \label{fig_OG}
     \end{subfigure}
     
     \vspace*{1cm}
     \begin{subfigure}[t]{0.75\textwidth}
         \centering
         \includegraphics[width=\textwidth]{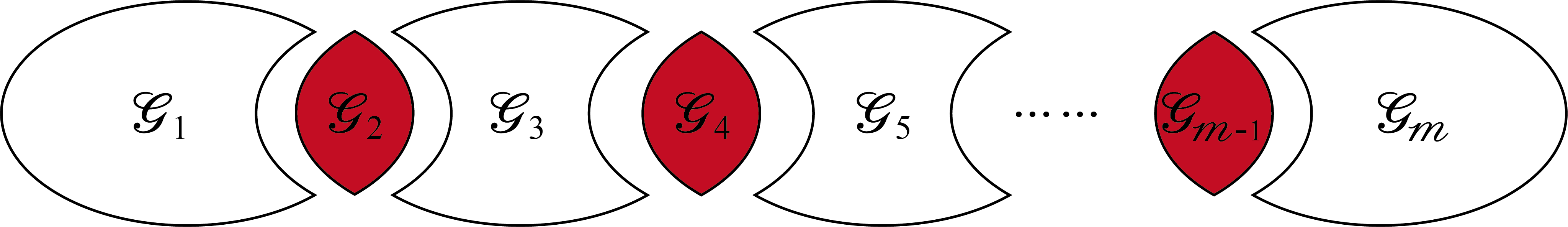}
         \caption{Partitioned group structure.}
         \label{fig_nonOG}
     \end{subfigure}
\caption{Illustration of proposed group partition in an interlocking group structure. Red regions are the overlapping variables in the original group structure. }
\label{figillu}
\end{figure}

We now partition the original overlapping groups in Figure~\ref{fig_nonOG} into smaller groups as in  Figure~\ref{fig_nonOG}. This partition identifies intersections as individual groups. We define these new groups as $\mathcal{G} = \{\mathcal{G}_1, \cdots, \mathcal{G}_\mathcal{m}\}$, where, in this specific instance, $\mathcal{m} = 2m - 1$. Taking $G_1$ as an example. We have $G_1 = \mathcal{G}_1\cup \mathcal{G}_2$ and by the triangular inequality,
$$\|\beta_{G_1}\|_2 \leq \|\beta_{\mathcal{G}_1}\|_2 + \|\beta_{\mathcal{G}_2}\|_2.$$ 
Extending this principle to each group, the norm of the overlapping group lasso based on $G$ can be bounded by a reweighted non-overlapping group norm based on $\mathcal{G}$:
\begin{equation}\label{eq:demo-ineq}
\sum_{g \in [m]}{\|\beta_{G_g}\|_2} \leq \sum_{\mathcal{g} \in [\mathcal{M}]} h_{\mathcal{g}}\|\beta_{\mathcal{G}_{\mathcal{g}}}\|_2,
\end{equation}
where $h_{\mathcal{g}}$ equals $1$ for odd $\mathcal{g}$ and $2$ for even $\mathcal{g}$. Consequently, controlling the sum on the right-hand side of \eqref{eq:demo-ineq} effectively controls the overlapping group norm on the left-hand side. The key advantage of this approach is the separability of the right-hand side norm, which substantially simplifies and enhances the efficiency of optimization.


While this example is about the interlocking group structures, the whole idea is applicable to any general overlapping pattern, as introduced in the two steps below.

\paragraph{Step 1: overlapping-induced partition construction.} Our method starts from constructing a new non-overlapping group structure $\mathcal{G}$ from $G$, following Algorithm~\ref{alg2}. We represent the initial group structure $G$ by an $m \times p$ binary matrix $\mathbf{G}$, where $\mathbf{G}_{gj} = 1$ if and only if the $j$-th variable is a member of the $g$-th group, and  $\mathbf{G}_{gj} = 0$ otherwise. To clearly differentiate the original group structure $G$ and the derived non-overlapping structure $\mathcal{G}$, we employ standard letters, such as $\{g, d, m, w, G\}$, to represent quantities about the original group structure, while calligraphic letters, like $\{\mathcal{g}, \mathcal{d}, \mathcal{m}, \mathcal{w}, \mathcal{G}\}$, are used for quantities about $\Gcal$. For instance, $\mathcal{m}$ denotes the number of groups in $\mathcal{G}$, and $\mathcal{g} \in [\mathcal{m}]$ serves as the index for groups within $\mathcal{G}$.

{\tiny
   \begin{algorithm}[H]
    \caption{Algorithm to construct the overlapping-induced partition $\mathcal{G}$}
    \label{alg2}
    \hspace*{\algorithmicindent} \textbf{Input: Binary matrix $\mathbf{G}$.} \\
    \hspace*{\algorithmicindent} \textbf{Output: New group structure $\mathcal{G}$.} 
     \begin{algorithmic}[1]
     \State Initialize the column index set as $C = \{1, \ldots, p\}$.
     \State Initialize $k = 1$.
     \While {$C$ is not empty}
         \State Choose the first column index $j$ in $C$, and set $I$ to be the set of all column indices in $G$ identical to $G_{,j}$: $I = \{j' \in C, G_{,j'}=G_{,j}\}$.
         \State Set $\mathcal{G}_k = I$, and remove $I$ from $C$: $C \leftarrow C\setminus I$.
         \State $k = k+1$.
     \EndWhile
     \State Return $\mathcal{G} \gets \{\mathcal{G}_1, \mathcal{G}_2, \ldots\}$ where each $\mathcal{G}_k$ represents a group.
    \end{algorithmic}
\end{algorithm}
}

\paragraph{Step 2: overlapping-based group weights calculation.} Note that each group within $\mathcal{G}$ is a subset of at least one of the original groups in $G$. Conversely, each group in $G$ can be reconstructed as the union of groups in $\mathcal{G}$. We introduce the following mappings:
$$F(\mathcal{g}) = \{g: g\in [m], \Gcal_{\gcal} \subset G_g \} ~~~~\text{and}~~~~F^{-1}(g) = \{\gcal: \gcal\in [\mcal], \Gcal_{\gcal} \subset G_g\}.
$$
Given positive weights $w$ of $G$, we set the weights $\mathcal{w}$ of $\mathcal{G}$ as:
\begin{equation}
    \label{eq:weight}
    \mathcal{w}_\mathcal{g} =  \sum_{g \in F(\mathcal{g})}w_g, \quad \mathcal{g} \in [\mathcal{m}].
\end{equation}

With the new partition $\mathcal{G}$ and the new weights $\mathcal{w}$ from the previous two steps, we define the following norm as the proposed alternative to the original overlapping group lasso norm:
\begin{equation}\label{eq:our-norm}
\psi^{\mathcal{G}}(\beta) = \sum\limits_{\mathcal{g}=1}^{\mathcal{m}}{\mathcal{w}_{\mathcal{g}}\left\|\beta_{\mathcal{G}_\mathcal{g}}\right\|_2}.
\end{equation}
In general, by triangular inequality, the proposed norm is always an upper bound of the original group lasso norm:
\begin{equation}\label{eq:general-ineq}
\phi^G(\beta) = \sum\limits_{g=1}^{m}{w_g\left\|\beta_{G_g}\right\|_2} \leqslant \sum\limits_{\mathcal{g}=1}^{\mathcal{m}}{\mathcal{w}_{\mathcal{g}}\left\|\beta_{\mathcal{G}_g}\right\|_2} = \psi^{\Gcal}(\beta).
\end{equation}
Our proposed penalty is essentially a weighted non-overlapping group lasso on $\Gcal$. For illustration, Figure \ref{cof} shows the unit ball of these two norms based on $G_1 = \{\beta_1, \beta_2 \}$ and $G_2 = \{\beta_1 ,\beta_2, \beta_3 \}$ in a three dimensional problem. All singular points of the $\phi^G$-ball (where exactly zero happens in \eqref{glreg}) are also singular points of the $\psi^{\Gcal}$-ball. 

Readers may observe that the inequality in \eqref{eq:general-ineq} can also hold for other separable norms. For instance, consider partitioning all $p$ variables into individual groups and employing a weighted lasso norm as another upper bound for $\phi^G$, represented by:
\begin{equation}\label{eq:weighted-lasso}
 \sum\limits_{j=1}^{p}{\Big(\sum_{\{g|\beta_j \in G_g\} }w_g\Big)\left|\beta_j\right|}.
\end{equation}
This approach was taken by \cite{hmgm}.  So what is special about our proposed norm in \eqref{eq:our-norm}? 

Intuitively, as illustrated by our construction process for $\mathcal{G}$ or Figure~\ref{cof}, our method introduces additional singular points in the norm only when it is necessary to achieve separability. Unlike the lasso upper bound, this process avoids adding redundancy. As such, our approximation is expected to maintain a certain level of tightness. We now formally substantiate this intuition. Given any group structure $G$ and weights $w$, following \citet{sglo}, we define the $\ell_{q_1}/\ell_{q_2}$ norm of $\beta$ for any $0 \leqslant q_1,q_2 \leqslant \infty$ as 
\begin{equation}\label{gnorm}
||\beta_{\{G,w\}}||_{q_1,q_2} = \Big(\sum\limits_{g \in [m]}w_g||\beta_{G_g}||_{q_2}^{q_1}\Big)^\frac{1}{q_1}.
\end{equation}
\begin{figure}[H]
     \centering    
              \vspace{-3cm}
        \includegraphics[width=0.6\textwidth]{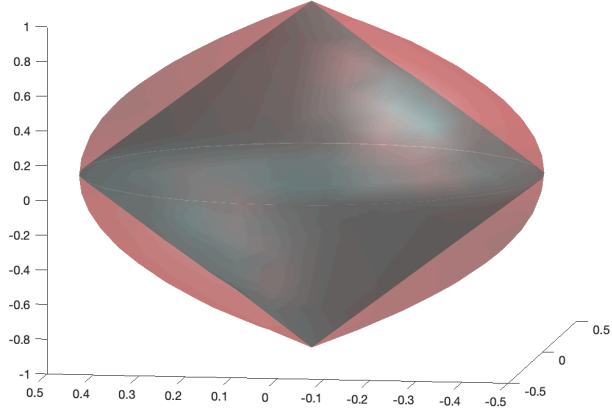}
         \vspace{-3cm}
\caption{Illustration of two norms in $\bR^3$: the outer region depicts the unit ball of the overlapping group lasso norm defined by $\{ \beta: \phi^G(\beta)  \leqslant 1 \}$; the inner region represents the unit ball of our proposed separable norm $\{ \beta: \psi^{\mathcal{G}}(\beta) \leqslant 1 \}$.}
        \label{cof}
\end{figure}

This general class of norms potentially includes most commonly used penalties, including the weighted lasso penalty. The subsequent theorem shows that the proposed $\psi^{\mathcal{G}}(\beta)$ is the \emph{tightest separable relaxation} of the original overlapping group lasso norm among all separable $\ell_{q_1}/\ell_{q_2}$ norms.

\begin{theorem}
\label{addtheo}
Let  $\mathbb{G}$ represent the set of all possible partitions of $[p]$. Given the original groups $G$ and their weights $w$, there does not exist $0 \leqslant q_1,q_2 \leqslant \infty, \tilde{G} \in \mathbb{G}, \tilde{w} \in (0, \infty)^p$ such that:
\begin{equation}\label{twocondi}
    \begin{cases}
 \phi^G(\beta) \leqslant ||\beta_{\{\tilde{G},\tilde{w}\}}||_{q_1,q_2} \leqslant \psi^{\mathcal{G}}(\beta)  &  \text{ for all $\beta \in \mathbb{R}^p$} \\
  ||\beta_{\{\tilde{G},\tilde{w}\}}||_{q_1,q_2} < \psi^{\mathcal{G}}(\beta)   &  \text{ for some $\beta \in \mathbb{R}^p$}
\end{cases}.
\end{equation}
\end{theorem}

\section{Statistical Properties}
\label{sec:results}
Incorporating the proposed norm $\psi_{\Gcal}$ into an M-estimation procedure leads to the following optimization problem:
\begin{equation}
\label{glreg-alt}
    \text{minimize}_{\beta \in \mathbb{R}^p} \left\{ L_n(\beta) + \lambda_n \psi_{\Gcal} \right\},
\end{equation}
which is different but related to \eqref{glreg}. In this section, by studying the statistical properties of the regularized estimator based on $\psi_{\Gcal}$ and the estimator based on $\phi_G$, we show that $\psi_{\Gcal}$ could be used as an alternative to $\phi_G$. Following previous group lasso studies \citep{tbog,oiao,chen2012smoothing,aufh,aebf}, our analysis will focus on high-dimensional linear models. Specifically, define the linear model as
\begin{equation}\label{eq:linear-model}
    Y=X\beta^{\ast}+\varepsilon,
\end{equation}
where $Y \in \mathbb{R}^{n \times 1}$ is the response vector,  $X \in \mathbb{R}^{n \times p}$ is the covariate matrix,  and $\varepsilon \in \mathbb{R}^{n \times 1}$ is a random noise vector. The overlapping group lasso coefficient estimator under the linear regression model is defined by a solution of \eqref{glreg} under the squared loss:
\begin{equation}\label{eq:OGL-est}
\hat{\beta}^G \in \argmin_{\beta \in \bR^p}~~\frac{1}{2n}\norm{Y-X\beta}_2^2 + \lambda_n  \phi_G(\beta). 
\end{equation}
Correspondingly, we define the regularized estimator by our approximation norm as
\begin{equation}\label{eq:Our-est}
\hat{\beta}^\Gcal \in \argmin_{\beta \in \bR^p}~~\frac{1}{2n}\norm{Y-X\beta}_2^2 + \mathcal{\lambda}_n \psi_{\Gcal}(\beta).
\end{equation}

The solution uniqueness of \eqref{eq:OGL-est} and \eqref{eq:Our-est} has been studied by \cite{svsw}, and we include their results in Appendix~\ref{app:uniqueness} for completeness. However, our study only requires the estimator to be one solution to the problem, as in \cite{svsw,aufh,wainwright2019high}. So we will not specifically worry about the uniqueness in our discussion.

As a remark, our objective is \textbf{not} to present \eqref{eq:Our-est} as an approximate optimization problem of \eqref{eq:OGL-est}. Rather, we focus on the statistical equivalence of the two classes of estimators defined by \eqref{eq:OGL-est} and \eqref{eq:Our-est} in terms of their statistical properties under sparse regression models when appropriate values of $\lambda_n$ are chosen (which may differ for each estimator). Our theoretical analysis focuses on three aspects. In Section~\ref{sec:l2-bound}, we establish that under reasonable assumptions, the $\ell_2$ estimation error bound for \eqref{eq:Our-est} is no larger than that for \eqref{eq:OGL-est}. In Section~\ref{sec:minimax}, we present the minimax error rate for the overlapping sparse group regression problem, showing that both \eqref{eq:OGL-est} and \eqref{eq:Our-est} are minimax optimal under additional requirements of the group structures. Lastly,  in Section~\ref{sec:support}, we demonstrate that both estimators consistently recover the support of the sparse $\beta^*$ with high probability under similar sample size requirements.

\subsection{Estimation Error Bounds}\label{sec:l2-bound}
We start by introducing additional quantities.  Define the overlapping degree $h^G_j$ as the number of groups in $G$ containing $\beta_j$, and $h^G_{\max} = \max h_j$. Given a group index set  $I \subseteq [m]$, we use $G_{I}$ to denote the union $\mathop{\bigcup}_{g\in I}G_g$. Given $G$ and $I$,  following \cite{wainwright2019high}, we define two parameter spaces: 
\begin{align*}
M(I) &= \left\{\beta\in{\mathbb{R}^p} \mid \beta_j=0 \text{ for all } j \in (G_{I})^c\right\}, \\
M^{\perp}(I) &= \left\{\beta\in \mathbb{R}^p \mid \beta_j=0 \text{ for all } j\in G_{I}\right\},
\end{align*}
and we further use $\beta_{M(I)}$ to denote the projection of $\beta$ onto $M(I)$.
 
Given any set $T \subseteq [p]$, we define the a set of groups $\mathsf{G}_{T}=\left\{g\in [m] \mid G_g  \cap T \neq \emptyset \right\}$. Notice that $(G_{\mathsf{G}_{T}})^c$ is called the hull of $T$ in \cite{svsw}. Let  $\textit{supp}(\beta)=\{j\in [p] \mid \beta_j\neq 0\}$  denotes the support set. We define the group support set $S^G(\beta) = \mathsf{G}_{\textit{supp}(\beta)}$, and the augmented group support $\overline{S^G(\beta)}=\{g \in [m] \mid G_g\cap G_{S(\beta)}\neq \varnothing\}$. Furthermore, define $s = |\textit{supp}(\beta)|$, $s_g = |S(\beta)|$, and $\overline{s_g} = |\overline{S(\beta)}|$.  We omit the subscript $G$ in notations when $G$ is clearly given in context. Now we introduce additional assumptions under the regression model \eqref{eq:linear-model}.

\begin{assu}[Sub-Gaussian noise for the response variable]
\label{ass:distribution_sub_noise}
 The coordinates of $\varepsilon$ are i.i.d. zero-mean sub-Gaussian with parameter $\sigma$. Specifically, there exists $\sigma > 0$ such that $\mathbb{E}[\exp(t\varepsilon)] \leqslant \exp(\sigma^2t^2/2)$ for all $t \in \mathbb{R}$.
\end{assu}
Our theoretical studies also hold for a fixed design of $X$, with trivial modifications. We prefer to introduce the random design here to make the statements more concise and interpretable, especially for the comparison in Section~\ref{sec:support}.
\begin{assu}[Normal random design for covariates]
\label{ass:distribution_normal_rd}
The rows of the data matrix $X$ are i.i.d. from $N(0,\Theta)$, where \(1/c_1 \leqslant \gamma_{\min}(\Theta) \leqslant \gamma_{\max}(\Theta) \leqslant c_1\) for some constant \(c_1>0\).
\end{assu}
Lastly, we need some mild constraints on the group dimensions.
\begin{assu}[Dimension of the group structure]
\label{ass:distribution_groupstructure}
The predefined group structure $G$ satisfies \(d_{\max} \leqslant c_2n\) for some constant \(c_2>0\). In addition, we assume \(\log{m}\ll n\). 
\end{assu}

The following theorem establishes the $\ell_2$ estimation error bounds for the two estimators.
\begin{theorem}
\label{thm:two-bounds}
Given $G$ and its induced $\Gcal$ according to Algorithm~\ref{alg2}, define $h_{\min}^{g} = \min\limits_{j \in G_g}h_j$, $h_{\max}^{g} = \max\limits_{j \in G_g}h_j$. Let $\delta \in (0,1)$ be a scalar that might depend on $n$. Under Assumptions~\ref{ass:distribution_sub_noise}, \ref{ass:distribution_normal_rd} and \ref{ass:distribution_groupstructure}, for $\hat{\beta}^G$ and $\hat{\beta}^{\Gcal}$ defined in \eqref{eq:OGL-est} and \eqref{eq:Our-est}, we have the following results:
\begin{enumerate}
    \item Suppose that  $\beta^*$ satisfies the group sparsity condition
    \begin{equation}\label{eq:orig-sparsity}
    \overline{s_g}(\beta^*) \lesssim
    \frac{n}{\log m + d_{\max}}
    \cdot
    \frac{\min\limits_{g\in[m]}(w_g^2 h_{\min}^{g})}{\max\limits_{g\in \overline{S}} (w_g^2 h_{\max}^{g})}.  
    \end{equation}
    When $\lambda_n = \frac{c' \sigma}{\min\limits_{g\in[m]}\left(w_g^2 h_{\min}^{g}\right)}\sqrt{\frac{d_{\max}}{n} + 
    \frac{\log{m}}{n} + \delta}$ for some constant $c'>0$, we have
    \begin{equation}\label{eq:ell2bound}
        \left\|\hat{\beta}^{G}-\beta^{\ast}\right\|^2_2 \lesssim \sigma^2\cdot\frac{\Big(\sum\limits_{g\in\overline{S}}{w_g}^2\Big)\cdot h^{G_{\overline{S}}}_{\max}}{\min\limits_{g\in[m]}\left(w_g^2 h_{\min}^{g}\right)}
        \cdot\left(\frac{d_{\max}{}}{n} + \frac{\log{m}}{n} + \delta\right).
    \end{equation}
    with probability at least $1-e^{-c_3n\delta}$ for constant $c_3>0$. 
    \item Suppose $\beta^*$ satisfies the group sparsity condition
    \begin{equation}\label{eq:new-sparsity}
    \overline{s_\gcal}(\beta^*) \lesssim
    \frac{n}{\log \mcal + \dcal_{\max}}
    \cdot
    \frac{\min\limits_{\gcal\in[\mcal]}(\wcal_\gcal^2)}{\max\limits_{\gcal\in S}(\wcal_\gcal^2)}.  
    \end{equation}
    When $\lambda_n = \frac{c' \sigma}{\min\limits_{\gcal\in[\mcal]}\wcal_g}
    \sqrt{\frac{\dcal_{\max}}{n} + \frac{\log{\mcal}}{n} + \delta}$ for some constant $c'>0$, we have
    \begin{equation}\label{eq:parell2bound}
       \left\|\hat{\beta}^{\Gcal}-\beta^{\ast}\right\|_2^2 \lesssim \sigma^2\cdot\frac{\sum\limits_{\gcal\in \{F^{-1}(g)\}_{g \in S} }{\wcal_\gcal}^2}{\min\limits_{\gcal\in[\mcal]}\left(\wcal_\gcal^2\right)}
       \cdot\left(\frac{\dcal_{\max}{}}{n} + \frac{\log{\mcal}}{n} + \delta\right).
    \end{equation}
    with probability at least $1-e^{-c_4n\delta}$ for constant $c_4>0$.
\end{enumerate}
\end{theorem}

The error bound in \eqref{eq:ell2bound} subsumes the non-overlapping group lasso error bound as a particular instance. When the groups in $G$ are disjoint, the reduced form of \eqref{eq:ell2bound} matches the bounds studied in \cite{tbog,oiao,aufh,wainwright2019high}. The main difference in the context of overlapping groups is the necessity to account for the overlapping degree and the extension of sparsity requirements to augmented groups. The conditions specified in \eqref{eq:orig-sparsity} and \eqref{eq:new-sparsity}  relate to the cardinality of the augmented group support set (the number of non-zero groups in non-overlapping group structure). Although the conditions in \eqref{eq:orig-sparsity} and \eqref{eq:new-sparsity} may initially appear distinct, they generally converge to a similar requirement in many typical cases, which can lead to an informative comparison between the two bounds in \eqref{eq:ell2bound} and \eqref{eq:parell2bound}. The following results can characterize this.

\begin{assu}\label{ass:compare-bound}
   Assume the predefined group structure $G$ and its induced group structure $\Gcal$ satisfy $\max\{d_{\max},m \}\asymp \max\{\dcal_{\max},\mcal\}$.
\end{assu}

\begin{proposition}
\label{cor1}
 Suppose that $\max\limits_{g \in \overline{S}} |F^{-1}(g)|$ is bounded by a constant. Under Assumption~\ref{ass:compare-bound}, the following inequality holds:
\begin{equation*}
    \frac{\sum\limits_{\mathcal{g}\in F^{-1}(S) }{\mathcal{w}_{\mathcal{g}}}^2}{\min\limits_{\mathcal{g}\in[m]}\left(w_\mathcal{g}^2 \right)}
    \cdot\left(\frac{\mathcal{d}_{\max}{}}{n}+\frac{\log{\mathcal{m}}}{n}+\delta\right) \lesssim \frac{\big(\sum\limits_{g\in\overline{S}}{w_g}^2\big)\cdot h^{G_{\overline{S}}}_{\max}}{\min\limits_{g\in[m]}\left(w_g^2 h_{\min}^{g}\right)}
    \cdot\left(\frac{d_{\max}{}}{n}+\frac{\log{m}}{n}+\delta\right).
\end{equation*}
This implies that the error bound for the estimator $\hat{\beta}^{G}$ in \eqref{eq:ell2bound} also serves as an upper bound for the error associated with the estimator $\hat{\beta}^{\Gcal}$. 
\end{proposition}

The quantity $|F^{-1}(g)|$ is the number of groups in $\Gcal$ that has intersect with $G_g$. Proposition~\ref{cor1} requires that every $G_g$ such that $G_g \cap\textit{supp}(\beta^*) \neq \emptyset$ is partitioned into bounded number of non-overlapping groups. On the other hand, Assumption \ref{ass:compare-bound} requires that the maximum of two quantities --- the maximum group size and the number of groups in the given group structure $G$ --- should have the same order as those in the induced structure $\mathcal{G}$. The above requirement always holds for interlocking groups with similar groups and overlap sizes (see Figure~\ref{figillu}). \emph{More importantly, we can always assess the assumption directly on data} by calculating the group sizes and numbers for both $G$ and $\mathcal{G}$. In Section 4.3, we evaluate five group structures from real-world gene pathways and examine the ratio of the maximum of two quantities from each  $G$ and $\mathcal{G}$. Assumption~\ref{ass:compare-bound} looks reasonable in all of these real-world grouping structures. See details in Table \ref{tab:pathways}.


\subsection{Lower Bound of Estimation Error}\label{sec:minimax}

Proposition~\ref{cor1} compares the two estimators' upper bounds of estimation errors. While the comparison gives intuitive ideas, it does not rigorously establish the statistical equivalence without the tightness of the error bounds. To strengthen our findings, we now investigate the minimax estimation error rate in linear regression models characterized by overlapping group sparsity. We will focus on the following class of group-wise sparse vectors:
\begin{equation}
\label{class}
    \Omega(G,s_g) = \bigg\{\beta : \sum\limits_{G_g \in G}\mathbbm{1}_{\{\norm{\beta_{G_g}}_2 \neq 0\} } \leqslant s_g\bigg\}
\end{equation}
Following the assumption of \cite{sglo}, we focus on the special case of equal-size groups.
\begin{assu}[Equal size groups]
\label{ass:lowerbound}
The $m$ predefined groups of $G$ come with equal group size $d$, $m \ll p, d \ll \log(p).$ 
\end{assu}

\begin{theorem}
\label{the:lowerbound}
	(Lower bound of estimation error)·    
	Under Assumptions~\ref{ass:distribution_sub_noise},\ref{ass:distribution_normal_rd} and \ref{ass:lowerbound},  we have
	\begin{equation}\label{eq:minimax}
	\inf\limits_{\hat{\beta}}  \sup\limits_{\beta \in \Omega(G,s_g)} E\|\hat{\beta}-\beta\|_{2}^{2} \gtrsim \frac{\sigma^2\left(s_g(d +\log (\frac{m}{s_g}))\right)}{n}. 
	\end{equation}
\end{theorem}

Combining Theorem~\ref{thm:two-bounds} and Theorem~\ref{the:lowerbound}, we can see that both estimators attain the minimax error rate and are statistically equivalent, as demonstrated by the following corollary.
\begin{coro}
\label{cor2}
Under Assumptions~\ref{ass:distribution_sub_noise}--\ref{ass:compare-bound}, if $h^{G_{\overline{S}}}_{\max} \asymp 1$, both $\hat{\beta}^G$ and $\hat{\beta}^{\Gcal}$ attain the minimax estimation rate specified in \eqref{eq:minimax}.
\end{coro}

\subsection{Support Recovery Consistency}\label{sec:support}

We now proceed to analyze the support recovery consistency of \( \hat{\beta}^G \) and  \( \hat{\beta}^{\Gcal} \). We begin by introducing more quantities for our analysis. For any \(\beta \in \mathbb{R}^p\), we define the  mapping \(r^G(\beta): \mathbb{R}^p \to \mathbb{R}^p\) as:
\begin{align}
     r^G(\beta)_j = \begin{cases}
       \beta_j  \sum\limits_{g \in \mathsf{G}_{\textit{supp}(\beta)}, G_g \cap j \neq \emptyset} \frac{w_g}{\|\beta_{G_g \cap \textit{supp}(\beta)}\|_2}, & \text{if } j \in \textit{supp}(\beta), \\
       0, & \text{if } j \notin \textit{supp}(\beta).
     \end{cases}
\end{align}
$r^G(\beta)$ is closely related to subgradients of the penalty and is used for determining optimality conditions. In the lasso case, $r^G(\beta)$ is the sign vector, which is exactly the lasso penalty. When focusing on $\beta^*$, we write  \(\mathbf{S} = \textit{supp}(\beta^*)\), \(\mathbf{r}^G = r^G(\beta^*)\), and  $\beta^*_{\min}=\min \left\{|\beta_j^*| ; \beta_j^* \neq 0\right\}$.

Our analysis essentially follows the strategy in \cite{svsw}. The major difference is that we study the problem with a more tailored setup for the random design rather than the fixed design as in \cite{svsw}. Using random designs, as discussed before, is helpful to compare the two estimators \( \hat{\beta}^G \) and  \( \hat{\beta}^{\Gcal} \) directly.  We now introduce additional assumptions used to study the pattern consistency, which can be seen as the population-level counterpart of the assumptions in \cite{svsw}.

\begin{assprime}[Gaussian noise for the response variable]\label{ass:distribution_normal noise}
 Under model \eqref{eq:linear-model}, the coordinates of $\varepsilon$ are i.i.d  from  $N(0,\sigma^2)$.
\end{assprime}

\begin{assu}[Irrepresentable condition]
\label{ass:Irrepresentable_condition}
For any $\beta \in \mathbb{R}^{p}$, define
$$\phi_{\mathbf{S}}^c(\beta_{\mathbf{S}^c}) = \sum_{g \in [m] \setminus \mathsf{G}_{\mathbf{S}}} w_g \|\beta_{ \mathbf{S}^c \cap G_g}\|_2,$$ and its dual norm $$(\phi_{\mathbf{S}}^c)^*[u] = \sup_{\phi_{\mathbf{S}}^c(\beta_{\mathbf{S}^c}) \leq 1} \beta_{\mathbf{S}^c}^\top u.$$ 
Assume that there exists $\tau \in (0,\frac{2}{3}]$, such that  
\begin{equation}
\label{condition: irre}
    (\phi_{\mathbf{S}}^c)^*[\Theta_{\mathbf{S}^c\mathbf{S}} \Theta_{\mathbf{S}\mathbf{S}}^{-1} \mathbf{r}_\mathbf{S}] \leqslant 1 - \frac{3\tau}{2}.
\end{equation}
 \end{assu}

Assumption~\ref{ass:distribution_normal noise} is widely used to study support recovery consistency of linear regression. For example, in addition to \cite{svsw}, it is also used in \cite{zhao2006model,wainwright2009sharp,wainwright2019high}. Assumption~\ref{ass:Irrepresentable_condition} is the population-level version of the irrepresentable condition as discussed in \citet{zhao2006model} and \citet{wainwright2019high}. 

\begin{theorem}
\label{pattern}
Suppose Assumption \ref{ass:distribution_normal noise}, Assumption~\ref{ass:distribution_normal_rd} and Assumption~\ref{ass:Irrepresentable_condition} hold. Under model~\eqref{eq:linear-model},  assume the support of $\beta^*$ is compatible with the overlapping group lasso penalty, such that the zero positions are given by an exact union of groups in $G$. Mathematically, that means
    \begin{equation}
    \label{condition:hull}
        [p] \setminus \big\{\bigcup_{G_g \cap \mathbf{S} = \emptyset} G_g\big\} = \mathbf{S}.
    \end{equation}

\begin{enumerate}
    \item \label{the6part1}  If  
    $$\log(p-|\mathbf{S}|) \geqslant |\mathbf{S}|,$$ 
    \begin{equation}
\label{condition: reg}
       \lambda_n|\mathbf{S}|^{\frac{1}{2}} \lesssim \min\Big\{\frac{\beta_{\min}^* }{A_{\mathbf{S}}},\frac{\beta_{\min}^*a_{\mathbf{S^c}} }{A_{\mathbf{S}}\sum\limits_{g\in \mathsf{G}_{\mathbf{S}}} w_g \sqrt{\left|G_g \cap \mathbf{S}\right|}}\Big\},
\end{equation}
\begin{equation}
\label{condition: sample}
   n \gtrsim \max\Big\{\frac{\sigma^2\log(p-|\mathbf{S}|)}{a^2_{\mathbf{S}^c}\lambda^2_n},\frac{ \max\limits_{j \in \mathbf{S}}\{(\beta^*_j)^2\}\log(p-|\mathbf{S}|)}{a^2_{\mathbf{S}^c}\lambda^2_n}\Big\},
\end{equation}
where $ a_{\mathbf{S}} = \min\limits_{g \in \mathsf{G}_{\mathbf{S}}} \frac{w_g}{d_g}$,  $  a_{\mathbf{S}^c} = \min\limits_{g \in \mathsf{G}_{\mathbf{S^c}}} \frac{w_g}{d_g}$, and $ A_{\mathbf{S}} =  h_{\max}(\mathbf{G_{S}})\max\limits_{g \in \mathsf{G}_{\mathbf{S}}}w_g\|u\|_1$.

Then for the overlapping group lasso estimator $\hat{\beta}^G$, we have:
\begin{equation}
\label{pattern1}
\begin{aligned}
    \mathbb{P}\Big( \textit{supp}(\hat{\beta}^G) \neq \mathbf{S}\Big) 
    \leqslant & 8\exp\Big(-\frac{n}{2}\Big) +  \exp\Big( - \frac{n a_{\mathbf{S}}^2\tau^2 \gamma_{\min}(\Theta_{\mathbf{S}\mathbf{S}})}{4\left\|\mathbf{r}_\mathbf{S}\right\|^2_2  \gamma_{\max}\left(\Theta_{\mathbf{S}^c\mathbf{S}^c|\mathbf{S}}\right)} \Big)
  \\   & +\exp \Big(-\frac{n \lambda_n^2\tau^2 a_{\mathbf{S}^c}^2}{144 \sigma^2}\Big) + 2|\mathbf{S}| \exp \Big(-\frac{n c^2(\mathbf{S},G)}{2 \sigma^2}\Big)
\end{aligned}
\end{equation}
with $$c(\mathbf{S},G) \asymp  \min\Big\{\frac{\beta^*_{\min}}{ A_{\mathbf{S}}},\frac{\beta^*_{\min}  a_{\mathbf{S^c}}}{A_{\mathbf{S}} \sum\limits_{g\in \mathsf{G}_{\mathbf{S}}} w_g \sqrt{\left|G_g \cap \mathbf{S}\right|}}\Big\}.$$
    \item  \label{the6part2} Furthermore, if  $\max_{g \in \mathsf{G}_{\mathbf{S}}}F^{-1}(g) \asymp 1$, for the  proposed estimator \( \hat{\beta}^{\Gcal} \) and assuming $\max_{g \in \mathsf{G}_{\mathbf{S}}}F^{-1}(g) \asymp 1$,  the property holds:
\begin{equation}
\label{pattern2}
\begin{aligned}
    \mathbb{P}\Big( \textit{supp}(\hat{\beta}^\Gcal) \neq \mathbf{S}\Big) 
    \leqslant & 8\exp\Big(-\frac{n}{2}\Big) +  \exp\Big( - \frac{n a_{\mathbf{S}}^2\tau^2 \gamma_{\min}(\Theta_{\mathbf{S}\mathbf{S}})}{4\left\|\mathbf{r}^{\Gcal}_\mathbf{S}\right\|^2_2  \gamma_{\max}\left(\Theta_{\mathbf{S}^c\mathbf{S}^c|\mathbf{S}}\right)} \Big)
  \\   & +\exp \Big(-\frac{n \lambda_n^2\tau^2 a_{\mathbf{S}^c}^2}{144 \sigma^2}\Big) + 2|\mathbf{S}| \exp \Big(-\frac{n c^2(\mathbf{S},\Gcal)}{2 \sigma^2}\Big),
\end{aligned}
\end{equation}
with $$c(\mathbf{S},\Gcal) \asymp  \min\Big\{\frac{\beta_{\min}^* }{A_{\mathbf{S}}},\frac{\beta_{\min}^*a_{\mathbf{S^c}} }{A_{\mathbf{S}}\sum\limits_{\gcal \in \mathsf{\Gcal}_{\mathbf{S}}} w_\gcal \sqrt{\left|\Gcal_\gcal \cap \mathbf{S}\right|}}\Big\}.$$
\end{enumerate}

\end{theorem}

The conditions involved in the above theorem can be seen as the population-level counterparts of those used in \cite{svsw} for the overlapping group lasso estimator under the fixed design. As an illustration of the conditions, in the lasso context, \eqref{condition: reg} and \eqref{condition: sample} reduce to the typical scaling of $n \approx \log p$ and $\lambda_n \approx \sigma(\log 
p/n)^{1/2}$. Together with the requirements on the sample size$|\mathbf{S}|\log(p-|\mathbf{S}|)$ and on $\beta^*_{\min}$, they match the requirements in \citet{wainwright2009sharp} for the support recovery by the lasso regression. For non-overlapping group lasso estimators, our assumptions align with the conditions outlined in Corollary 9.27 of \citet{wainwright2019high} under the random design. 

Theorem~\ref{pattern} shows that both estimators consistently identify the support of the group sparse regression coefficients. Compared to the previous study of the overlapping group lasso estimator of \cite{svsw}, we switch to the random design of $X$, because such a setting renders a common basis for the comparison of the two estimators directly. Specifically, comparing \eqref{pattern1} and \eqref{pattern2}, as well as the common conditions, we can see that the two estimators give comparable performance in support recovery with respect to the sampling complexity.

\section{Simulation}
\label{sec:simu}

In this section, we assess the performance of the proposed estimator to demonstrate our claimed properties. At a high level, we want to use the simulation experiments to show that the proposed estimator based on \eqref{eq:our-norm} gives similar statistical performance to the overlapping group lasso estimator while admitting much better computational efficiency. Our estimator achieves this primarily because of the tightest separable relaxation property of Theorem~\ref{addtheo}, which can be attributed to two designs of the norm \eqref{eq:our-norm}: the induced partition $\Gcal$ and the corresponding overlapping-based weights $\wcal$. Therefore, in our simulation experiments, we will also evaluate the effects of these two designs by comparing the proposed estimator with other benchmark estimators. In Sections~\ref{subsec:simu_int}--\ref{secsec:sim-gene}, we evaluate the performance of the proposed estimator and compare it with the weighted lasso estimator with overlapping-based weights, as discussed in \eqref{eq:weighted-lasso}, under various configurations. This sequence of experiments will demonstrate the importance of our proposed partition $\Gcal$.  In Section~\ref{secsec:sim-weight}, we compare the proposed estimator with two other group lasso estimators, using the same $\Gcal$ but overlapping-ignorant weights, under the same set of configurations. The results will demonstrate the importance of using the proposed overlapping-based weights $\wcal$.

Two MATLAB-based solvers for the overlapping group lasso problems are employed. The first solver \citep{ES} is from the SLEP package \citep{Liu:2009:SLEP:manual}. 
It can handle general overlapping group structures. The second solver is from the SPAM package \citep{mairal2014spams}, which is designed to solve the overlapping group lasso problem when the groups can be represented by tree structures, formally defined in Section~\ref{secsec:sim-tree}. Therefore, the SPAM solver is used only for the experiment in Section~\ref{secsec:sim-tree}. The SLEP solver is more general, but using the two solvers can provide a more thorough evaluation across multiple implementations. For a fair comparison, the SLEP and SPAM package solvers were also applied to solve lasso and non-overlapping group lasso estimators in our benchmark set to ensure that the timing comparison implementation is consistent.  

As an important side note, SLEP is widely acknowledged as one of the most efficient solvers for the overlapping group lasso problem \citep{ES,chen2012smoothing,hmgm}. Still, for non-overlapping group lasso problems, alternative solvers, such as \cite{afua}, may offer much better computational efficiency. For example, \citet{afua} reported that their solver is about 10--30 times faster than the SLEP package when solving non-overlapping group lasso problems. Such solvers are available because of the separability in non-overlapping groups and are not available for overlapping problems. For a fair comparison to avoid implementation bias, we use SLEP to solve for our estimator. Therefore, the computational advantage we demonstrate will be conservative. In practice, with the better solvers used, our method would enjoy an even more substantial computational advantage over the original overlapping group lasso than reported in the experiments.

\paragraph{Evaluation criterion.} For each configuration, we generate 50 independent replicates and report the average result.  The performance assessment is conducted in three aspects:
\begin{itemize}
\item \textbf{Regularization path computing time.} We begin by performing a line search to determine two pivotal values: \(\lambda_{\text{max}}\) and \(\lambda_{\text{min}}\). The search for \(\lambda_{\text{max}}\) starts at \(10^{8}\) and decreases progressively, multiplying by 0.9 at each iteration, until reaching the first value at which no variables are selected. In contrast, the determination of \(\lambda_{\text{min}}\) starts from \(10^{-8}\) and increases incrementally, multiplying by 1.1 each time, until the first value is found that retains the entire set of variables. Following this, we select 50 values in log-scale within the range $[\lambda_{\min}, \lambda_{\max}]$. Subsequently, We compute the entire regularization path using these $\lambda$ values and record the computation time associated with this process as a performance metric. The computing time evaluation mimics the most practical situation where the whole regularization path is solved for tuning purposes.

\item \textbf{Relative $\ell_2$ estimation error}: From the entire regularization path, we select the smallest relative estimation error, defined as $\norm{\hat{\beta}-\beta^{\ast}}_2/\norm{\beta^{\ast}}_2$, as the estimation error for the method. This serves as the measure of the ideally tuned performance.

\item \textbf{Support discrepancy}: From the entire regularization path, we select the smallest support discrepancy, defined as $|\{i \in [p]: |\text{sign}(\hat{\beta}_i)| \neq |\text{sign}(\beta^{\ast}_i)|\}|/p$. Such a (normalized) Hamming distance is commonly used as a performance metric for support recovery  \citep{grave2011trace,svsw} to quantify the accuracy of pattern selection.

\end{itemize}

\subsection{Interlocking group structure}\label{subsec:simu_int}

In the first set of experiments, we evaluate the performances based on interlocking group structure (Figure~\ref{fig_OG}). This group structure exhibits a relatively low degree of overlap and is frequently used for evaluating overlapping group lasso methods \citep{ES,chen2012smoothing}. Specifically, we set $m$ interlocked groups with $d$ variables in each group and  $0.2d$ variables in each intersection. For example, $G_1 = \{1,\cdots,10\}, G_2 = \{8, 9,\cdots,17\},\cdots,G_{10} = \{33, 34,\cdots,42\}$ when $m = 5$ and $d = 10$. In the experiment, we will vary $m$ and $d$ to evaluate their impacts on the performance.

Following the strategy of \cite{hvs}, we generate the data matrix $X$ from a Gaussian distribution $N(0, \Theta)$, where $\Theta$ is determined to match the correlations within the specified group structure. Initially, we construct a matrix $\tilde{\Theta}$ as follows:
\[
\tilde{\Theta}_{ij} =
\begin{cases}
1, & \text{if } i = j, \\
0, & \text{if $\beta_i$ and $\beta_j$ belong to different groups in $G$}, \\
0.6, & \text{if $\beta_i$ and $\beta_j$ are in the same group in $\Gcal$}, \\
0.36, & \text{if $\beta_i$ and $\beta_j$ are in the same group in $G$ but different groups in $\Gcal$},
\end{cases}
\]
and then $\Theta$ is derived as the projection of $\tilde{\Theta}$ onto the set of symmetric positive definite matrices with a minimum eigenvalue of $0.1$.  Such strong within-group correlation patterns have also been used in \citet{ogll,afua}.

We generate $\beta^{\ast}$ by first sampling its $p$ coordinates from the normal distribution $N(10, 16)$, then randomly flipping signs of the covariates and randomly setting $90\%$ of the groups to be zero. This setup aligns with the setting in \citet{glc,anotg,tbog}. The response variable $Y$ is generated from $Y = X\beta^{\ast} + \epsilon$, where $\epsilon$ follows a normal distribution with mean $0$ and variance $\sigma^2$, and we set $\sigma^2 = 3$ following \citet{afua}. The group weight in the overlapping group lasso problem is $w_g = \sqrt{d_g}$, as is usually used in practice. We used the absolute difference in function values between iterations for all methods as the stopping criterion, with a tolerance set at $10^{-5}$. 

\begin{figure}[H]
\centering
\begin{subfigure}[t]{\textwidth}
\centering
\includegraphics[width =0.32\textwidth]{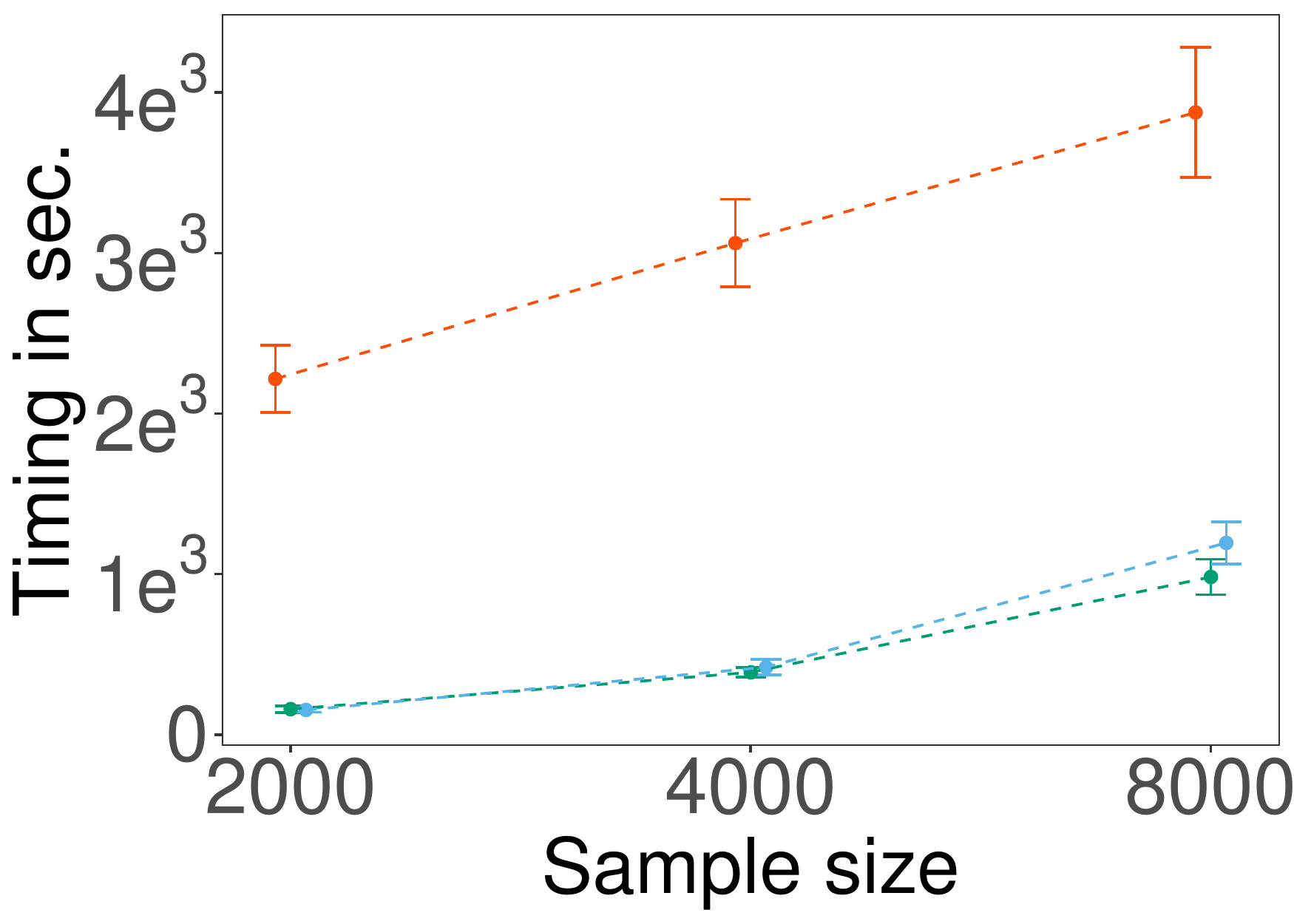}
\hfill
 \includegraphics[width=0.32\textwidth]{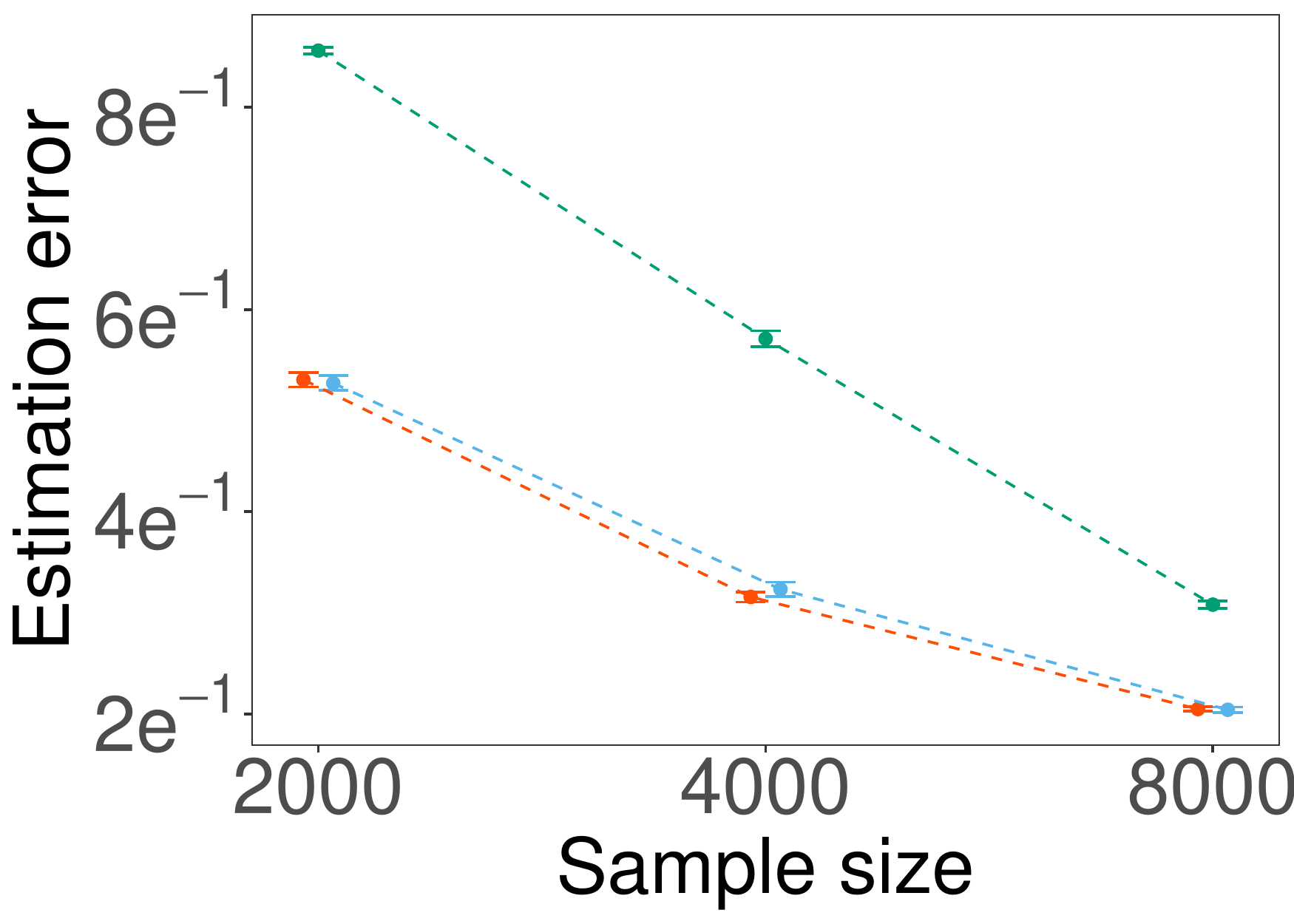}
  \includegraphics[width=0.32\textwidth]{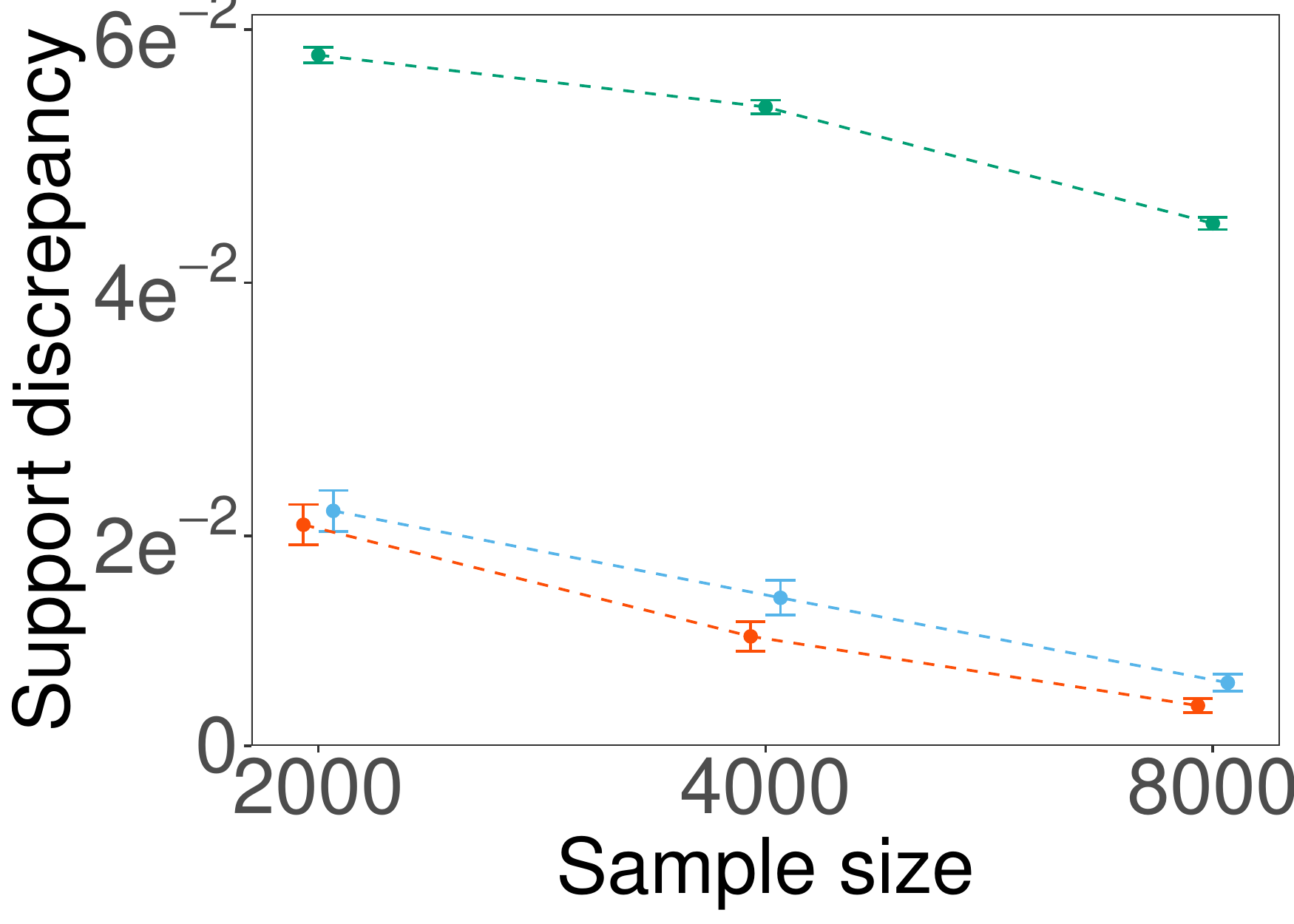}
\caption{Performance vs. Sample size}
\label{fig:tsam}
\end{subfigure}
\begin{subfigure}[t]{\textwidth}
\centering
\includegraphics[width =0.32\textwidth]{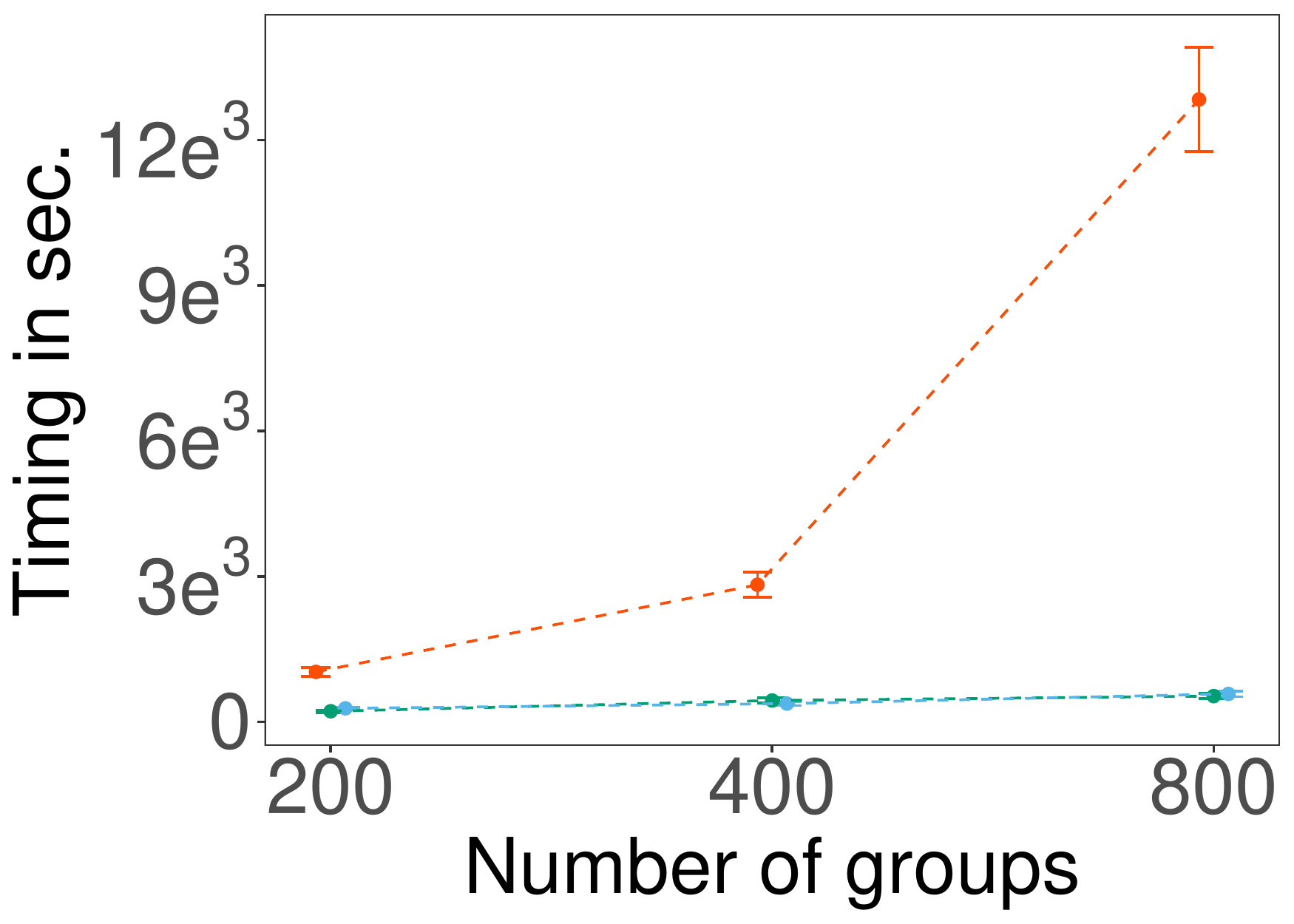}
\hfill
 \includegraphics[width=0.32\textwidth]{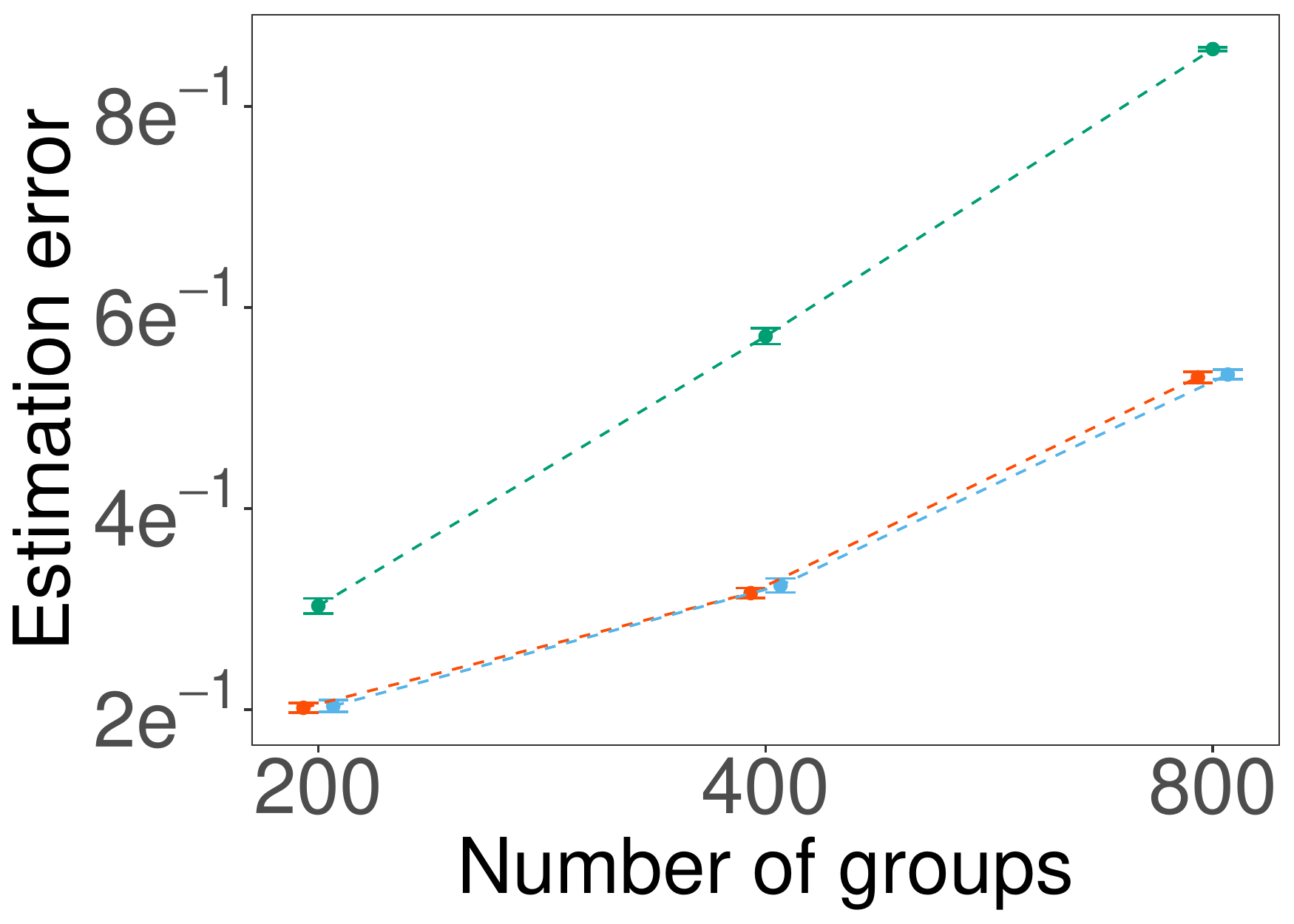}
  \includegraphics[width=0.32\textwidth]{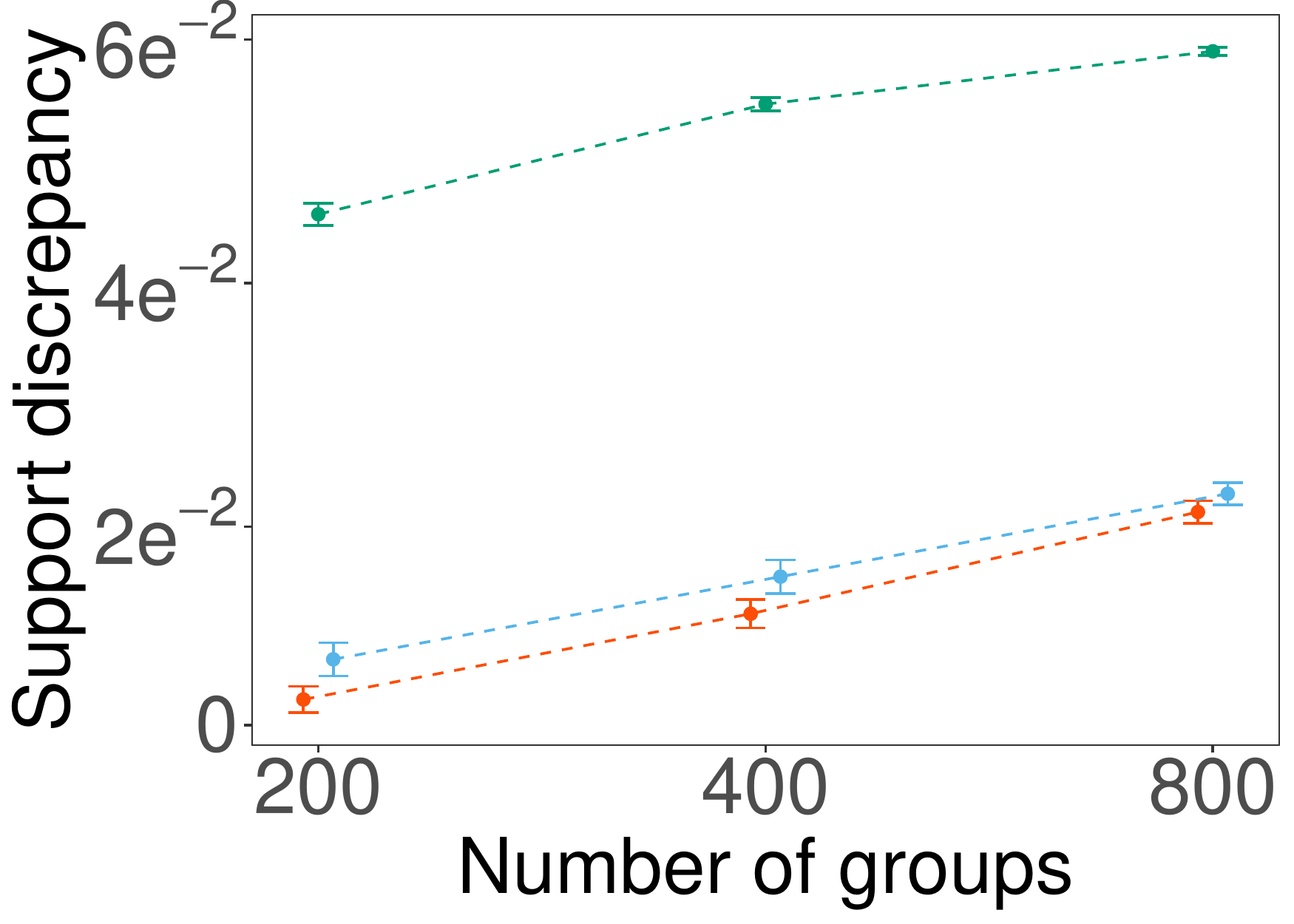}
   \caption{Performance vs. Number of groups}\label{fig:tnumg}
\end{subfigure}
\hfill
\begin{subfigure}[t]{\textwidth}
\centering
\includegraphics[width =0.32\textwidth]{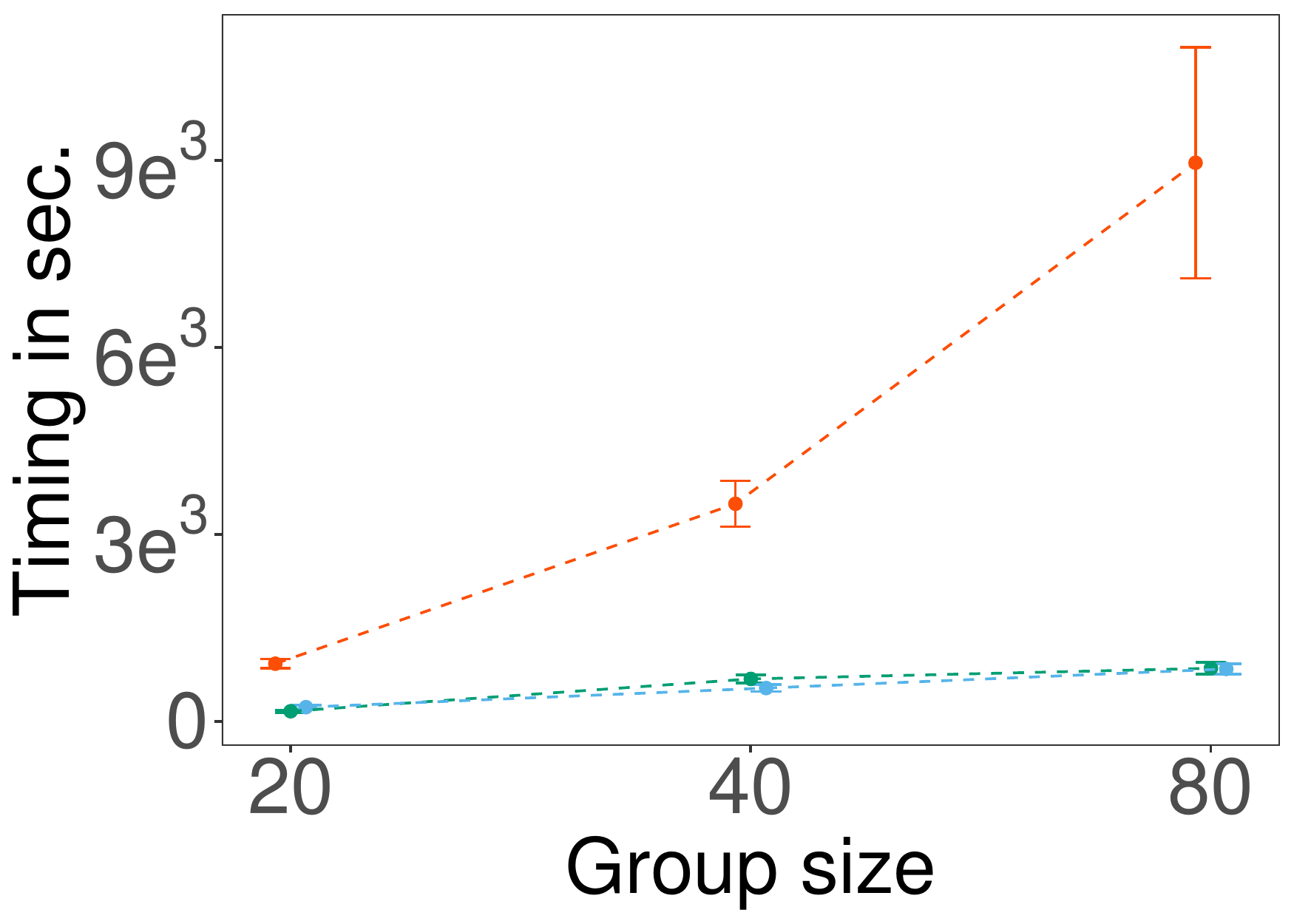}
\hfill
 \includegraphics[width=0.32\textwidth]{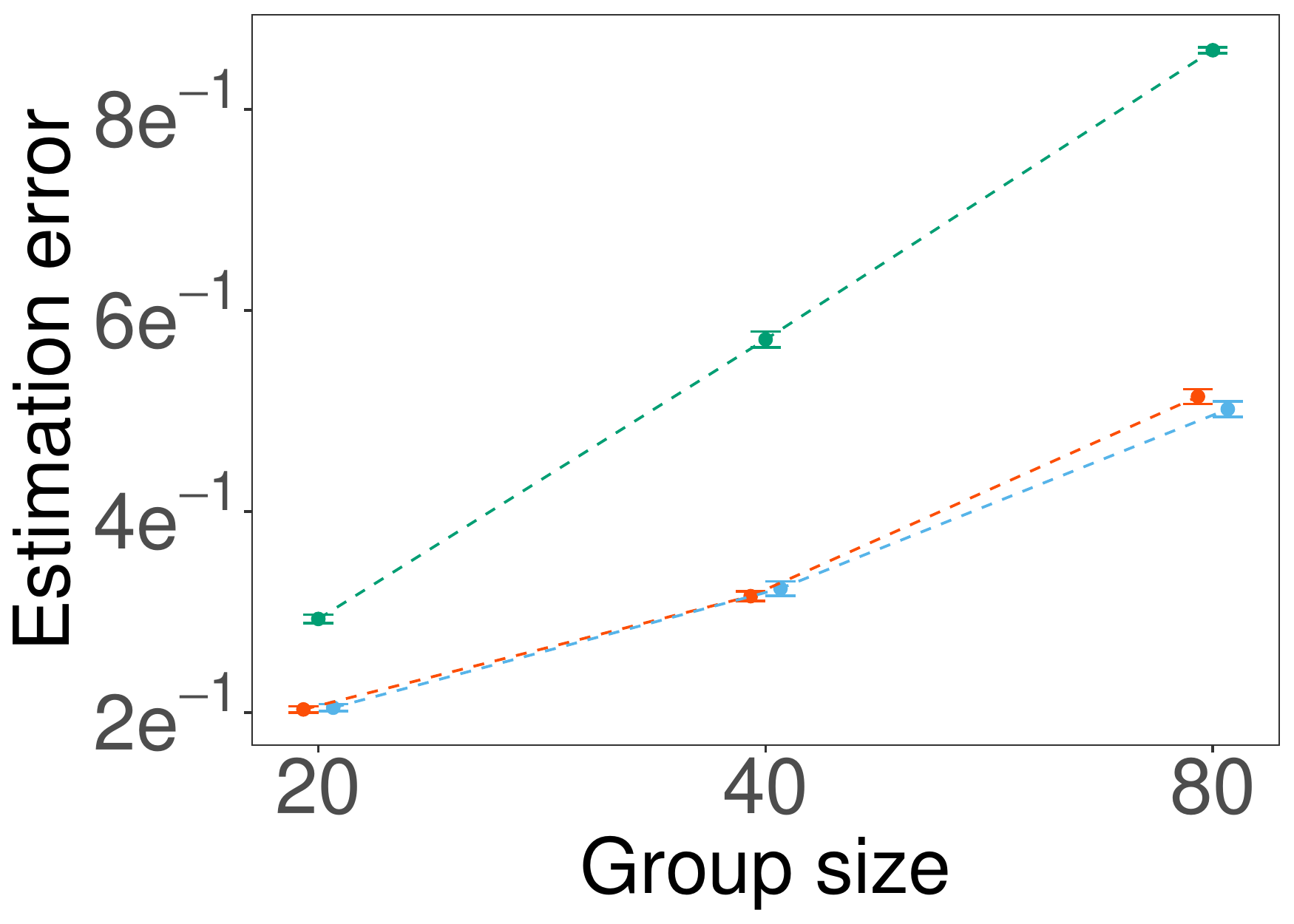}
  \includegraphics[width=0.32\textwidth]{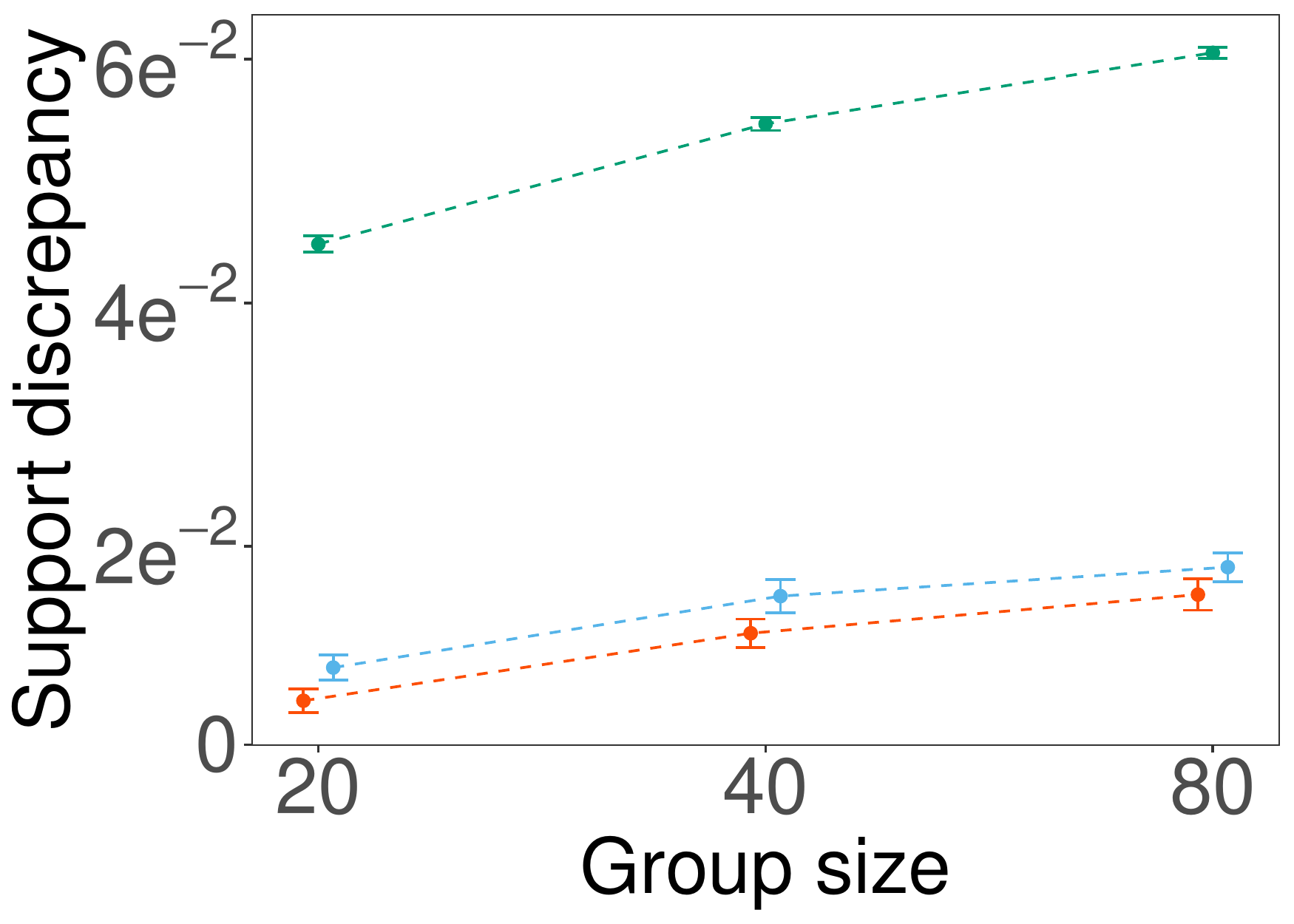}
   \caption{Performance vs. Groups size}\label{fig:tgrop}
   \includegraphics[width=\textwidth]{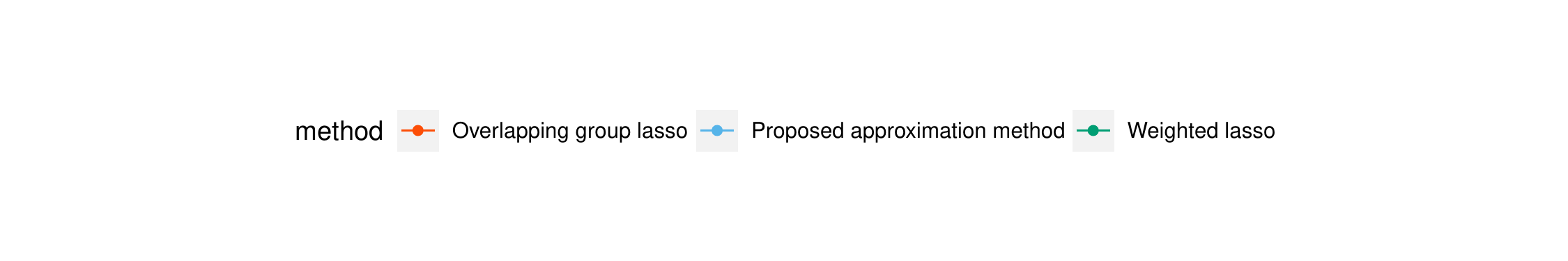}
\end{subfigure}
\vspace{-0.5cm}
\caption{Regularization path computing time, $\ell_2$ estimation error, and support discrepancy under different configurations of interlocking groups.  (a) Varying sample size $n$ when fixing $m = 400$ and $d = 40$ (p = 12808); (b) Varying number of groups $m$ when fixing $n = 4000$ and $d = 40$ ;  (c) Varying group size $d$ when fixing $n = 4000$ and $m = 400$. }
\label{fig:timecompare}
\end{figure}

Figure~\ref{fig:timecompare} presents the average computation times, estimation errors, and support discrepancy with 95\% confidence intervals (CIs). The result highlights the significant computational advantage of the proposed method over the original overlapping group lasso. Specifically, our method is 5--20 times faster than the original overlapping group lasso.

Even though the overlap is not severe within the interlocking group structure, solving the overlapping group lasso problem carries a more substantial computational burden due to the non-separable structure within its penalty term. The computational time increases with larger sample sizes, a greater number of variables, and larger group sizes, and the computational disadvantage of the overlapping group lasso is more substantial as the problem scales up. In contrast, our proposed method consistently achieves accuracy similar to the overlapping group lasso estimator in both the estimation error and support discrepancy. This consistency in performance, observed across a spectrum of configurations, serves as an empirical confirmation of the validity of our theoretical findings.

On the other hand, the weighted lasso approximation is slightly faster than our method. This is expected from the optimization perspective. However, the weighted lasso approximation exhibits much higher errors than the overlapping group lasso estimator and our estimator across all configurations, revealing that the weighted also gives a poor approximation to the overlapping group lasso. This is because the weighted lasso fails to leverage the group information, different from the induced groups $\Gcal$ used in our estimator.

In summary, our proposed estimator achieves comparable statistical performance to the original overlapping group lasso estimator while significantly enhancing computational efficiency. In contrast, although computationally efficient, the weighted lasso yields notably poor estimations, rendering it an uncompetitive alternative for approximating the original problems.

\subsection{Nested tree structure of overlapping groups}\label{secsec:sim-tree}

In this experiment, we evaluate the performance of the estimators under a configuration of the tree-group structures introduced in  \cite{jenatton2011proximal}, as below.

\begin{defi}\citep{jenatton2011proximal}
A set of groups $G = \{G_1,\cdots,G_m\}$ is said to be tree-structured in $[p]$ if $\cup_{g \in [m]}G_g = [p]$ and if for all $g,g' \in [m]$. $G_g \cap G_{g'} \neq \emptyset $ implies either $G_g \subset G_{g'}$ or $G_{g'} \subset G_g $.
\end{defi}
In particular, we consider the special case of the tree groups, the nested group structure where all groups are nested. This configuration is interesting as it represents an extreme setting of overlapping groups -- the overlapping degree is maximized in a certain sense and we hope to evaluate the methods in this extreme scenario. The nested group structure was also used in a few previous studies \citep{kim2012tree,nowakowski2023improving}. In this experiment, the SPAM solver, designed for the tree group structures, is also used to provide a more thorough evaluation across different implementations. We consider the following nested group configuration: \( 800 \) groups \( G = \{G_1, \ldots, G_{800}\} \) are established, where \( G_g \subset G_{g+1} \) and \( |G_g| = g \times 4 \), $g=1, \cdots, 800$ with $p=3200$ in total. The sample size varies from 600 to 2400. The data matrix \( X \) is generated from \( N(0, \Theta) \), where \( \Theta \) is generated by first constructing the matrix \( \tilde{\Theta} \) as
\[
\tilde{\Theta}_{ij} =
\begin{cases}
1, & \text{if } i = j, \\
0.6, & \text{if } \beta_i \text{ and } \beta_j \text{ belong to the same group in } \mathcal{G}, \\
0.36, & \text{if } \beta_i \text{ and } \beta_j \text{ are in the same group in } G \text{ but in different groups in } \mathcal{G},
\end{cases}.
\]
and then projecting  \( \tilde{\Theta} \) onto the set of symmetric positive definite matrices with minimum eigenvalue $0.1$. The generative process for $\beta^*$ and $y$ remains nearly identical as before,  where the only difference is that the first 90\%  of the groups are set to zero following the hierarchical structure. The group weights are set to \( w_g = 1/d_g \) as suggested \citep{nowakowski2023improving}. For a fair comparison of the two solvers, in this experiment, we adopt the stopping criterion provided in the SPAM package \citep{mairal2014spams} with a convergence tolerance $10^{-5}$.

\begin{figure}[H]
\centering
\begin{subfigure}[t]{\textwidth}
\centering
\includegraphics[width =0.32\textwidth]{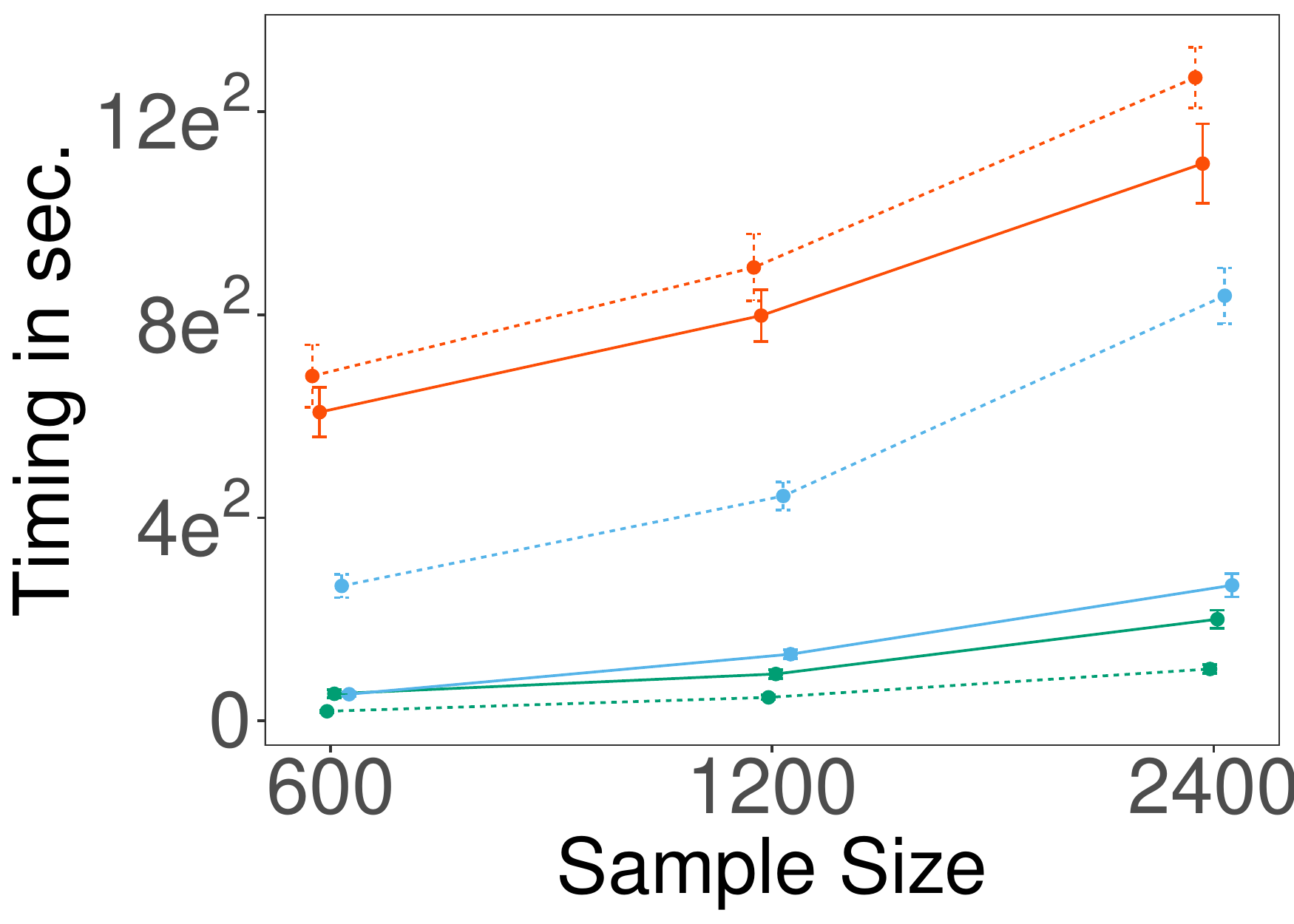}
\hfill
 \includegraphics[width=0.32\textwidth]{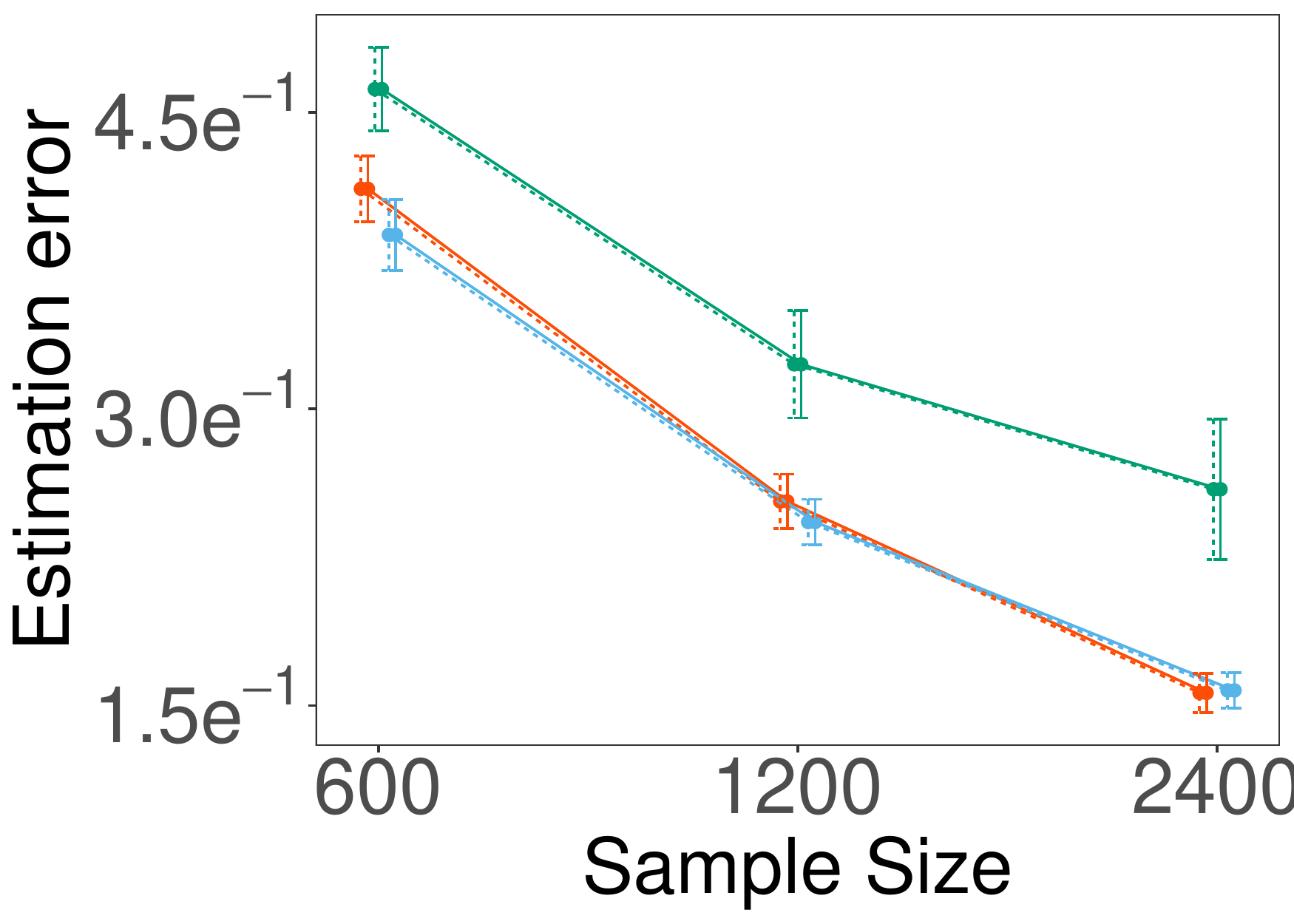}
  \includegraphics[width=0.32\textwidth]{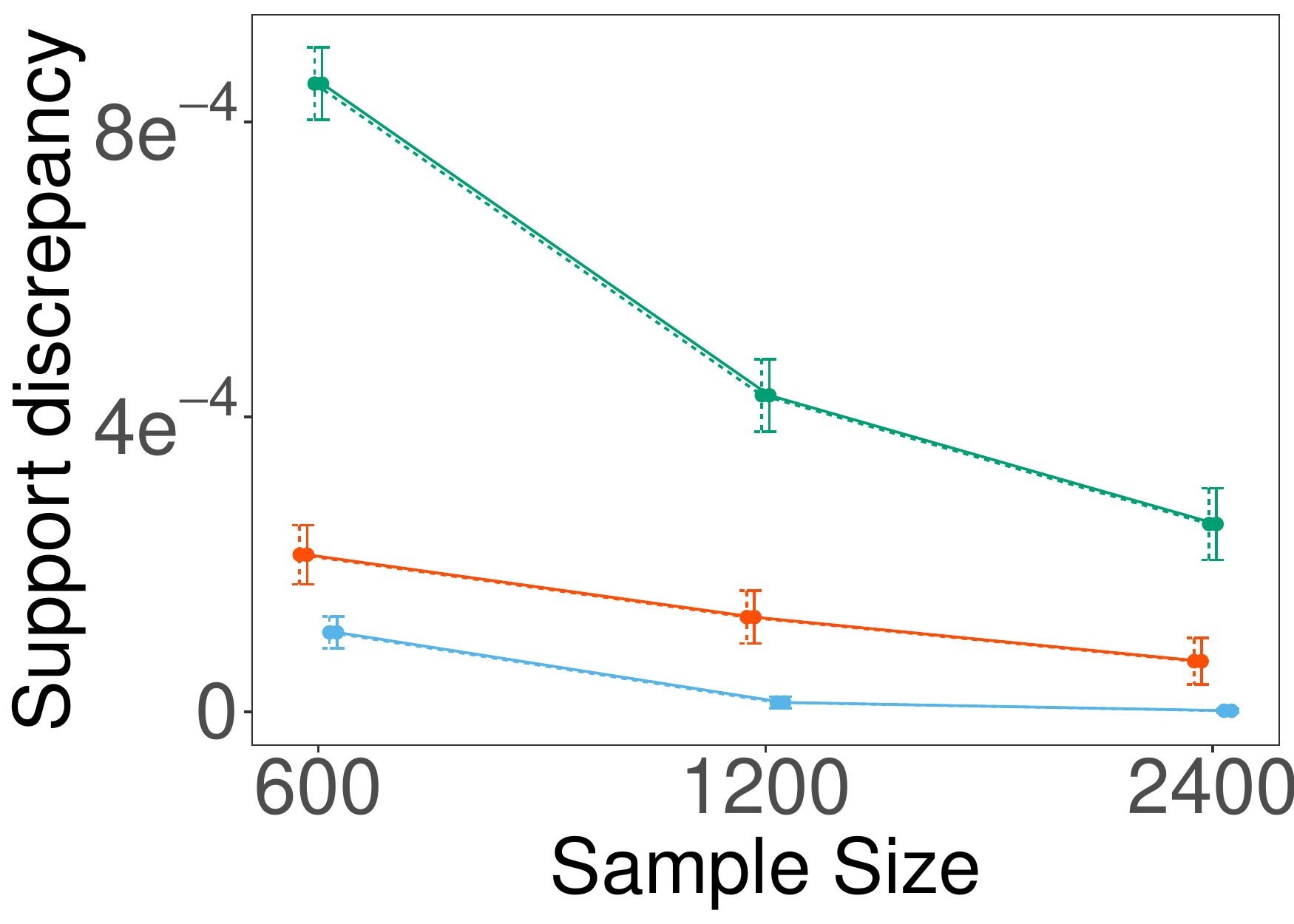}
\label{fig:tree}
   \includegraphics[width=\textwidth]{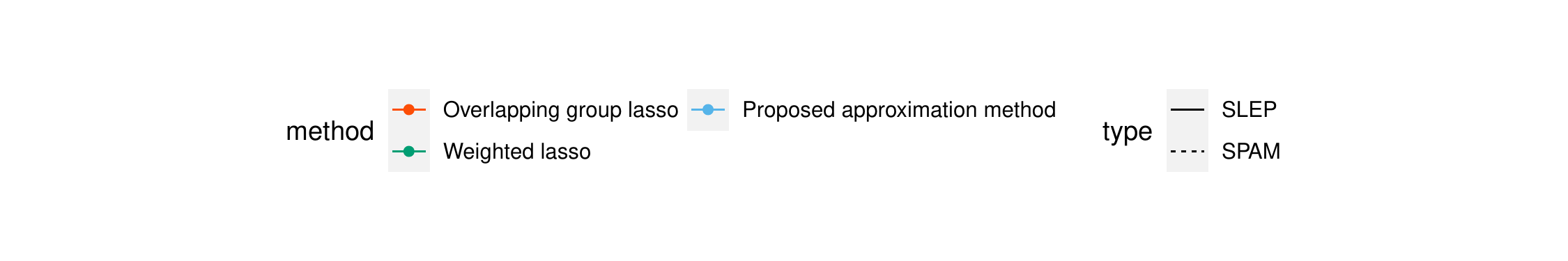}
\end{subfigure}
\vspace{-0.5cm}
\caption{Regularization path computing time, $\ell_2$ estimation error, and support discrepancy across various sample sizes under the nested tree group structure. }
\label{fig:timetreecompare}
\end{figure}

Figure~\ref{fig:timetreecompare} shows the performance of the three methods based on both solvers. SLEP is generally faster than SPAM, but the two solvers give consistent conclusions about the estimators. As studied by \cite{jenatton2011proximal}, solving the overlapping group lasso problem becomes highly efficient under such a nested group structure because, under a tree structure, a single iteration over all groups is adequate to obtain the exact solution of the proximal operator. Our timing results support this statement. Compared with the previous setting, the timing advantage of our method is reduced. However, our method is still at least twice as fast as the overlapping group lasso. When considering estimation error and support discrepancy, our proposed estimator consistently delivers similar results compared to the overlapping group lasso estimator. The comparison with the weighted lasso remains similar to the previous experiment; while the lasso estimator is also fast to compute, it delivers very poor approximation.

In summary, solving overlapping group lasso problems exhibits efficiency when applied to tree structures. However, even in such cases, our proposed estimator maintains reasonable computational advantage and similar statistical estimation performance compared to the original overlapping group lasso estimator. 

\subsection{Group structures based on real-world gene pathways}\label{secsec:sim-gene}

\begin{table}[!ht]
 \caption{\label{tab:pathways}Summary information for the gene pathways: the mean and standard deviation of both group size $(\bar{d}/\text{sd}(d))$ and the overlapping degree $(\bar{h}/\text{sd}(h))$, the number of genes ($p$), and the ratio required in Assumption~\ref{ass:compare-bound}.}
 \centering
 {
 \tabcolsep=5pt
 \begin{tabular}{|c|c|c|c|c|} 
  \hline
  Pathways                              & $\bar{d}/\mathrm{sd}(d)$ & $\bar{h}/\mathrm{sd}(h)$ & $p$    & \makecell[c]{$\max\{\mathcal{m},\mathcal{d}_{\max}\}/$ \\ $\max\{m,d_{\max}\}$}  \\ \hline
  BioCarta \citep{kong2006multivariate} & 15.4/ 8.71      & 3.25/ 5.56      & 1129 & 2.35  \\ \hline
  PID \citep{pid}                       & 38.51/ 19.59    & 3.28/ 5.09      & 2297 & 5.95  \\ \hline
  KEGG \citep{gkv1070}                  & 58.48/ 47.36    & 2.58/ 3.39      & 4207 & 3.61  \\ \hline
  WIKI \citep{wiki}                     & 38.17/ 44.10    & 4.35/ 7.70      & 6242 & 4.94  \\ \hline
  Reactome \citep{reac}                 & 45.31/ 54.10    & 8.78/ 13.26     & 8331 & 2.35 \\ \hline
\end{tabular}
\label{genesum}
}
\end{table}

The previous two sets of experiments are based on human-designed group structures. To reflect more realistic situations, in this set of experiments, we use five gene pathway sets from the Molecular Signatures Database \citep{msig} as group structures, summarized in Table~\ref{genesum}. Each gene pathway represents a collection of genes united by common biological characteristics. These pathways have been widely adopted in studies of cancer and biological mechanisms \citep{menashe2010pathway,ES,livshits2015pathway,chen2020kegg}.

In particular, this data set can be used to assess the empirical applicability of Assumption~\ref{ass:compare-bound} in our theory. The last column of Table~\ref{genesum} shows the ratio between $\max\{\mathcal{m},\mathcal{d}_{\max}\}$ and $\max\{m,d_{\max}\}$. All values are within the range of [2,6], indicating that the two terms can be treated as terms in the same order.

We use the gene expression data from  \cite{breast} as the covariate matrix $X$, which can be accessed through the R package \texttt{breastCancerNKI} \citep{bnki}. This design matrix has 295 observations and 24,481 genes. We perform gene filtering for each gene pathway set to exclude genes not defined within any pathways, a data processing step commonly used in similar studies \citep{ogl,ESS,chen2012smoothing}. The data-generating procedure for $\beta^*$ and $y$ remains almost the same as before, except that we use a much sparser model because of the smaller sample size of the data. Specifically, we randomly sample $0.05m$ active groups and set the coefficients in other groups to zero. The weights in overlapping group lasso are set to be $\sqrt{d_g}$.

\begin{table}[H]
\centering
\caption{Comparison of the average computing time (in seconds) and the corresponding 95\% confidence intervals for each pathway group structure.}
\tabcolsep=3pt
\begin{tabular}{|c|c|c|c|}
\hline
\makecell[c]{Group\\Structure} &      Overlapping group lasso       &          Weighted lasso      & The proposed approximation  \\ \hline
BioCarts        & \phantom{00}67.18 [\phantom{00}62.28, \phantom{00}72.08] & \phantom{00}6.22 [\phantom{00}5.99, \phantom{00}6.45] & \phantom{0}16.03 [\phantom{0}15.17, \phantom{0}16.89]   \\
KEGG            & \phantom{0}287.27 [\phantom{0}267.18, \phantom{0}307.36] & \phantom{0}28.77 [\phantom{0}26.42, \phantom{0}31.12] & \phantom{0}48.32 [\phantom{0}45.12, \phantom{0}51.52]   \\
PID             & \phantom{0}445.99 [\phantom{0}420.56, \phantom{0}471.42] & \phantom{0}10.27 [\phantom{00}9.74, \phantom{0}10.80] & \phantom{0}31.25 [\phantom{0}29.43, \phantom{0}33.07]   \\
WIKI            &           1279.22 [1214.34,                     1344.10] & \phantom{0}63.56 [\phantom{0}57.36, \phantom{0}69.76] & 132.79           [121.82,                     143.76] \\
Reactome        &           3739.97 [3569.27,                     3910.67] &           116.34 [106.32,                     126.36] & 194.61           [181.31,                     207.91] \\ \hline
\end{tabular}
\label{table:gene_time_comparison}
\end{table}

\begin{table}[ht]
\centering
\caption{Comparison of the relative $\ell_2$ estimation errors and the corresponding 95\% confidence intervals for each group structure. }
\tabcolsep=5pt
\begin{tabular}{|c|c|c|c|}
\hline
    Group Structure & Overlapping group lasso & Lasso & Proposed approximation \\ \hline
BioCarts & 0.22 [0.20, 0.24]           & 0.28           [0.24, 0.32] & 0.25 [0.22,           0.28] \\
KEGG     & 0.52 [0.47, 0.57]           & 0.80           [0.76, 0.84] & 0.54 [0.51,           0.57] \\
PID      & 0.23 [0.21, 0.25]           & 0.50           [0.44, 0.56] & 0.25 [0.23,           0.28] \\
WIKI     & 0.55 [0.49, 0.61]           & 0.65           [0.58, 0.72] & 0.55 [0.49,           0.61] \\
Reactome & 0.66 [0.63, 0.69]           & 0.85           [0.83, 0.87] & 0.65 [0.62,           0.68] \\ \hline
\end{tabular}
\label{table:gene_est_comparison}
\end{table}

\begin{table}[ht]
\centering
\caption{Comparison of the support discrepancy and the corresponding 95\% confidence intervals for each group structure. }
\tabcolsep=5pt
\begin{tabular}{|c|c|c|c|}
\hline
    Group Structure & Overlapping group lasso & Lasso & Proposed approximation \\ \hline
BioCarts & 0.041 [0.039, 0.043]           & 0.043           [0.040, 0.046] & 0.041 [0.039,           0.043] \\
KEGG     & 0.023 [0.021, 0.025]           & 0.026           [0.024, 0.028] & 0.023 [0.021,           0.025] \\
PID      & 0.033 [0.031, 0.035]              & 0.033 [0.031, 0.035] &  0.033 [0.031, 0.035] \\
WIKI     & 0.013 [0.012, 0.014]           &  0.013 [0.011, 0.015]  &  0.013 [0.012, 0.014]  \\
Reactome  & 0.012 [0.011, 0.013]           &  0.020 [0.019, 0.021]  &  0.012 [0.010, 0.014] \\ \hline
\end{tabular}
\label{table:gene_pat_comparison}
\end{table}

Table~\ref{table:gene_time_comparison} displays the computing time, and Table~\ref{table:gene_est_comparison}  displays the estimation error results for the five pathway group structures. The high-level message remains consistent. Both our proposed group lasso approximation and the lasso approximation could substantially reduce the computing time. Across all settings, the proposed method reduces the computation time by 4 - 20 times and is more than 10 times faster in all settings with higher dimensions. Meanwhile, the proposed estimator delivers statistical performance similar to that of the original overlapping group lasso estimator. In contrast, the lasso approximation fails to leverage the group information effectively and yields inferior estimation results.

\subsection{Comparison of the proposed weights against other weighting choices}\label{secsec:sim-weight}

In addition to the partitioned groups, the overlapping-based weight defined in \eqref{eq:weight} for each partitioned group $\mathcal{g}$ is another crucial component to ensure the tightness of \eqref{eq:our-norm}. We will demonstrate this aspect by experiments here to compare the proposed weights \eqref{eq:weight}  with two other commonly used choices of weights that do not consider the original overlapping pattern: the uniform weights and group size-dependent weights \citep{yuan2006model}, on the same induced groups $\Gcal$. Specifically, uniform weighting is the setting when all groups share the same weight while the size-dependent weighting uses the weight $\sqrt{\mathcal{d}_\mathcal{g}}$ if $w_g = \sqrt{d_g}$ (interlocking and gene pathway groups) and is $1/d_\mathcal{g}$ if $w_g = 1/d_g$ (nested groups). The comparative analysis is performed under all group structures in the previous simulations, maintaining consistent simulation settings.

\begin{figure}[H]
\centering
\begin{subfigure}[t]{\textwidth}
\centering
 \includegraphics[width=0.32\textwidth]{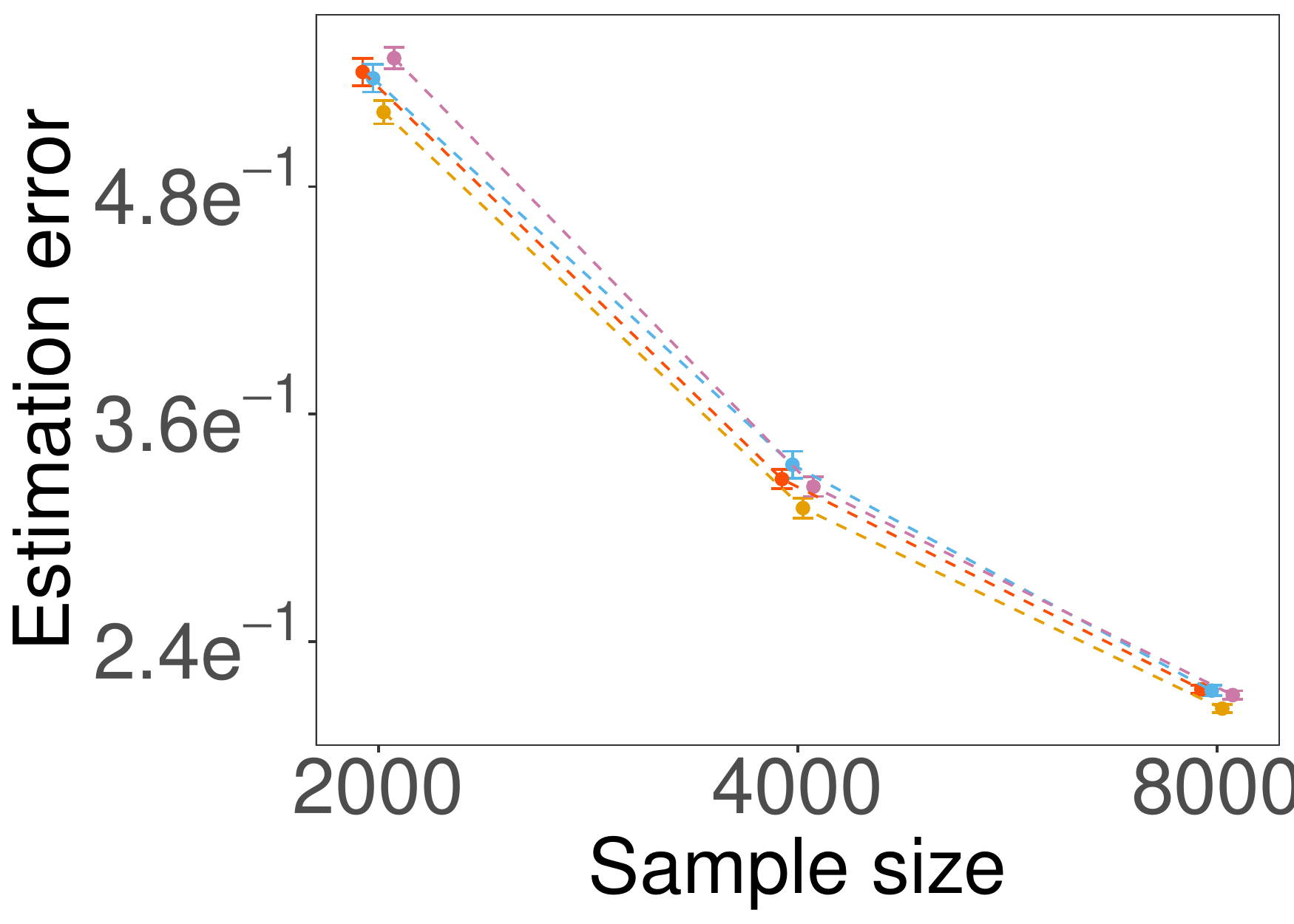}
  \includegraphics[width=0.32\textwidth]{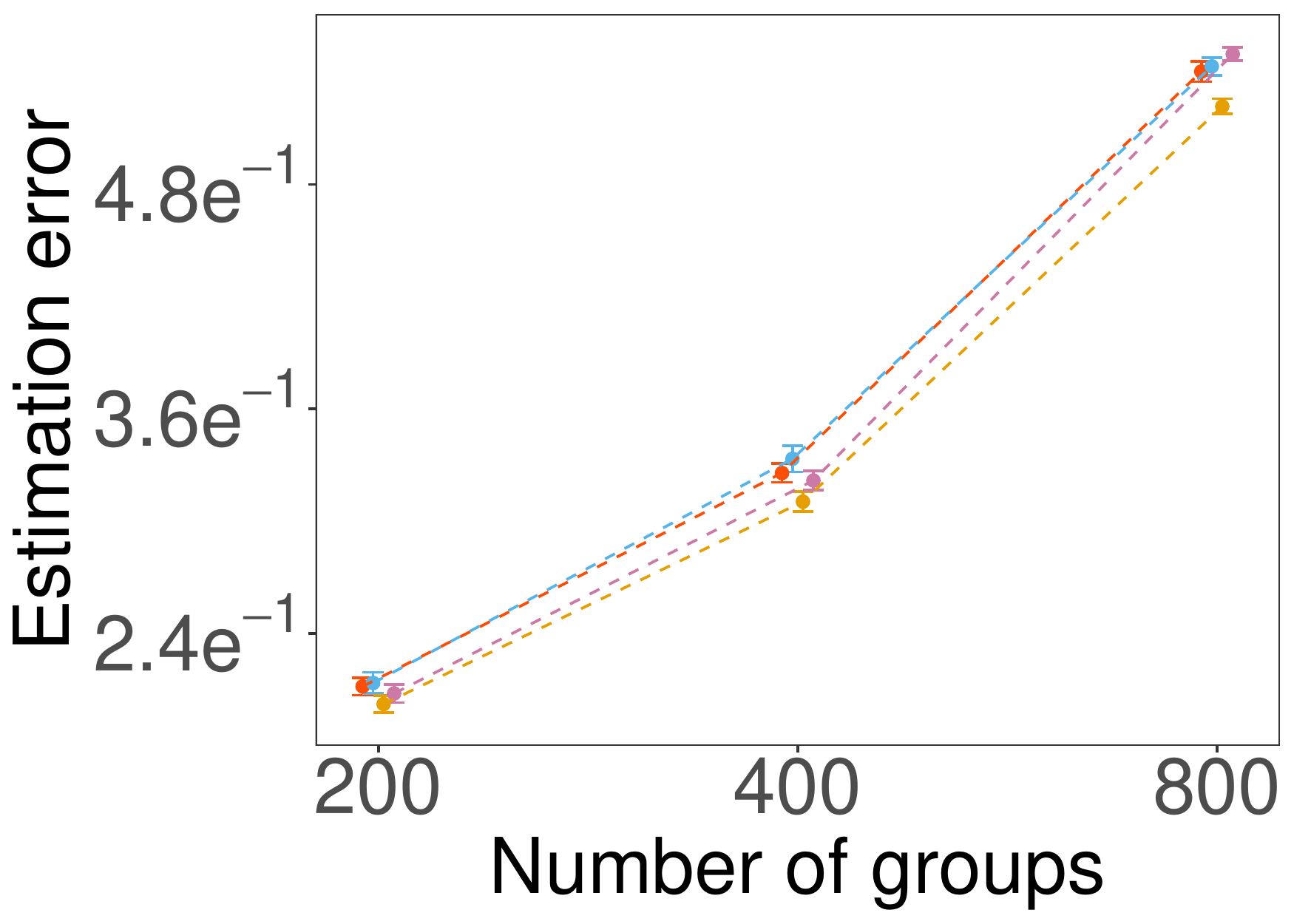}
   \includegraphics[width=0.32\textwidth]{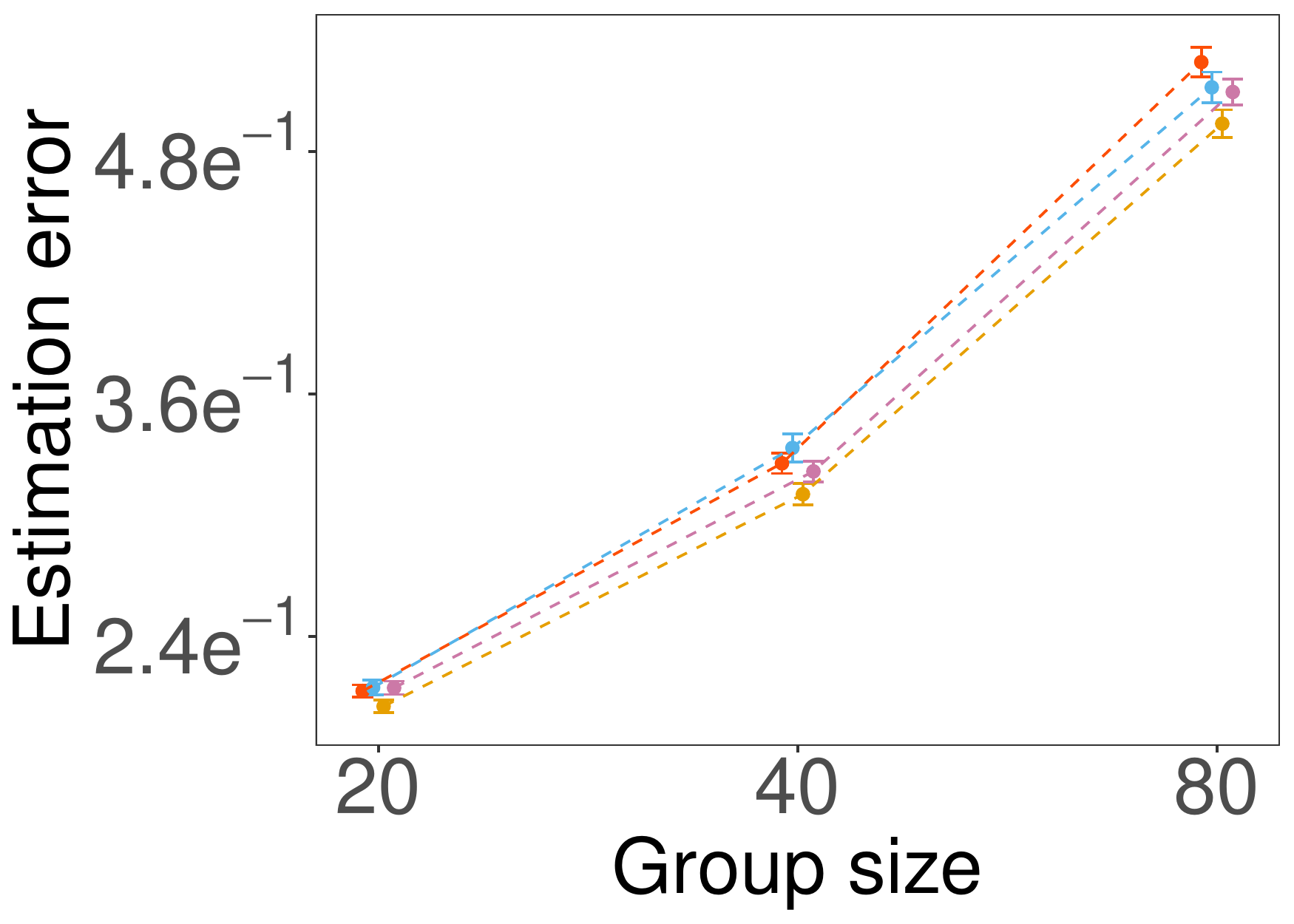}
\end{subfigure}
\begin{subfigure}[t]{\textwidth}
\centering
 \includegraphics[width=0.32\textwidth]{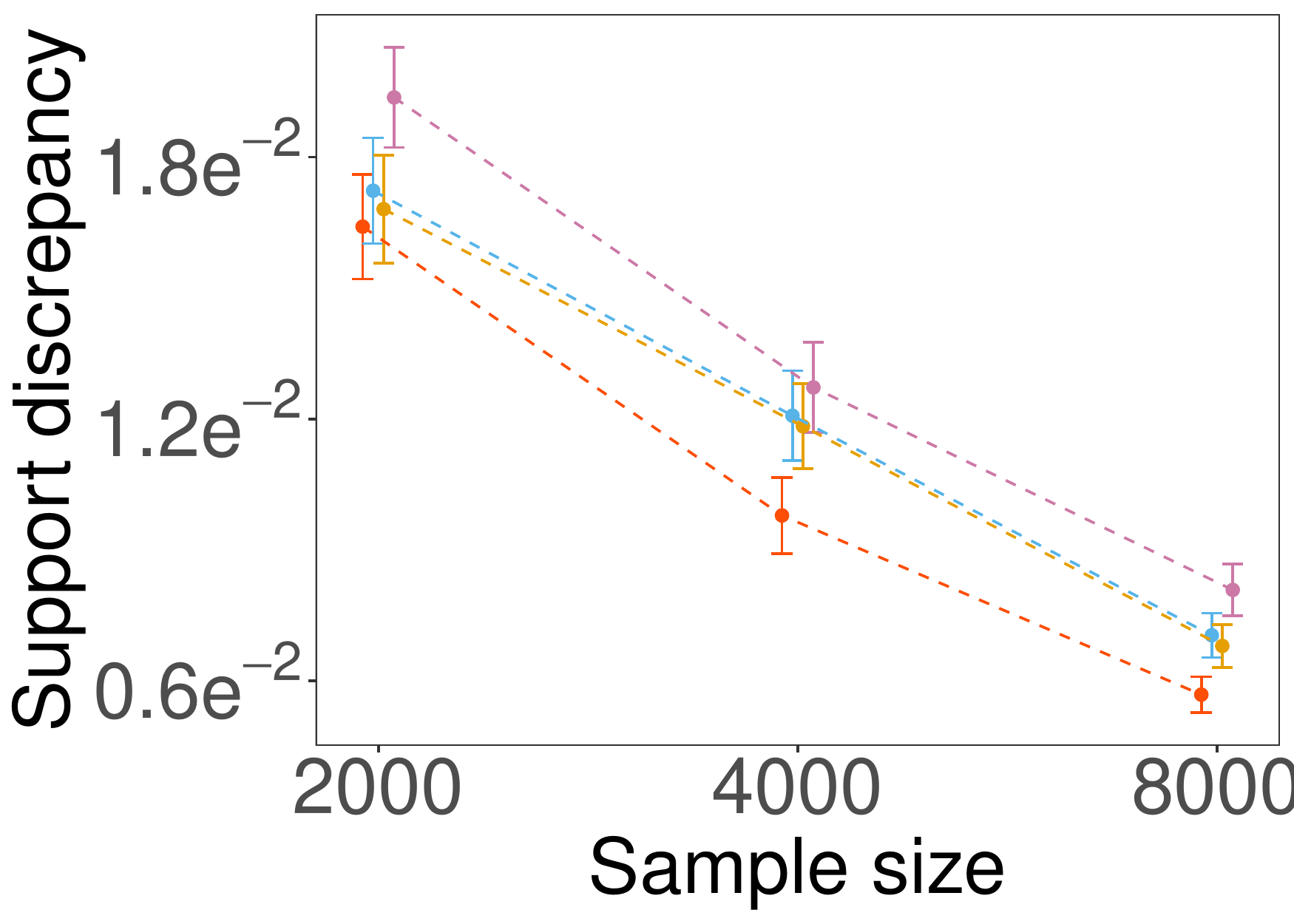}
  \includegraphics[width=0.32\textwidth]{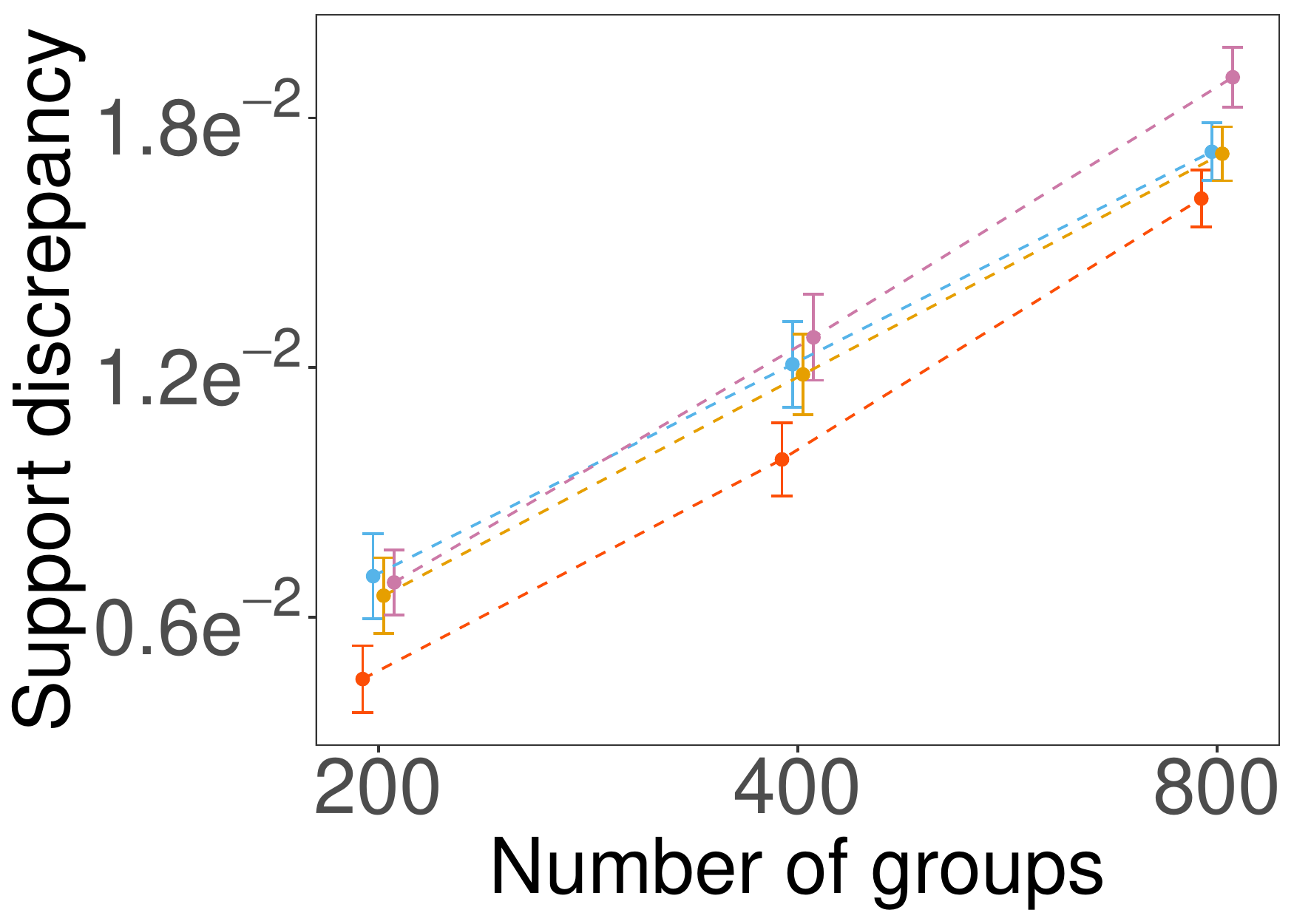}
  \includegraphics[width =0.32\textwidth]{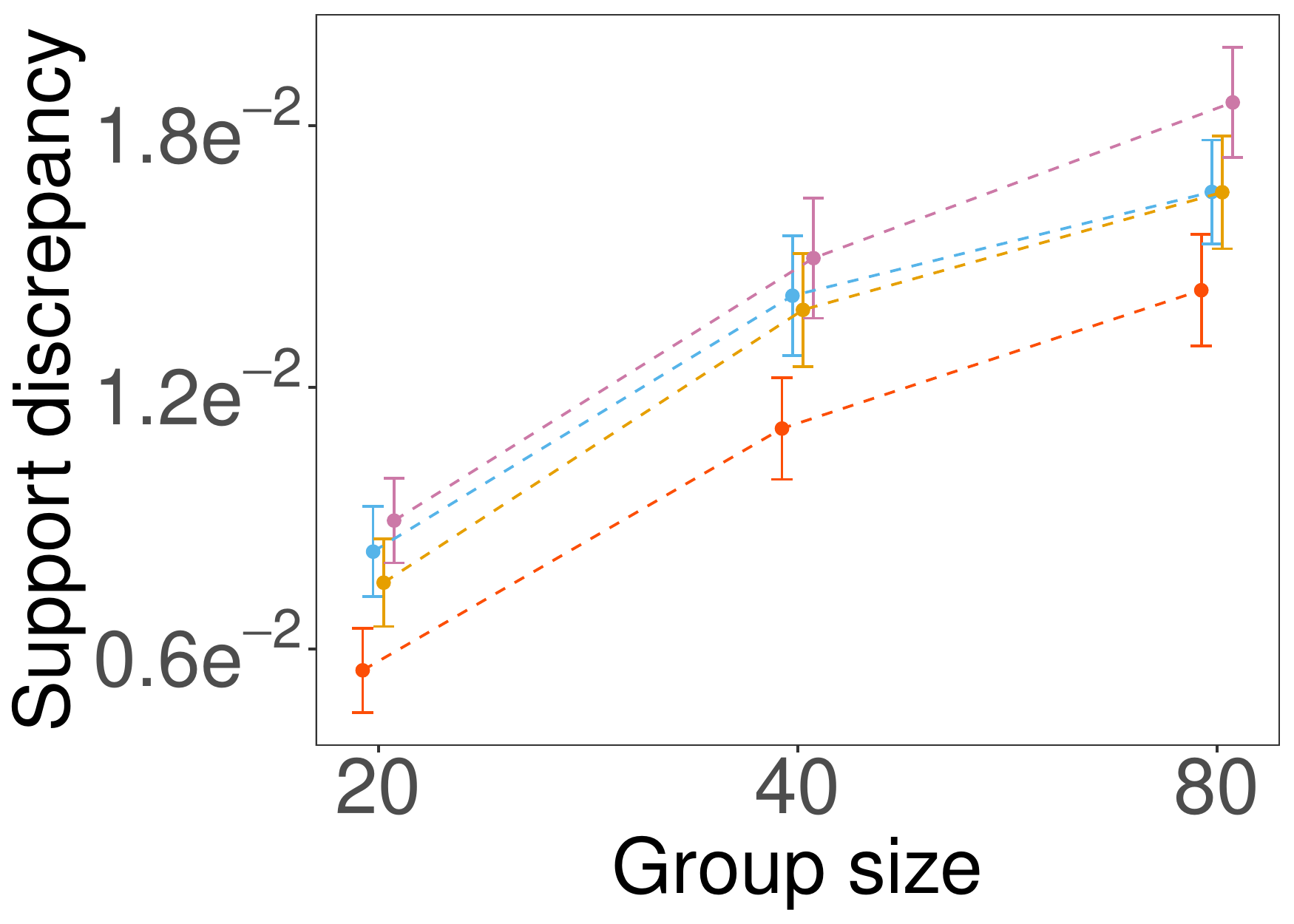}
  \caption{Performance  under interlocking group structure}
\label{fig:wtsam}
\end{subfigure}

\begin{subfigure}[t]{\textwidth}
\centering
 \includegraphics[width=0.32\textwidth]{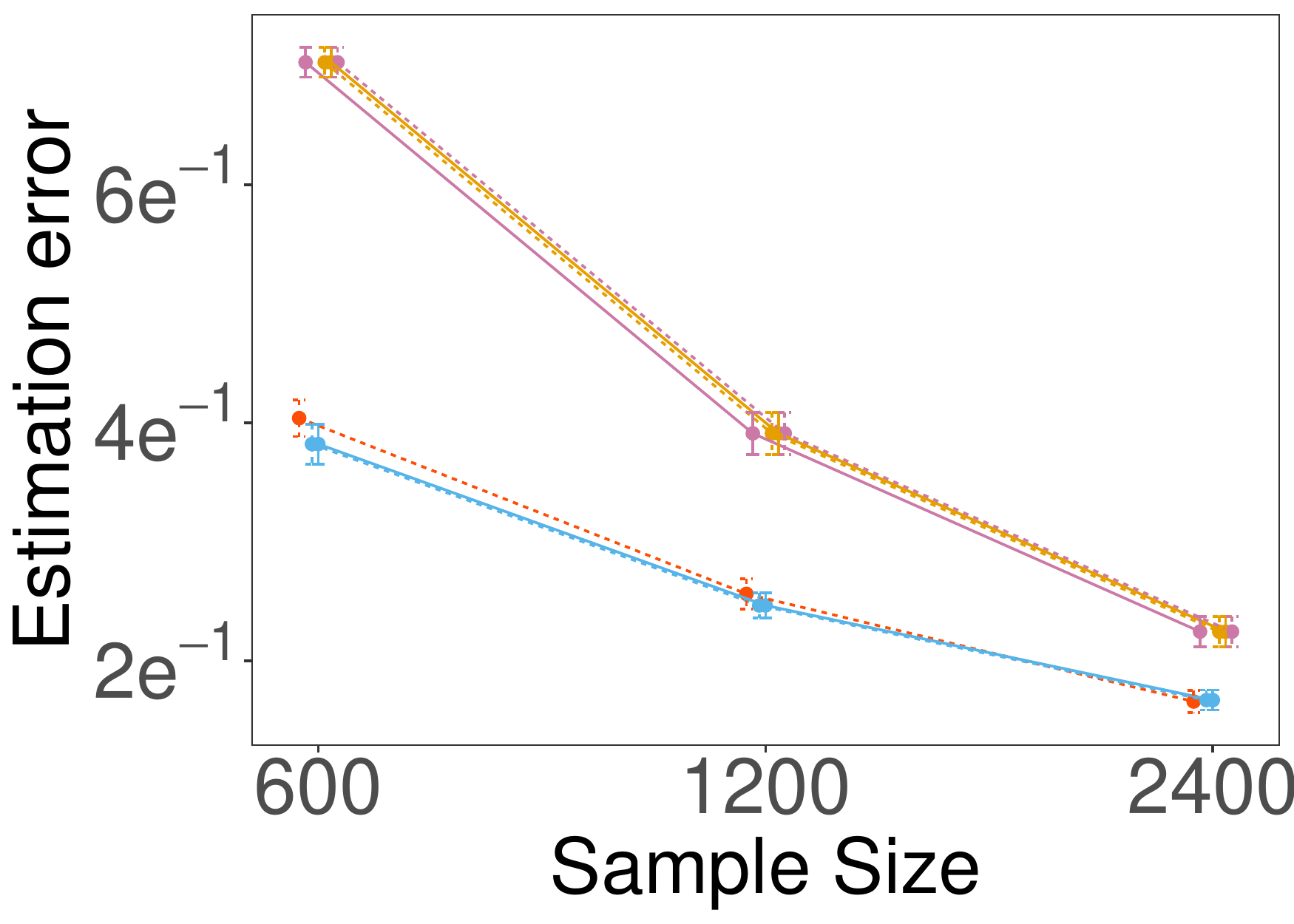}
  \includegraphics[width=0.32\textwidth]{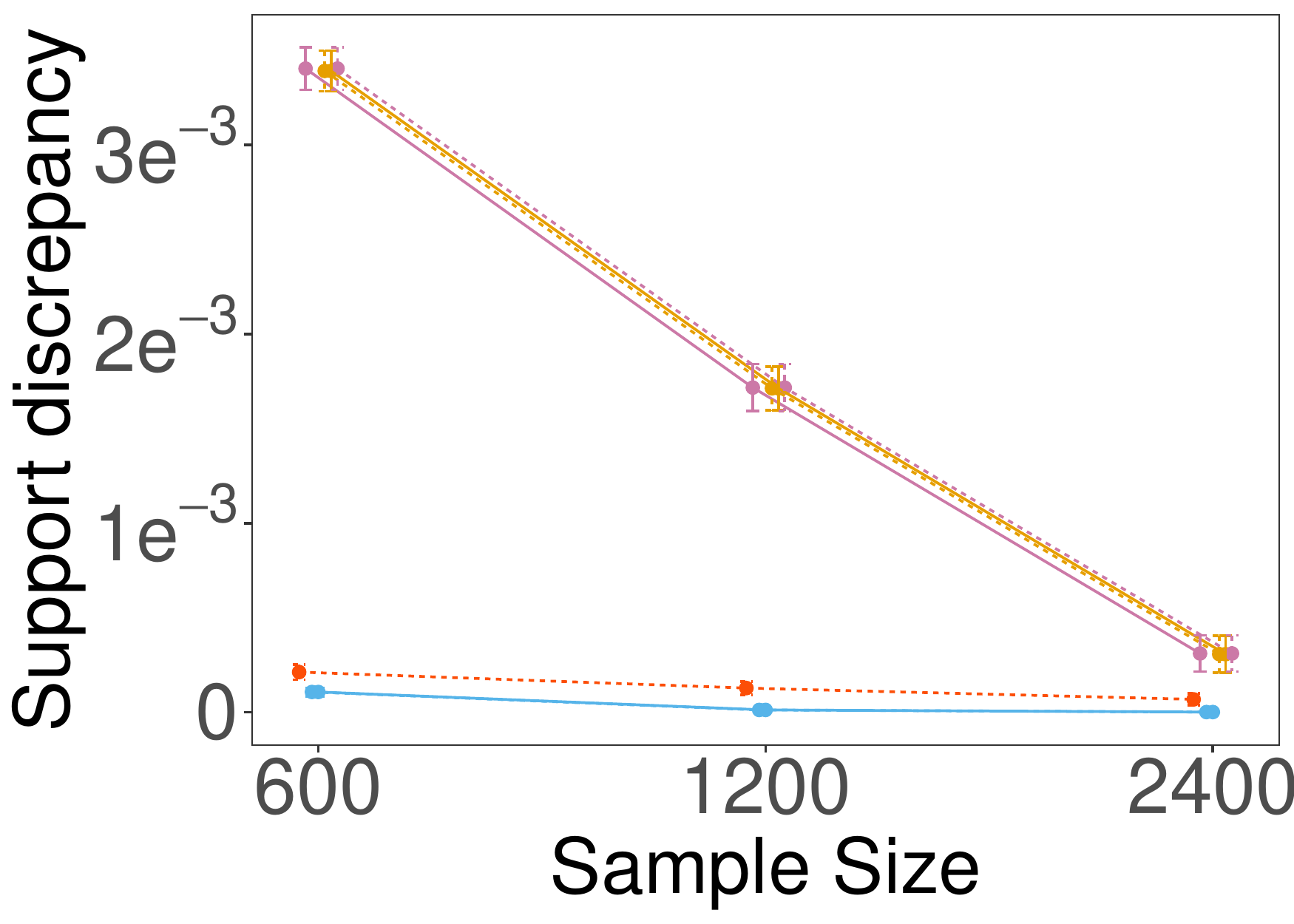}
\caption{Performance vs. Sample size under nested tree structure}
\label{fig:wtsamtree}
\end{subfigure}
\includegraphics[width=\textwidth]{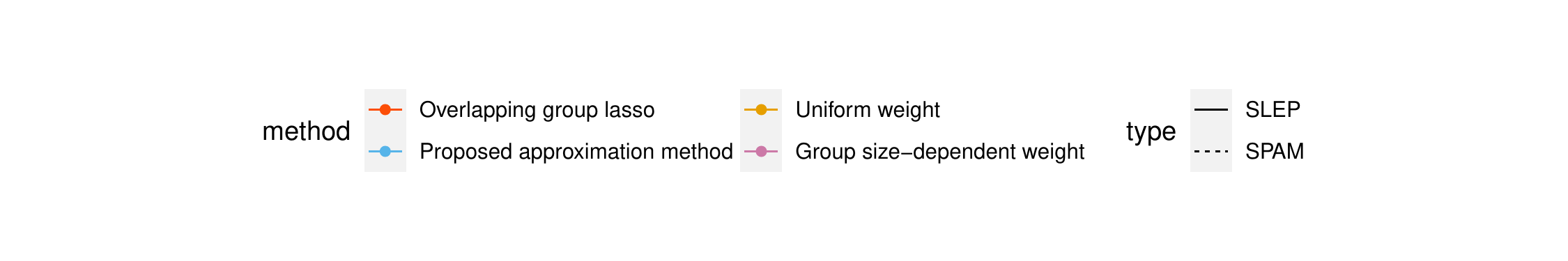}
\vspace{-0.5cm}
\caption{Regularization $\ell_2$ estimation error and support discrepancy of the proposed method using different choices of weights under interlocking group structure and nested tree structure.  Figure \ref{fig:wtsam} is an extension to Figure \ref{fig:timecompare}, and Figure \ref{fig:wtsamtree}  is an extension to Figure \ref{fig:timetreecompare}. }

\label{fig:weightcompare}
\end{figure}

Figure~\ref{fig:wtsam} and Figure~\ref{fig:wtsamtree} illustrate the weigh effects comparison in the settings of Figure~\ref{fig:timecompare} and Figure~\ref{fig:timetreecompare}, respectively. Under the interlocking group structure (Figure \ref{fig:wtsam}), three weighting themes deliver similar performance in terms of estimation errors. Still, the size-dependent weighting leads to a larger support discrepancy. This interlocking group structure is not very distinctive for the three weights themes because the overlapping degree is nearly uniform. The nested group structures (Figure \ref{fig:wtsamtree}) more effectively highlight the importance of the proposed weights. Our method significantly outperforms the other two weighting themes and aligns well with the original overlapping group lasso estimator. The weight design comparison on the gene pathway group structure is shown in Tables \ref{table:gene_west_comparison}--\ref{table:gene_wpat_comparison}. The proposed estimator gives a close approximation to the original overlapping group lasso, but the other two weighing designs lead to significantly different performances in several settings.

\begin{table}[ht]
\centering
\caption{Comparative analysis of average estimation errors and the corresponding 95\% confidence intervals for three weighting designs. The $*$ indicates that the error is statistically different from that of overlapping group lasso by a paired t-test.}
\tabcolsep=14pt
\begin{tabular}{|c|c|c|c|}
\hline
  \makecell[c]{Group\\Structure}  &   Proposed weight       &      Uniform weight    & \makecell[c]{Group size-\\dependent weight}  \\ \hline
BioCarts            & 0.25 [0.22, 0.28]           & 0.28          [0.26, 0.30]* & 0.35    [0.30,           0.40]* \\
KEGG        & 0.54 [0.51, 0.57]           & 0.80           [0.77, 0.83]* & 0.58 [0.51,           0.65]* \\
PID           & 0.25 [0.23, 0.27]           & 0.24           [0.21, 0.27]\phantom{0} & 0.39 [0.36,           0.42]* \\
WIKI      & 0.55 [0.49, 0.61]           & 0.83           [0.80, 0.86]* & 0.74 [0.67,           0.81]* \\
Reactome  & 0.65 [0.62, 0.68]           & 0.58          [0.55, 0.61]* & 0.69 [0.63,           0.75]\phantom{0}\\ \hline
\end{tabular}
\label{table:gene_west_comparison}
\end{table}

\begin{table}[ht]
\centering
\caption{Comparative analysis of average support discrepancy and the corresponding 95\% confidence intervals for three weighting designs. The $*$ indicates that the value is statistically different from that of overlapping group lasso by a paired t-test.}
\tabcolsep=11pt
\begin{tabular}{|c|c|c|c|c|}
\hline
  \makecell[c]{Group\\Structure}  &     Proposed weight       &      Uniform weight    & \makecell[c]{Group size-\\dependent weight}  \\ \hline
BioCarts   & 0.041 [0.039, 0.043]           & 0.045 [0.042, 0.048]* & 0.042 [0.039,           0.045]\phantom{0} \\
KEGG     &0.023 [0.021, 0.025]            & 0.059           [0.055, 0.063]* & 0.024 [0.022, 0.026]\phantom{0} \\
PID         & 0.033 [0.031, 0.035]  &   0.037 [0.035, 0.039]*    & 0.030 [0.027, 0.033]* \\
WIKI            &  0.013 [0.012, 0.014]  &  0.025 [0.023, 0.027]* &  0.013 [0.012, 0.014]\phantom{0}  \\
Reactome        &  0.012 [0.010, 0.014]  &  0.010 [0.008, 0.012]* &  0.022 [0.021, 0.023]* \\ \hline
\end{tabular}
\label{table:gene_wpat_comparison}
\end{table}

In summary, the experiments demonstrate that the weights designed in our penalty also serve as an indispensable part of a successful approximation to the overlapping group lasso estimation, which is another aspect of the tightest separable relaxation property in Theorem~\ref{addtheo}.

\section{Application Example: Pathway Analysis of Breast Cancer Data}
\label{sec:data}

In this section, we demonstrate the proposed method by predictive tasks on the breast cancer tumor data, as previously used in Section~\ref{secsec:sim-gene}. This time, unlike the previous simulation studies, we use the complete data set with tumor labels for each observation. Specifically, each observation is labeled according to the status of the breast cancer tumors, with 79 classified as metastatic and 216 as non-metastatic. These labels serve as the response variable for our analysis.

Gene pathways have been widely used to key gene groups in cancer studies. In particular, \cite{ES, chen2012smoothing,ESS} used the overlapping group lasso techniques to exclude less significant biological pathways in cancer prediction. As a detailed example, \cite{chen2012smoothing} leveraged the overlapping group lasso penalty to pinpoint biologically meaningful gene groups. Their analysis revealed multiple groups of genes associated with essential biological functions, such as protease activity, protease inhibitors, nicotine, and nicotinamide metabolism, which turned out to be important breast cancer markers \citep{ma2010detection}. This evidence highlights the potential of using the overlapping group lasso penalty in cancer analysis. On the other hand, another way to incorporate gene pathway information in such analysis is to retain genes by entire pathways. \cite{ogl} used the latent overlapping group lasso penalty to achieve this while \cite{mairal2013supervised} introduced an $\ell_{\infty}$ variant further. The success of all these previous studies reveals the potential of the gene pathway information in cancer prediction. They also show that the proper way to use the pathways (e.g., either eliminating-by-group, as in overlapping group lasso, or including-by-group, as in latent overlapping group lasso) highly depends on the data set and genes. 

In our analysis, we use regularized logistic regression to build a classifier with the overlapping group lasso penalty (OGL), our proposed group lasso approximation penalty (Proposed approximation), the standard lasso penalty, the latent overlapping group lasso penalty (LOG) \citep{ogl}, and the $\ell_{\infty}$ latent overlapping group lasso penalty of \citep{mairal2013supervised}.  As mentioned in previous sections, our focus is not on justifying the overlapping group lasso should be used. Instead, \textbf{our primary objective is to demonstrate that when an overlapping group lasso penalty is used, our method provides a good approximation to the overlapping group lasso (with a much faster computation) across various pathway sets} (Table~\ref{tab:pathways}), whether or not the overlapping group lasso penalty is the best option for the problem. 

Two additional aspects can also be evaluated as by-products of our analysis. First, as the lasso penalty does not consider the pathway information, comparing the performance of the group-based penalty and the lasso penalty in this problem would verify whether a specific gene pathway set contains predictive grouping information for breast cancer tumor type. Second, by assessing the predictive performances among the overlapping group lasso classifier and the latent overlapping group lasso classifiers, we can verify whether a specific gene pathway set is more suitable for eliminating-by-group or including-by-group strategies for prediction.

\begin{table}[H]
    \centering
   \caption{\label{table:time} Computing time (in seconds) under different pathway databases.}
    \setlength{\tabcolsep}{10pt}
    {%
    \begin{tabular}{|p{.15\textwidth}|c|c|c|}
        \hline
                \diagbox[innerwidth = .15\textwidth]{Database}{Method} & \makecell[l]{OGL} & Lasso  & \makecell[l]{Proposed approximation} \\
            \hline
            BioCarts & 732 & 26 & 75  \\
            \hline KEGG & 2468 & 102 & 225 \\
            \hline PID & 1231 & 41 & 107  \\
            \hline WIKI & 5172 & 170 & 395 \\
            \hline Reactome & 11356 & 321  & 1186 \\
        \hline
    \end{tabular}}
\end{table}

We adopt the evaluation procedure of \cite{ESS}, where we randomly split the data set into 200 training observations and 95 test observations. All methods are tuned by 5-fold cross-validation on the training data. We calculate the area under the receiver operating characteristic (AUC) curve, a commonly used metric for classifying accuracy \citep{hanley1982meaning}, on the test data. The total time for the entire cross-validation process is recorded as computation time. The experiment is repeated 100 times independently. Table~\ref{table:time} and Table~\ref{table:auc} show the average computing time and AUC, respectively.

The following can be summarized from the results:
\begin{itemize}
\item First and foremost, the proposed estimator acts as an effective and computationally efficient approximation for the overlapping group lasso estimator. The results evidently support this claim. The proposed estimator delivers predictive performance that is (the most) similar to the overlapping group lasso estimator across various pathway datasets while significantly reducing the computing time by roughly ten times. 
\item Second, the lasso classifier performs best only on the WIKI pathway set, suggesting that the pathways in the WIKI database might not be sufficiently informative for cancer prediction.  
\item Third, the superiority between the overlapping group lasso regularizations and the latent overlapping group lasso regularizations depends on the specific group information. Among the four pathway sets with useful group information, the overlapping group lasso delivers superior predictive performance for the Biocarts and PID databases, while the latent overlapping group lasso classifiers provide better predictions on the KEGG and Reactome databases.
\end{itemize}

\begin{table}[H]
    \centering
   \caption{\label{table:auc} Predictive AUC results of the three methods under different pathway databases. }
    {%
    \setlength{\tabcolsep}{10pt}    
    \begin{tabular}{|p{.15\textwidth}|c|c|c|c|c|}
        \hline
        \diagbox[innerwidth = .15\textwidth]{Database}{Method} & \makecell[l]{OGL} & Lasso  & \makecell[c]{Proposed \\ approximation} & \makecell[l]{LOG} &  \makecell[l]{LOG $\infty$}  \\
            \hline BioCarts & 0.7103\phantom{0} & 0.6989           & 0.7242\phantom{0}      & 0.6888           & 0.6995 \\
            \hline KEGG     & 0.7021\phantom{0} & 0.6862           & 0.7081\phantom{0} & 0.7390          & 0.7333 \\
            \hline PID      & 0.7475\phantom{0}      & 0.7004           & 0.7301\phantom{0} & 0.6881           & 0.6891 \\
            \hline WIKI     & 0.6862\phantom{0} & 0.7282           & 0.6893\phantom{0} & 0.7149           & 0.7207 \\
            \hline Reactome & 0.6921\phantom{0} & 0.7301           & 0.7053\phantom{0} & 0.7463           & 0.7438 \\
        \hline
    \end{tabular}}
\end{table}

As a remark, while our evaluation is based on prediction accuracy, it is not the only criterion to determine if a method is proper for the dataset. For example, \cite{mairal2013supervised} found that neither the overlapping group lasso model nor the latent overlapping group lasso model outperformed simple ridge regularization in prediction. The value of structured penalties also lies in their ability to identify potentially more interpretable genes, depending on the biological interpretations.

\section{Discussion}\label{sec:disc}
We have introduced a separable penalty as an approximation to the group lasso penalty when groups overlap. The penalty is designed by partitioning the original overlapping groups into disjoint subgroups and reweighing the new groups according to the original overlapping pattern. The penalty is the tightest separable relaxation of the overlapping group lasso among all $\ell_{q_1}/\ell_{q_2}$ norms. We have also shown that for linear problems, the proposed estimator is statistically equivalent to the original overlapping group lasso estimator but enjoys significantly faster computation for large-scale problems.

Several interesting directions could be considered for future research. The overlapping group lasso penalty presents a variable selection by eliminating variables by entire groups. A counterpart selection procedure can include variables by entire groups, which is achieved by the latent overlapping group lasso \citep{ogl}. This penalty also suffers from a non-separability computational bottleneck. It would be valuable to investigate whether a similar approximation strategy could be designed to boost the computational performance in this scenario. More generally, the introduced concept of ``tightest separable relaxation" might be a promising direction for optimizing non-separable functions. Studying the more general form and corresponding properties of this concept may generate fundamental insights about optimization.


\acks{The work is supported in part by the NSF grant DMS-2015298 and the 3-Caveliers award from the University of Virginia. The authors acknowledge the Minnesota Supercomputing Institute (MSI) at the University of Minnesota and the Research Computing at The University of Virginia for providing resources that contributed to the research results reported within this paper. We appreciate the insightful feedback and comments from the editor and reviewer, which significantly improved the paper.}


\bibliography{sample}
\appendix
\section{Notation summary}\label{sec:notations}

\begin{table}[H]
{\small
\caption{\label{tab:notations}Mathematical notations in the paper.}
\begin{tabularx}{\textwidth}{p{0.22\textwidth}X}
    \toprule
  \multicolumn{2}{l}{{\noindent Indices:}}                                       \\
  $[z]$ & index set $\{1,...,z\}$ \\ 
   $G_g$ & index set  of $g^{th}$ group\\
    $G_S$ & collection of non-zero groups, $\mathop{\bigcup}_{g\in S(\beta)}G_g $\\
     $G_{\overline{S}}$ &  $\mathop{\bigcup}_{g\in \overline{S(\beta)}}G_g $\\
   $\beta_j$ & the $j^{th}$ element of $\beta$\\
    $\beta_{G_g}$ & sub-vector of $\beta$ indexed by $G_g$\\
    $\beta_{M(S)}$ & projection of $\beta$ onto $M(S)$\\
      $A_{,T}$ & sub-matrix  consisting of the columns indexed by  T \\[0.2in]
      \multicolumn{2}{l}{{\noindent Parameters:}}                                       \\
       $H$ &  a diagonal matrix, $diag(\frac{1}{h_1},\cdots,\frac{1}{h_p})$\\
  $\mathbf{G}$ &  group structure matrix, $\mathbf{G}_{gj} = 1$ iff $\beta_j \in G_g$\\
  $d_g$ & group size,  $d_g = \sum_{j\in[p]}\mathbf{G}_{gj}$ \\
  $d_{\max}$ &  maximum group size, $d_{\max} = \max_{g \in [m]}d_g$  \\
  $h_j$ & overlap degree, $h_j = \sum_{g \in [m]}\mathbf{G}_{gj}$ \\
  $h_{\max}^g$  & maximum overlap degree in $G_g$, $h_{\max}^g= \max_{j \in G_g}h_j$ \\
    $h_{\min}^{g}$ & minimum overlap degree in $G_g,$ $\min_{j \in G_g}h_j$\\
  $h_{\max}$ & maximum overlap degree, $h_{\max} = \max_{j \in [p]}h_j$\\
  $\mathcal{h}_{\mathcal{g}}$& overlap degree of $\mathcal{G}_\mathcal{g},$ $h_{\text{\{}j|j \in \mathcal{G}_{\mathcal{g}}\text{\}}}$\\
  $\sigma$   & parameter in the sub-Gaussian distribution  \\
  $s_g$ &        number of non-zero groups  $|S|$ \\
  $\overline{s_g}$ & number of groups in the argument group support set $|\overline{S}|$\\
 $\kappa$  & parameter controls convexity \\[0.2in]
  \multicolumn{2}{l}{{\noindent Definitions:}}                                       \\
  $\phi(\beta)$   & group lasso norm, $\sum\limits_{g \in [m]}w_g \left\|\beta_{G_g}\right\|_2,$  \\
    $\phi^{\ast}(\beta)$ & dual norm of $\phi(\beta),$ $\max\limits_{g\in[m]}\frac{1}{w_g}\left\|(H\beta)_{G_g}\right\|_2$\\
  $F(\mathcal{g}) \subseteq [m]$  &  overlapping groups which include the variables in $\mathcal{G}_\mathcal{g}$  \\
  $F^{-1}(g) \subseteq [\mathcal{m}]$  &   non-overlapping groups that were partitioned from $G_g$\\
  $||\beta_{\{G,w\}}||_{q_1,q_2}$ & $\ell_{q_1,q_2}$ norm, $ \left\{ \sum\limits_{g \in [m]} w_g\left( \sum\limits_{j \in G_g}|\beta_j|^{q_2}\right)^{\frac{q_1}{q_2}}\right\}^{\frac{1}{q_1}} $ \\
$\textit{supp}(\beta)$ & support set, $\{j\in\{1,\cdots,p\}|\beta_j\neq 0\}$\\
$S(\beta)$ &  group support set, $\left\{g\in \{1,\cdots,m\}|G_g \cap \textit{supp}(\beta)\neq 
\varnothing\right\}$\\
$\overline{S(\beta)}$ &  $\{g=\{1,\cdots,m\}|G_g\cap
{G_{S(\beta)}\neq \varnothing}\}$\\
$M(S)$ &$\left\{\beta\in{\mathbb{R}^p}|\beta_j=0 \text{ for all } j \in (G_{S})^c\right\}$\\
$M^{\perp}(S)$ & $\left\{\beta\in \mathbb{R}^p |\beta_j=0 \text{ for all } j\in G_{S}\right\}$\\ 
$\Omega(G,s_g)$ & $\big\{\beta : \sum\limits_{G_g \in G}\mathbbm{1}_{\{\norm{\beta_{G_g}}_2 \neq 0\} } \leqslant s_g\big\}$\\
$J_G(\beta) $ & $[p] \backslash \big\{\bigcup_{G_g \cap \textit{supp}(\beta) = \emptyset} G_g\big\}.$\\
$\mathsf{G}_{J_G(\beta)}$ & $\left\{g\in [m] \mid G_g  \cap J_G(\beta) \neq \emptyset \right\}$\\
$\mathsf{G}_{J_G(\beta)^c}$ & $\left\{g \in [m] \mid G_g \cap J_G(\beta)^c \neq \emptyset \right\}$\\
  \bottomrule
  \end{tabularx}
  }
\end{table}
\newpage

\section{Uniqueness of the overlapping group lasso problem}\label{app:uniqueness}

The group lasso penalization problems \eqref{eq:OGL-est} and \eqref{eq:Our-est} are generally convex, but may not be strictly convex. The uniqueness of these problems has been studied by \cite{svsw}. Here we introduce their results for completeness. Note that our theoretical properties in Section~\ref{sec:results} do not rely on such uniqueness. 

\begin{lem}(Proposition 1 of \cite{svsw})
\label{unique}
 If the gram matrix $Q = X^{\top}X/n$ is invertible, or if there exists $g \in [m]$ such that $G_g = [p]$, then the optimization problem specified in \eqref{eq:OGL-est}, with $\lambda_n > 0$, is guaranteed to have a unique solution. The same holds for problem \eqref{eq:Our-est} with $G$ replaced by $\Gcal$.
\end{lem}

\section{Additional Theoretical Results}

To begin with, we introduce our proposed upper bound for the dual norm of the overlapping group lasso.

\begin{pro}
\label{pro1}
The sharp upper bound for $\phi^{\ast}$ (the dual norm of overlapping group lasso penalty in  \eqref{eq:glnorm}) is 
 \begin{equation*} 
     \max\limits_{g\in[m]}\frac{1}{w_g}\left\|(H\beta)_{G_g}\right\|_2,
 \end{equation*}
 where $H$ is a diagonal matrix with diagonals $(\frac{1}{h_1},\cdots,\frac{1}{h_p})$.
\end{pro}

 \begin{assu}\label{ass:distr}
Under model \eqref{eq:linear-model}, we assume 
\begin{enumerate}
\item (Sub-Gaussian noises) The coordinates of $\varepsilon$ are i.i.d zero mean sub-Gaussian random variable denote with parameter $\sigma$, which means that there exist $\sigma > 0$ such that
$$ E[e^{t\varepsilon})] \leqslant \frac{e^{\sigma^2 t^2}}{2}, \hspace{0.2cm} \text{for all} \hspace{0.2cm} t \in \mathbb{R}.$$
\item (Group normalization condition) $ \sqrt{\gamma_{\max}(\frac{{X_{G_g}}^TX_{G_g}}{n} })\leqslant c$ for some constant c. 
\item (Restricted strong convexity condition) For some $\kappa >0$,  $$\frac{\left\|X\left(\bar{\beta}-\beta^{\ast}\right)\right\|_2^2}{n}\geqslant \kappa\left\|\bar{\beta}-\beta^{\ast}\right\|_2^2, \hspace{0.2cm}  \text{for all}\hspace{0.2cm} \bar{\beta}\in\left\{\beta \mid \phi \left(\left(\beta-\beta^{\ast}\right)_{M^{\perp}(\overline{S})}\right) \leqslant 3\phi\left(\left(\beta-\beta^{\ast}\right)_{M(\overline{S})}\right)\right\}.$$
\end{enumerate}
 \end{assu}

\textbf{Remark:} The assumption requires an upper bound for the quadratic form associated with each group. This type of assumption is commonly used for developing the upper estimation error bound for non-overlapping group lasso  \citep{oiao,tbog,aebf,aufh,wainwright2019high}. Additionally, the restricted curvature conditions have been well discussed by \citet{wainwright2019high}. The curvature $\kappa$ in Assumption~\ref{ass:distr} is a parameter measuring the convexity. Generally speaking, the restricted curvature conditions state the loss function is locally strongly convex in a neighborhood of ground truth and thus guarantees that a small distance between the estimate and the true parameter implies the closeness in the loss function.  However, such a strong convexity condition cannot hold in the high-dimensional setting. So, we focus on a restrictive set of estimates. Restricted curvature conditions are milder than the group-based RIP conditions used in \citep{tbog,aebf}, which require that all submatrices up to a certain size are close to isometries \citep{wainwright2019high}. Based on Assumption~\ref{ass:distr}, Theorem~\ref{thm:two-bound} gives $\ell_2$ norm estimation upper error bound for overlapping group lasso.

\begin{theorem}
\label{thm:two-bound}
Define $h_{\min}^{g} = \min\limits_{j \in G_g}h_j$, $d_{\max} = \max\limits_{g \in [m]}d_g$, and $\dcal_{\max} = \max\limits_{\gcal \in [\mcal]}\dcal_g$. Suppose Assumption~\ref{ass:distr} holds, for any $\delta \in [0,1]$,
\begin{enumerate}
\item   with $\lambda_n =  \frac{8c \sigma}{\min\limits_{g\in[m]}\left(w_g^2 h_{\min}^{g}\right)}\sqrt{\frac{d_{\max}\log{5}}{n}+ 
	\frac{\log{m}}{n}+\delta}$, the following  bound hold for $\hat{\beta}^G$ in \eqref{eq:OGL-est} 
\begin{equation}
\label{eq:el2bound}
    \left\|\hat{\beta}^{G}-\beta^{\ast}\right\|^2_2\lesssim\frac{ \sigma^2}{\kappa^2}\cdot\frac{\left(\sum\limits_{g\in\overline{S}}{w_g}^2\right)\cdot h^{G_{\overline{S}}}_{\max}}{\min\limits_{g\in[m]}\left(w_g^2 h_{\min}^{g}\right)}
\cdot\left(\frac{d_{\max}{\log{5}}}{n}+\frac{\log{m}}{n}+\delta\right).
\end{equation}
with probability at least $1-e^{-2n\delta}$.
\item  with $
\lambda_n =  \frac{8c\sigma}{\min\limits_{g\in[m]} \wcal_g}\sqrt{\frac{\dcal_{\max}\log{5}}{n}+ 
	\frac{\log{m}}{n}+\delta}$, the following  bound hold for $\hat{\beta}^{\Gcal}$ in \eqref{eq:Our-est}
\begin{equation}
\label{xxq1}
   \left\|\hat{\beta}^{\mathcal{G}}-\beta^{\ast}\right\|_2^2 \lesssim\frac{ \sigma^2}{\kappa^2}\cdot\frac{\sum\limits_{\mathcal{g}\in F^{-1}(S) }{\mathcal{w}_{\mathcal{g}}}^2}{\min\limits_{\mathcal{g}\in[\mathcal{m}]}\left(w_\mathcal{g}^2 \right)}
\cdot\left(\frac{\mathcal{d}_{\max}{\log{5}}}{n}+\frac{\log{\mathcal{m}}}{n}+\delta\right).
\end{equation}
\end{enumerate}
\end{theorem}

Following the framework in \citep{aufh,wainwright2019high}, we further study the applicability of the restricted curvature conditions in terms of a random design matrix.
Given a group structure $G$, Theorem~\ref{thm:two-bound} is developed based on the assumption that the fixed design matrix X satisfies the restricted curvature condition.  In practice, verifying that a given design matrix $X$ satisfies this condition is difficult. Indeed, developing methods to “certify” design matrices this way is one line of ongoing research \citep{wainwright2019high}. However, it is possible to give high-probability results based on the following assumptions. 
%

\begin{theorem}
\label{the:rsccondtion}
Under Assumptions~\ref{ass:distribution_sub_noise},\ref{ass:distribution_normal_rd}, and \ref{ass:distribution_groupstructure}, we have

\begin{enumerate}
\item  With probability at least $1 - e^{-c'n}$, $\max_{g\in [m]}\sqrt{ \gamma_{\max}(\frac{X_{G_g}^TX_{G_g}}{n}) } \leqslant c $ for some constants $c,c'>0$, as long as $\log m = o(n)$.
\item   The restricted strong convexity condition, which is  
 $$\frac{\left\|X\left(\bar{\beta}-\beta^{\ast}\right)\right\|_2^2}{n}\geqslant \kappa\left\|\bar{\beta}-\beta^{\ast}\right\|_2^2, \hspace{0.2cm}  \text{for all}\hspace{0.2cm} \bar{\beta}\in\left\{\beta \mid \phi \left(\left(\beta-\beta^{\ast}\right)_{M^{\perp}(\overline{S})}\right) \leqslant 3\phi\left(\left(\beta-\beta^{\ast}\right)_{M(\overline{S})}\right)\right\}.$$
hold with probability at least $1- \frac{e^{-\frac{n}{32}}}{1-e^{-\frac{n}{64}}}$ for some constant $\kappa > 0$.
\end{enumerate}
\end{theorem}

\newpage

\section{Proofs}\label{app:proof}

\subsection{Proof of Theorem \ref{addtheo}}
\begin{lemma}
\label{lem:tightbounds}
Consider a norm $||\cdot_{\{\tilde{G},\tilde{w}\}}||_{q_1,q_2}$ satisfying the conditions of Equation \eqref{twocondi}. The following two statements hold:
\begin{enumerate}
    \item For any $\mathcal{g} \in [\mathcal{m}]$, there exists a $\tilde{g} \in [|\tilde{G}|]$ such that $\mathcal{G}_{\mathcal{g}} \subseteq \tilde{G}_{\tilde{g}}$.
    \item For any $\tilde{g} \in [|\tilde{G}|]$, there exists a $\mathcal{g} \in [\mathcal{m}]$ such that $\tilde{G}_{\tilde{g}} = \mathcal{G}_{\mathcal{g}}$.
\end{enumerate}
\end{lemma}

\begin{proof}
Based on Lemma~\ref{lem:tightbounds}, if a norm $||\beta_{\{\tilde{G},\tilde{w}\}}||_{q_1,q_2}$ exists that satisfies Equation \eqref{twocondi}, then it must be that $\tilde{G} = \mathcal{G}$. Consequently, any disparity between $||\beta_{\{\tilde{G},\tilde{w}\}}||_{q_1,q_2}$ and our proposed norm could only be due to differences in weights or the values of $q_1$ or $q_2$.

For any $\beta$ with non-zero elements solely in the $\mathcal{g}$th group $\mathcal{G}_\mathcal{g}$, we have:
\begin{equation}
    \label{eq:equalnorm}
    \sum\limits_{g\in [m]}w_g||\beta_{G_g}||_2 = \sum\limits_{g\in [m]}\Big( \sum_{g\in F(\mathcal{g}) }w_g\Big)||\mathcal{G}_\mathcal{g}||_2 \leqslant  ||\beta_{\{\mathcal{G},\tilde{w}\}}||_{q_1,q_2} \leqslant \sum\limits_{\mathcal{g} \in [\mathcal{m}]}\mathcal{w}_\mathcal{g}||\beta_{\mathcal{G}_\mathcal{g}}||_2.
\end{equation}

This implies that $(\tilde{w}_{\mathcal{g}}||\beta_{\mathcal{G}_{\mathcal{g}}}||_{q_2}^{q_1})^{\frac{1}{q_1}} = \mathcal{w}_{\mathcal{g}}||\beta_{\mathcal{G}_{\mathcal{g}}}||_2$. By setting one element in $\mathcal{G}_\mathcal{g}$ to 1, and other elements to 0, it follows that $\tilde{w}_{\mathcal{g}} = \mathcal{w}_{\mathcal{g}}$. Since this holds for any group in $\mathcal{G}$, we have $\tilde{w} = \mathcal{w}$.

From Equation \eqref{eq:equalnorm}, it is evident that $(\mathcal{w}_{\mathcal{g}}||\beta_{\mathcal{G}_{\mathcal{g}}}||_{q_2}^{q_1})^{\frac{1}{q_1}} = \mathcal{w}_{\mathcal{g}}||\beta_{\mathcal{G}_{\mathcal{g}}}||_2$ for any $\beta$ with non-zero elements only in $\mathcal{G}_{\mathcal{g}}$. This suggests that $q_1 = 1$ and $q_2 = 2$. Therefore, the existing norm $||\beta_{\{\tilde{G},\tilde{w}\}}||_{q_1,q_2}$ does not satisfy the second condition in Equation \eqref{twocondi}.
\end{proof}

\subsubsection{Proof of Lemma \ref{lem:tightbounds}}
\begin{proof}
We begin by proving the first item in Lemma \ref{lem:tightbounds}. Recall that $\mathbb{G}$ represents the space of all possible partitions of $[p]$. Given that $\tilde{G} \in \mathbb{G}$, for an arbitrary $\mathcal{g} \in [\mathcal{m}]$, suppose $\mathcal{G}_\mathcal{g} \nsubseteq \tilde{G}_{\tilde{g}}$ for any $\tilde{g}$. Then, we can identify the smallest set $T$ such that 
$$\mathcal{G}_{\mathcal{g}} \subseteq \bigcup\limits_{\tilde{g} \in T}\tilde{G}_{\tilde{g}}.$$

Let $T = \{t_1,t_2,\cdots,t_{|T|}\}$. Choose $\beta_{j} \in (\mathcal{G}_{\mathcal{g}} \cap \tilde{G}_{t_1})$ and $\beta_{k} \in (\mathcal{G}_{\mathcal{g}} \cap \tilde{G}_{t_2})$. As $\beta_j$ and $\beta_k$ are both in $\mathcal{G}_{\mathcal{g}}$, if an original group includes $\beta_j$, it also contains $\beta_k$. Consider a vector $\beta$ where only $\beta_j$ and $\beta_k$ are non-zero. We have 
\begin{equation*}
    \begin{aligned}
 \sum\limits_{g\in [m]}w_g||\beta_{G_g}||_2 &= \Big( \sum_{\{g|\beta_j \in G_g\} }w_g\Big)\sqrt{\beta_j^2 + \beta_k^2}
 \leqslant ||\beta_{\{\tilde{G},\tilde{w}\}}||_{q_1,q_2} \\&\leqslant \sum\limits_{\mathcal{g} \in [\mathcal{m}]}\mathcal{w}_\mathcal{g}||\beta_{\mathcal{G}_\mathcal{g}}||_2 =  \Big( \sum_{\{g|\beta_j \in G_g\} }w_g\Big)\sqrt{\beta_j^2 + \beta_k^2},
 \end{aligned}
\end{equation*}
leading to
$$||\beta_{\{\tilde{G},\tilde{w}\}}||_{q_1,q_2} = \left((\tilde{w}_{t_1}|\beta_j|)^{q_1} + (\tilde{w}_{t_2}|\beta_k|)^{q_1}\right)^{\frac{1}{q_1}} = \tilde{w}_{t_1}^{\frac{1}{q_1}}|\beta_j| + \tilde{w}_{t_2}^{\frac{1}{q_1}}|\beta_k| = \Big( \sum_{\{g|\beta_j \in G_g\} }w_g\Big)\sqrt{\beta_j^2 + \beta_k^2},$$  
for any $0 \leqslant q_1,q_2 \leqslant \infty$. However, by setting
$$\begin{cases}
 \beta_j = \beta_k = 1, \beta_{\{[p]\backslash\{j,k\}\}} = 0 &  \text{if $w_{t_1}^{\frac{1}{q_1}} + w_{t_2}^{\frac{1}{q_1}} \neq \sqrt{2}\left( \sum_{\{g|\beta_j \in G_g\} }w_g\right)$ } \\
 \beta_j = 2, \beta_k = 1, \beta_{\{[p]\backslash\{j,k\}\}} = 0 &  \text{if $w_{t_1}^{\frac{1}{q_1}} + w_{t_2}^{\frac{1}{q_1}} = \sqrt{2}\left( \sum_{\{g|\beta_j \in G_g\} }w_g\right)$ }
\end{cases},$$ 
we arrive at a contradiction. Thus, we demonstrate that if a norm $||\cdot_{\{\tilde{G},w\}}||_{q_1,q_2}$ exists, then each group in $\tilde{G}$ is a union of groups in $\mathcal{G}$.

Now, to prove the second item: Since the first part implies that each group in $\tilde{G}$ is a union of groups in $\mathcal{G}$, let us consider $\tilde{g} \in [|\tilde{G}|]$. Suppose there exists a index set $V \subseteq [\mathcal{m}]$ such that $\tilde{G}_{\tilde{g}} = \bigcup\limits_{\mathcal{g} \in V}\mathcal{G}_{\mathcal{g}}$ with $|V| > 1$. Denote $V = \{v_1,\cdots,v_{|V|}\}$, and consider two cases:

\textbf{Case 1:} $\nexists a \in [m]$ s.t. $(\mathcal{G}_{\mathcal{v_1}} \cup \mathcal{G}_{\mathcal{v_2}}) \subseteq G_a$. 

\textbf{Case 2:} $\exists a \in [m]$ s.t. $(\mathcal{G}_{\mathcal{v_1}} \cup \mathcal{G}_{\mathcal{v_2}}) \subseteq G_a$. 

Under Case 1, if only $\mathcal{G}_{\mathcal{v_1}}$ and $\mathcal{G}_{\mathcal{v_2}}$ have non-zero values in $\beta$, we obtain:
\begin{equation*}
    \begin{aligned}
     \sum\limits_{g\in [m]}w_g||\beta_{G_g}||_2 &= \Big( \sum_{g\in F(\mathcal{v}_1) }w_g \Big)\sqrt{\beta_{\mathcal{G}_{\mathcal{v_1}}}^2} + \Big( \sum_{g\in F(\mathcal{v}_2) }w_g \Big)\sqrt{\beta_{\mathcal{G}_{\mathcal{v_2}}}^2} \\
     &\leqslant ||\beta_{\{\tilde{G},\tilde{w}\}}||_{q_1,q_2} \leqslant \sum\limits_{\mathcal{g} \in [\mathcal{m}]}\mathcal{w}_\mathcal{g}||\beta_{\mathcal{G}_\mathcal{g}}||_2 \\
     &= \mathcal{w}_{\mathcal{v_1}}\sqrt{\beta_{\mathcal{G}_{\mathcal{v_1}}}^2} +\mathcal{w}_{\mathcal{v_2}}\sqrt{\beta_{\mathcal{G}_{\mathcal{v_2}}}^2}\\
     &=\Big( \sum_{g\in F(\mathcal{v}_1) }w_g \Big)\sqrt{\beta_{\mathcal{G}_{\mathcal{v_1}}}^2} + \Big( \sum_{g\in F(\mathcal{v}_2) }w_g \Big)\sqrt{\beta_{\mathcal{G}_{\mathcal{v_2}}}^2},
    \end{aligned}
\end{equation*}
which leads to 
$$\mathcal{w}_{\mathcal{v_1}}\sqrt{\beta_{\mathcal{G}_{\mathcal{v_1}}}^2} +\mathcal{w}_{\mathcal{v_2}}\sqrt{\beta_{\mathcal{G}_{\mathcal{v_2}}}^2} = \tilde{w}_{\tilde{g}} \Big( \sum\limits_{j \in \tilde{G}_{\tilde{g}}}|\beta_j|^{q_2} \Big)^{\frac{1}{q_2}}= \tilde{w}_{\tilde{g}} \Big( \sum\limits_{j \in \{\mathcal{G}_{\mathcal{v_1}}\cup\mathcal{G}_{\mathcal{v_2}}\}}|\beta_j|^{q_2} \Big)^{\frac{1}{q_2}}.$$

This equation does not hold by picking $j \in \mathcal{G}_{\mathcal{v_1}}, k \in \mathcal{G}_{\mathcal{v_2}}$, and setting  
$$\begin{cases}
 \beta_j = \beta_k = 1, \beta_{\{[p]\setminus\{j,k\}\}} = 0 &  \text{if $\mathcal{w}_{\mathcal{v_1}} + \mathcal{w}_{\mathcal{v_2}} \neq  \tilde{w}_{\tilde{g}}\cdot 2^{\frac{1}{q_2}}$ } \\
 \beta_j = 2, \beta_k = 1, \beta_{\{[p]\setminus\{j,k\}\}}=0 &  \text{if $\mathcal{w}_{\mathcal{v_1}} + \mathcal{w}_{\mathcal{v_2}} =  \tilde{w}_{\tilde{g}}\cdot 2^{\frac{1}{q_2}}$ } 
\end{cases}.$$ 
Therefore, $|V| > 1$ cannot happen.

Under Case 2, let $\beta_j \in \mathcal{G}_{\mathcal{v_1}}$ and $\beta_k \in \mathcal{G}_{\mathcal{v_2}}$. Define $\beta^j$ as the vector with $1$ at the $j$-th element and $0$ elsewhere, and $\beta^k$ as the vector with $1$ at the $k$-th element and $0$ elsewhere, with $j \neq k$. 

When $\beta = \beta^j$, we find:
$$\sum\limits_{g\in [m]}w_g||\beta_{G_g}||_2 = \Big( \sum_{g\in F(\mathcal{v}_1) }w_g\Big) \leqslant  \tilde{w}_{\tilde{g}} \leqslant \sum\limits_{\mathcal{g} \in [\mathcal{m}]}\mathcal{w}_\mathcal{g}||\beta_{\mathcal{G}_\mathcal{g}}||_2 = \mathcal{w}_{\mathcal{v_1}},$$
indicating that $\tilde{w}_{\tilde{g}} = \mathcal{w}_{\mathcal{v_1}}$ for all $q_1, q_2$.

Similarly, for $\beta = \beta^k$, we have:
$$\sum\limits_{g\in [m]}w_g||\beta_{G_g}||_2 = \Big( \sum_{g\in F(\mathcal{v}_2) }w_g\Big) \leqslant  \tilde{w}_{\tilde{g}} \leqslant \sum\limits_{\mathcal{g} \in [\mathcal{m}]}\mathcal{w}_\mathcal{g}||\beta_{\mathcal{G}_\mathcal{g}}||_2 = \mathcal{w}_{\mathcal{v_2}},$$
indicating that $\tilde{w}_{\tilde{g}} = \mathcal{w}_{\mathcal{v_2}}$ for all $q_1, q_2$.

If $\mathcal{w}_{\mathcal{v_1}} \neq \mathcal{w}_{\mathcal{v_2}}$, then such a weight assignment is not feasible. Assuming $\mathcal{w}_{\mathcal{v_1}} =\mathcal{w}_{\mathcal{v_2}} =  w_{\tilde{g}} = k$, then for any $\beta$ with non-zero values only in $\mathcal{G}_{\mathcal{v_1}}$, we have $w_{\tilde{g}}||\beta_{\mathcal{G}_{\mathcal{v_1}}}||_2 = (w_{\tilde{g}}||\beta_{\mathcal{G}_{\mathcal{v_1}}}||_{q_2}^{q_1})^{\frac{1}{q_1}}$, implying that if a norm satisfies \eqref{twocondi}, it must be an $\ell_1/\ell_2$ norm.

Since $\mathcal{G}_{\mathcal{v_1}}$ and $\mathcal{G}_{\mathcal{v_2}}$ are different groups, there is at least one original group that contains variables in $\mathcal{G}_{\mathcal{v_1}}$ but not in $\mathcal{G}_{\mathcal{v_2}}$, and vice versa. Taking $\beta$ with non-zero values in both $\mathcal{G}_{\mathcal{v_1}}$ and $\mathcal{G}_{\mathcal{v_2}}$, we find:

$$\sum\limits_{g\in [m]}k||\beta_{G_g}||_2 > k||\beta_{\mathcal{G}_{\mathcal{v_1}}}\cup \beta_{\mathcal{G}_{\mathcal{v_2}}}||_2 = ||\beta_{\{\tilde{G},\tilde{w}\}}||_{1,2},$$
which is a contradiction. Hence, in both cases, $|V| > 1$ is not possible, implying that there exists a $\mathcal{g} \in [\mathcal{m}]$ such that $\tilde{G}_{\tilde{g}} = \mathcal{G}_{\mathcal{g}}.$ 
\end{proof}

\subsection{Proof of Theorem~\ref{thm:two-bounds}}
\begin{proof}
We begin by examining the bound for the estimator $\hat{\beta}^G$. Considering a fixed design matrix $X$ and a group structure $G$ that comply with Assumption~\ref{ass:distr}, and selecting an appropriate $\lambda_n$, Theorem~\ref{thm:two-bound} asserts that both inequalities \eqref{eq:ell2bound} and \eqref{eq:parell2bound} hold with a probability of at least $1-e^{-2n\delta}$.

Under Assumptions~\ref{ass:distribution_sub_noise},\ref{ass:distribution_normal_rd}, and \ref{ass:distribution_groupstructure}, Theorem~\ref{the:rsccondtion} establishes that Assumption~\ref{ass:distr} is valid with a probability of at least $1-e^{-c_2n\delta^2}-\frac{e^{-\frac{n}{32}}}{1-e^{\frac{n}{64}}}$, where $c_2$ is a positive constant.

Considering these two theorems together, we conclude that under Assumptions~\ref{ass:distribution_sub_noise},\ref{ass:distribution_normal_rd}, and \ref{ass:distribution_groupstructure}, both \eqref{eq:ell2bound} and \eqref{eq:parell2bound} are satisfied with a probability of at least $1-e^{-c_2n\delta^2}- e^{-2n\delta}-\frac{e^{-\frac{n}{32}}}{1-e^{\frac{n}{64}}}$. This probability can be further bounded below by $1-e^{-c'n\delta}$ for some suitable constant $c'$.

The bound for $\hat{\beta}^{\Gcal}$ can be directly derived, noting that it represents a group lasso estimator with group $\Gcal$ and weights $\wcal$.

\end{proof}

\subsection{Proof of Corollary~\ref{cor1}}
\begin{proof}

Assuming \( \max\{d_{\max},m \}\asymp \max\{\dcal_{\max},\mcal\} \), we have $\left(\frac{\mathcal{d}_{\max}{\log{5}}}{n}+\frac{\log{\mathcal{m}}}{n}+\delta\right) \asymp \left(\frac{d_{\max}{\log{5}}}{n}+\frac{\log{m}}{n}+\delta\right)$.

With \( \mathcal{w}_\mathcal{g} =  \sum\limits_{g\in F(\mathcal{g}) }w_g \), by the Cauchy–Schwarz inequality, we have
\begin{align*}
\mathcal{w}_\mathcal{g}^2 &= \Big(\sum\limits_{g\in F(\mathcal{g}) }w_g \Big)^2
\leqslant \mathcal{h}_\mathcal{g} \Big(\sum\limits_{g\in F(\mathcal{g})}w_g^2 \Big).
\end{align*}
Therefore,
\begin{align*}
    & \sum\limits_{\mathcal{g}\in F^{-1}(S) }{\mathcal{w}_{\mathcal{g}}}^2
\leqslant \sum\limits_{\mathcal{g}\in F^{-1}(S) }{\mathcal{h}_\mathcal{g}(\sum\limits_{g\in F(\mathcal{g})}w_g^2)}
\leqslant h^{G_{\overline{S}}}_{\max}\Big(\sum\limits_{\mathcal{g}\in F^{-1}(S) }{\sum\limits_{g\in F(\mathcal{g})}w_g^2} \Big).
\end{align*}

Let's introduce \( k_g \) as the number of non-overlapping groups from \( G \) into which the \( g \)th group is partitioned in the new structure \( \Gcal \). We also define \( K \) as the maximum number of such partitions, i.e., \( K = \max_g k_g \) and \( K \leqslant \infty \).

Now we want to show that 
\begin{equation*}
\sum\limits_{\mathcal{g}\in F^{-1}(S)}{\sum\limits_{g\in F(\mathcal{g})}w_g^2} \leqslant \sum\limits_{g \in \overline{S}} k_g w_g^2.
\end{equation*}
Recall the definition of \( F^{-1}(S) \) as 
\begin{equation*}
F^{-1}(S) = \{\gcal \mid \gcal \in F^{-1}(g), g \in S\}.
\end{equation*}
For each \( \gcal \in F^{-1}(g) \) that also belongs to \( F^{-1}(S) \), we add \( w_g^2 \) to the summation. Therefore, the maximum contribution from each original group \( g \) to the sum \( \sum\limits_{\mathcal{g}\in F^{-1}(S)}{\sum\limits_{g\in F(\mathcal{g})}w_g^2} \) is \( k_g w_g^2 \).

Given that 
\begin{equation*}
\{g | g \in F(\mathcal{g}) \text{ and } \mathcal{g}\in F^{-1}(S)\} = \overline{S},
\end{equation*}
it follows that 
\begin{align*}
    h^{G_{\overline{S}}}_{\max}\Big(\sum\limits_{\mathcal{g}\in F^{-1}(S)}{\sum\limits_{g\in F(\mathcal{g})}w_g^2}\Big) &\leqslant h^{G_{\overline{S}}}_{\max}\sum\limits_{g \in \overline{S}} k_g w_g^2 \\
    &\leqslant h^{G_{\overline{S}}}_{\max}K\sum\limits_{g \in \overline{S}} w_g^2.
\end{align*}

On the other hand, we have
\begin{align*}
    {\min\limits_{\mathcal{g}\in[m]}\Big(\mathcal{w}_\mathcal{g}^2 \Big)} &	= \min\limits_{\mathcal{g}\in[\mathcal{m}]} \Big(\sum\limits_{g\in F(\mathcal{g}) }w_g \Big)^2\geqslant \min\limits_{\mathcal{g}\in[\mathcal{m}]} \bigg(\sum\limits_{g\in F(\mathcal{g}) }\min\limits_{g\in[m]}\{w_g\} \bigg)^2 \\
	& \geqslant \min\limits_{\mathcal{g}\in[\mathcal{m}]} \left(h_{\min}^{g}\min\limits_{g\in[m]}\{w_g\} \right)^2 = \left(h_{\min}^{g}\min\limits_{g\in[m]}\{w_g\} \right)^2 
	\geqslant \min\limits_{g\in[m]}\left(w_g^2 h_{\min}^{g}\right).
\end{align*}

Therefore, 

$$ \frac{\sum\limits_{\mathcal{g}\in F^{-1}(S) }{\mathcal{w}_{\mathcal{g}}}^2}{\min\limits_{\mathcal{g}\in[\mathcal{m}]}\left(w_\mathcal{g}^2 \right)} \leqslant \frac{K\Big(\sum\limits_{g\in\overline{S}}{w_g}^2\Big)\cdot h^{G_{\overline{S}}}_{\max}}{\min\limits_{g\in[m]}\left(w_g^2 h_{\min}^{g}\right)}.$$

Consequently, if $K$ is upper bounded by a constant, then
$$
\frac{ \sigma^2}{\kappa^2}\cdot\frac{\sum\limits_{\mathcal{g}\in F^{-1}(S) }{\mathcal{w}_{\mathcal{g}}}^2}{\min\limits_{\mathcal{g}\in[\mathcal{m}]}\left(w_\mathcal{g}^2 \right)}
\cdot\left(\frac{\mathcal{d}_{\max}{\log{5}}}{n}+\frac{\log{\mathcal{m}}}{n}+\delta\right) \lesssim\frac{ \sigma^2}{\kappa^2}\cdot\frac{\Big(\sum\limits_{g\in\overline{S}}{w_g}^2\Big)\cdot h^{G_{\overline{S}}}_{\max}}{\min\limits_{g\in[m]}\left(w_g^2 h_{\min}^{g}\right)}
\cdot\left(\frac{d_{\max}{\log{5}}}{n}+\frac{\log{m}}{n}+\delta\right).
$$

\end{proof}

\subsection{Proof of Proposition \ref{pro1}}
\begin{proof}
Let $H_{G_g}$ be the sub-matrix of $H$ consisting of the columns indexed by $G_g$. Let  $u_{G_g}$, $v_{G_g}$ be the sub-vectors of $u,v$ indexed by $G_g$ respectively. Given two vectors $u, v \in \mathbb{R}^p$, we have
\begin{equation*}
    \begin{aligned}
    	\phi^{\ast}(v) & =\sup\limits_{\phi(u)\leqslant 1}\left\{u^Tv\right\}
     =\sup\limits_{\phi(u)\leqslant 1}\left\{u_1v_1+u_2v_2+\cdots+u_pv_p\right\}\\
	&  =\sup\limits_{\phi(u)\leqslant 1}\left\{ \frac{v_1}{h_1}\cdot{h_1}\cdot{u_1} + \cdots+ \frac{v_p}{h_p}\cdot{h_p}\cdot{u_p}    \right\}\\
	& =\sup\limits_{\phi(u)\leqslant 1}\left\{  \sum\limits_{g=1}^{m} \left(H_{G_g}v_{G_g}\right)^T u_{G_g}   \right\}
=\sup\limits_{\phi(u)\leqslant 1}\left\{  \sum\limits_{g=1}^{m} \frac{\left((Hv)_{G_g}\right)}{w_g}\cdot w_g \cdot u_{G_g}   \right\}\\
&\leqslant \sup\limits_{\phi(u)\leqslant 1}\left\{  \sum\limits_{g=1}^{m} \frac{\left\|(Hv)_{G_g}\right\|_2}{w_g} \cdot \left\|w_g{u_{G_g}}\right\|_2 \right\}
\leqslant \left( \max\limits_{g\in[m]}\frac{1}{w_g}\cdot\left\|(Hv)_{G_g}\right\|_2\right)\cdot\phi(u)\\
& \leqslant \max\limits_{g\in[m]}\frac{1}{w_g}\cdot\left\|(Hv)_{G_g}\right\|_2,
\end{aligned}
\end{equation*}
where the first inequality is achieved by using Cauchy's inequality.

Let $g_0 =  \argmax\limits_{g\in[m]}{\frac{1}{w_g}\left\|\left(Hv\right)_{G_g}\right\|_2}$ and $h^{g_0}_{\max} = 1$. Define $u\in \mathbb{R}^p$ as
\begin{equation*}
    u_j=\left\{\begin{array}{l}
0\hspace{0.2cm}\textit{for}\hspace{0.2cm} j\notin G_{g_0}\\
\frac{1}{w_{g_0}}\cdot \frac{v_j}{{h_j}^2}\cdot\frac{1}{\left\|\left(Hv\right)_{G_{g_0}}\right\|_2}\hspace{0.2cm}\textit{for}\hspace{0.2cm} j\in G_{g_0},
\end{array}\right.
\end{equation*}
then we have
\begin{equation*}
    \begin{aligned}
    \phi\left(u\right)  &=\sum\limits_{g=1}^m{w_g}\left\|u_{G_g}\right\|_2
={w_{g_0}}\cdot\frac{1}{w_{g_0}}\cdot\frac{1}{\left\|\left(Hv\right)_{G_{g_0}}\right\|_2}\cdot\sqrt{\sum\limits_{j\in{G_{g_0}}}\frac{{v_j}^2}{{h_j}^4}}\\
	&  =\frac{1}{\left\|\left(Hv\right)_{G_{g_0}}\right\|_2}\sqrt{\sum\limits_{j\in{G_{g_0}}}\frac{{v_j}^2}{{h_j}^2}}= 1,
\end{aligned}
\end{equation*}
where the last equality holds due to the fact that $h_j = 1$ for any $j \in G_{g_0}$, and we also have
\begin{equation*}
    \begin{aligned}
   u^Tv &=\frac{1}{{w_{g_0}}}\frac{1}{\left\|\left(Hv\right)_{G_{g_0}}\right\|_2}\cdot\sum\limits_{j\in{G_{g_0}}}\frac{{v_j}^2}{{h_j}^2}
=\frac{1}{{w_{g_0}}}\frac{1}{\left\|\left(Hv\right)_{G_{g_0}}\right\|_2}\cdot{\left\|\left(Hv\right)_{G_{g_0}}\right\|_2^2}\\
&=\frac{1}{{w_{g_0}}} \left\|\left(Hv\right)_{G_{g_0}}\right\|_2
=\max\limits_{g\in[m]}\frac{1}{{w_{g_0}}}{\left\|\left(Hv\right)_{G_{g}}\right\|_2}=\phi^{\ast}\left(v\right).
\end{aligned}
\end{equation*}

Therefore, this is a sharp bound.

\end{proof}

\subsection{Proof of Theorem \ref{thm:two-bound}}

\begin{proof}
In this section, we mostly follow the proof in Chapter 14 of \cite{wainwright2019high}. By default, we take $S = S(\beta^*)$ and $\overline{S} = \overline{S}(\beta^*)$ in all settings. From the optimality of $\hat{\beta}^G$, we have
\begin{equation*}
    \begin{aligned}
0 &\geqslant \frac{1}{n}\Big\|Y-X\hat{\beta}\Big\|_2^2 - \frac{1}{n}\Big\|Y-X\beta^{\ast}\Big\|_2^2 +\lambda_n\left(\phi(\hat{\beta})-\phi(\beta^{\ast})\right)\\
&=\frac{1}{n}\left(Y^TY-2Y^TX\hat{\beta}+\hat{\beta}^TX^TX\hat{\beta}-Y^TY+2Y^TX{\beta^{\ast}}-\beta^{\ast{T}}X^TX\beta^\ast\right)+\lambda_n\left(\phi(\hat{\beta})-\phi\left(\beta^{\ast}\right)\right)\\
&=\frac{1}{n}\left((2{X^T}X{\beta^{\ast}}-2{X^T}Y)^T(\hat{\beta}-\beta^{\ast})+(\hat{\beta}-\beta^{\ast})^T{X^T}X(\hat{\beta}-\beta^{\ast})\right)+\lambda_n\left(\phi(\hat{\beta})-\phi(\beta^{\ast})\right)\\
&=\bigg\langle\bigtriangledown\frac{\Big\|Y-X{\beta^\ast}\Big\|_2^2}{n},\left(\hat{\beta}-\beta^{\ast}\right)\bigg\rangle +\frac{\Big\|X\left(\hat{\beta}-\beta^{\ast}\right)\Big\|_2}{n}+\lambda_n\left(\phi(\hat{\beta})-\phi(\beta^{\ast})\right)\\
&\geqslant\bigg\langle \bigtriangledown\frac{\Big\|Y-X\beta^\ast\Big\|_2^2}{n}, \left(\hat{\beta}-\beta^{\ast}\right)\bigg\rangle + \kappa{\Big\|\left( \hat{\beta}-\beta^\ast\right) \Big\|_2^2} + \lambda_n\left(\phi(\hat{\beta})-\phi(\beta^{\ast})\right)\\
&\geqslant-\bigg| \bigg\langle \bigtriangledown\frac{\Big\|Y-X\beta^\ast\Big\|_2^2}{n}, \left(\hat{\beta}-\beta^{\ast}\right)\bigg\rangle \bigg| + \kappa{\Big\|\left( \hat{\beta}-\beta^\ast\right) \Big\|_2^2} + \lambda_n\left(\phi(\hat{\beta})-\phi(\beta^{\ast})\right),
\end{aligned}
\end{equation*}
where the penultimate step is valid due to the assumption of restrictive strong convexity.

By applying Holder's inequality with the regularizer $\phi$ and its dual norm $\phi^{\ast}$, we have
\begin{equation}
\label{b}
    \bigg|\bigg\langle \bigtriangledown\frac{\left\|Y-X{\beta^{\ast}}\right\|_2^2}{n},\left(\hat{\beta}-\beta^{\ast}\right)\bigg\rangle  \bigg|\leqslant \phi^{\ast}\Big(\bigtriangledown\frac{\left\|Y-X{\beta^\ast}\right\|_2^2}{n}\Big)
\phi\left(\hat{\beta}-\beta^{\ast}\right).
\end{equation}

Next, we have 
\begin{equation*}
    \begin{aligned}
    &\phi(\hat{\beta})=\phi\left(\beta^{\ast}+(\hat{\beta}-\beta^{\ast})\right)
=\phi\left(\beta_{M(S)}^{\ast}+\beta_{M^{\perp}(S)}^{\ast}+(\hat{\beta}-\beta^{\ast})_{M(\overline{S})}+(\hat{\beta}-\beta^{\ast})_{M^{\perp}(\overline{S})}\right)\\
&\hspace{0.8cm}\geqslant\phi\left(\beta_{M(S)}^{\ast}+(\hat{\beta}-\beta^{\ast})_{M^{\perp}(\overline{S})}\right)-
\phi(\beta_{M^{\perp}(S)}^{\ast})
-\phi\left((\hat{\beta}-\beta^{\ast})_{M(\overline{S})}\right)\\
& \hspace{0.8cm} = \phi(\beta_{M(S)}^{\ast})+\phi\left((\hat{\beta}-\beta^{\ast})_{M^{\perp}(\overline{S})}\right)-
\phi(\beta_{M^{\perp}(S)}^{\ast})
-\phi\left((\hat{\beta}-\beta^{\ast})_{M(\overline{S})}\right).
\end{aligned}
\end{equation*}

The inequality holds by applying the triangle inequality on $\phi(\hat{\beta})$, and the last step holds by applying Lemma~\ref{l2}. Consequently, we have
\begin{equation}
\label{c}
  \begin{aligned}
    \phi(\hat{\beta})-\phi\left(\beta^{\ast}\right)&\geqslant \phi\left((\hat{\beta}-\beta^{\ast})_{M^{\perp}(\overline{S})}\right)-
\phi\left((\hat{\beta}-\beta^{\ast})_{M(\overline{S})}\right)
-2\phi(\beta_{M^{\perp}(S)}^{\ast}) \\
&= \phi\left((\hat{\beta}-\beta^{\ast})_{M^{\perp}(\overline{S})}\right)-
\phi\left((\hat{\beta}-\beta^{\ast})_{M(\overline{S})}\right),
\end{aligned}
\end{equation}
where $\phi\left(\beta_{M^{\perp}(S)}^{\ast}\right)=0$ as $\beta_{M^{\perp}\left(S\right)}^{\ast}$ is a zero vector.

Based on Equation(\ref{b}) and Equation(\ref{c}), we have
\begin{equation*}
    \begin{aligned}
&\quad \frac{1}{n}\left\|Y-X\hat{\beta}\right\|_2^2-\frac{1}{n}\left\|Y-X{\beta^{\ast}}\right\|_2^2+\lambda_n\left(\phi(\hat{\beta})-\phi(\beta^{\ast})\right)
\\&\geqslant-\bigg| \bigg\langle \bigtriangledown\frac{\left\|Y-X\beta^\ast\right\|_2^2}{n}, \left(\hat{\beta}-\beta^{\ast}\right)\bigg\rangle \bigg|+
\kappa{\left\|\left( \hat{\beta}-\beta^\ast\right) \right\|_2^2}
+\lambda_n\left(\phi(\hat{\beta})-\phi(\beta^{\ast})\right)
\\&\geqslant \kappa{\left\|\left( \hat{\beta}-\beta^\ast\right) \right\|_2^2}
+\lambda_n\left(\phi\left((\hat{\beta}-\beta^{\ast})_{M^{\perp}(\overline{S})}\right)-
\phi\left((\hat{\beta}-\beta^{\ast})_{M(\overline{S})}\right)\right)-\bigg|\bigg\langle \bigtriangledown\frac{\left\|Y-X\beta^\ast\right\|_2^2}{n}, \left(\hat{\beta}-\beta^{\ast}\right)\bigg\rangle \bigg|
\\&\geqslant \kappa{\left\|\left( \hat{\beta}-\beta^\ast\right) \right\|_2^2}
+\lambda_n\left(\phi\left((\hat{\beta}-\beta^{\ast})_{M^{\perp}(\overline{S})}\right)-
\phi\left((\hat{\beta}-\beta^{\ast})_{M(\overline{S})}\right)\right)- \phi^{\ast} \Big(\bigtriangledown\frac{\left\|Y-X{\beta^{\ast}} \right\|_2^2}{n}\Big) \phi\left(\hat{\beta}-\beta^{\ast}\right)
\\&\geqslant \kappa\left\|\left(\hat{\beta}-\beta^{\ast}\right) \right\|_2^2
+\lambda_n\left(\phi\left((\hat{\beta}-\beta^{\ast})_{M^{\perp}(\overline{S})}\right)-
\phi\left((\hat{\beta}-\beta^{\ast})_{M(\overline{S})}\right)\right)
-\frac{\lambda_n}{2}\phi\left(\hat{\beta}-\beta^{\ast}\right),
\end{aligned}
\end{equation*}

where the last step is valid because Lemma~\ref{l1} implies that we can guarantee $\lambda_n\geqslant 2\phi^{\ast} \left(\bigtriangledown\frac{\left\|Y-X{\beta^{\ast}} \right\|_2^2}{n}\right)$ with high probability by taking appropriate $\lambda_n$. Moreover, Lemma~\ref{l3} implies that $$\hat{\beta}\in\left\{\beta\in{\mathbb{R}^p} \mid \phi \left(\left(\beta-\beta^{\ast}\right)_{M^{\perp}(\overline{S})}\right) \leqslant 3\phi\left(\left(\beta-\beta^{\ast}\right)_{M(\overline{S})}\right)\right\}.$$

By the triangle inequality, we have
$$\phi(\hat{\beta}-\beta^{\ast})=\phi\left((\hat{\beta}-\beta^{\ast})_{M(\overline{S})}+(\hat{\beta}-\beta^{\ast})_{M^{\perp}(\overline{S})}\right)\leqslant \phi\left((\hat{\beta}-\beta^{\ast})_{M(\overline{S})}\right)+\phi\left((\hat{\beta}-\beta^{\ast})_{M^{\perp}(\overline{S})}\right),$$
and hence we have
\begin{equation*}
    \begin{aligned}
&\quad \frac{1}{n}\left\|Y-X\hat{\beta}\right\|_2^2-\frac{1}{n}\left\|Y-X{\beta^{\ast}}\right\|_2^2+\lambda_n\left(\phi(\hat{\beta})-\phi(\beta^{\ast})\right)
\\&\geqslant \kappa\left\|\left(\hat{\beta}-\beta^{\ast}\right) \right\|_2^2
+\lambda_n\left(\phi\left((\hat{\beta}-\beta^{\ast})_{M^{\perp}(\overline{S})}\right)-
\phi\left((\hat{\beta}-\beta^{\ast})_{M(\overline{S})}\right)\right)
-\frac{\lambda_n}{2}\phi\left(\hat{\beta}-\beta^{\ast}\right)
\\&\geqslant \kappa\left\|\left(\hat{\beta}-\beta^{\ast}\right) \right\|_2^2
+\lambda_n\left(\phi\left((\hat{\beta}-\beta^{\ast})_{M^{\perp}(\overline{S})}\right)-
\phi\left((\hat{\beta}-\beta^{\ast})_{M(\overline{S})}\right)\right)\\
&\quad~~~~~~~~~~ \quad -\frac{\lambda_n}{2}\left(\phi\left((\hat{\beta}-\beta^{\ast})_{M(S)}\right)+\phi\left((\hat{\beta}-\beta^{\ast})_{M^{\perp}(\overline{S})}\right)\right)\\
&\geqslant \kappa\left\|\hat{\beta}-\beta^{\ast}\right\|_2^2
+\frac{\lambda_n}{2}\left(\phi(\hat{\beta}-\beta^{\ast})_{M^{\perp}(\overline{S})}-3\phi(\hat{\beta}-\beta^{\ast})_{M(\overline{S})}\right)\\
&\geqslant \kappa\left\|\hat{\beta}-\beta^{\ast}\right\|_2^2
-\frac{3\lambda_n}{2}\phi\left((\hat{\beta}-\beta^{\ast})_{M(\overline{S})}\right).
\end{aligned}
\end{equation*}

By definition, we have $\phi\left((\hat{\beta}-\beta^{\ast})_{M(\overline{S})}\right)=\sum\limits_{g\in\overline{S}}{w_g}\left\|\left(\hat{\beta}-\beta^{\ast}\right)_{G_g}\right\|_2,$ and by Cauchy-Schwarz inequality, we have

\begin{equation*}
    \begin{aligned}
   \sum\limits_{g\in\overline{S}}{w_g}\left\|\left(\hat{\beta}-\beta^{\ast}\right)_{G_g}\right\|_2 
   &\leqslant \sqrt{\sum\limits_{g\in\overline{S}}{w_g}^2}\cdot\sqrt{h^{G_{\overline{S}}}_{\max}\cdot\max\limits_{g\in\overline{S}}\left\|\left(\hat{\beta}-\beta^{\ast}\right)_{G_g}\right\|^2_2}\\
&\leqslant \sqrt{\sum\limits_{g\in\overline{S}}{w_g}^2}\cdot\sqrt{h^{G_{\overline{S}}}_{\max}\cdot\left\|\left(\hat{\beta}-\beta^{\ast}\right)\right\|^2_2} \\
&= \sqrt{\sum\limits_{g\in\overline{S}}{w_g}^2}\cdot\sqrt{h^{G_{\overline{S}}}_{\max}}\left\|\left(\hat{\beta}-\beta^{\ast}\right)\right\|_2.
\end{aligned}
\end{equation*}

On the other hand, since
$\kappa\left\|\hat{\beta}-\beta^{\ast}\right\|_2^2-\frac{3\lambda_n}{2}\sqrt{\sum\limits_{g\in\overline{S}}{w_g}^2}\cdot\sqrt{h^{G_{\overline{S}}}_{\max}}\left\|\left(\hat{\beta}-\beta^{\ast}\right)\right\|_2\leqslant 0$, we have
\begin{equation*}
\begin{aligned}
    \left\|\hat{\beta}-\beta^{\ast}\right\|_2^2&\leqslant
\frac{9\lambda_n^2}{4\kappa^2}\sum\limits_{g\in\overline{S}}{{w_g}^2}\cdot h^{G_{\overline{S}}}_{\max}\\
&\leqslant\frac{9}{4\kappa^2}\cdot\frac{64{ c^2 \sigma^2} \sum \limits_{g\in\overline{S}}{w_g}^2\cdot h_{\max}(\overline{S})}{\min\limits_{g\in[m]}\left(w_g^2 h_{\min}^{g}\right)}\cdot\left(\frac{d_{\max}{\log{5}}}{n}+\frac{\log{m}}{n}+\delta\right)\\
&\leqslant \frac{144c^2 \sigma^2}{\kappa^2}\cdot\frac{\sum\limits_{g\in\overline{S}}{w_g}^2\cdot h^{G_{\overline{S}}}_{\max}}{\min\limits_{g\in[m]}\left(w_g^2 h_{\min}^{g}\right)}
\cdot\left(\frac{d_{\max}{\log{5}}}{n}+\frac{\log{m}}{n}+\delta\right)
\end{aligned}
\end{equation*}
\end{proof}

\subsection{Lemmas for the proof of Theorem \ref{thm:two-bound}}

In these lemmas, we abbreviate $\hat{\beta}^{G}$ by $\hat{\beta}$. 

\begin{lemma}
\label{l1}
Under the Assumption~\ref{ass:distr} and \eqref{glreg}, taking
$$\begin{array}{c}
\lambda_n =  \frac{8c \sigma}{\sqrt{\min\limits_{g\in[m]}\left(w_g^2 h_{\min}^{g}\right)}}
\sqrt{\frac{d_{\max}\log{5}}{n}+ 
	\frac{\log{m}}{n}+\delta} \hspace{0.5cm}\textit{for some $\delta \in [0,1]$,}
\end{array}$$

then $P \left(\lambda_n \geqslant 2 \phi^{*}(\frac{X^{\top} \varepsilon}{n})\right) \geqslant 1-e^{-2 n \delta}. $
\end{lemma}

\begin{proof}[Proof of Lemma \ref{l1}]

Let $V_{i\cdot g}= - \varepsilon_i \left(\frac{X_{ig_1}}{h_{g_1}w_g},\frac{X_{ig_2}}{h_{g_2}w_g},\dots,\frac{X_{ig_{d_g}}}{h_{g_{d_g}}w_g}\right) \in \mathbb{R}^{d_g}$.
According to the variational form of $\ell_2$ norm, we have $\frac{1}{n}\left\|\sum_{i=1}^{n} V_{i \cdot g}\right\|_{2}=\underset{u\in S^{d_{g-1}}}{\operatorname{sup}} \left\langle u, \frac{1}{n} \sum_{i=1}^{n} V_{i \cdot g} \right\rangle$, where  $S^{d_{g}-1}$ is the Euclidean sphere $\operatorname{in} \mathbb{R}^{d_g}$. Also, for any vector $u \in S^{d_{g-1}}$ and $t \in \mathbb{R}$, we have
\begin{equation*}
\begin{aligned}
    \frac{1}{n} \log \mathbb{E}\bigg(e^{t\left\langle  u, \sum\limits_{i=1}^{n} V_{i\cdot g} \right\rangle } \bigg)&
= \frac{1}{n} \log \mathbb{E}\bigg(e^{t  \sum\limits_{j=1}^{d_{g}} u_{j } \sum\limits_{i=1}^nV_{i\cdot g_j}  }  \bigg)
= \frac{1}{n} \log \mathbb{E}\Bigg(e^{t  \sum\limits_{i=1}^{n} \Big( \sum\limits_{j=1}^{d_{g}} u_{j } V_{i\cdot g_j}\Big)   }  \Bigg)\\
&=\frac{1}{n}\log \mathbb{E} \bigg( e^{-t  \sum\limits_{i=1}^n\left( \sum\limits_{g=1}^{d_g} \frac{u_j X_{ig_j} \varepsilon_{i}}{h_{g_j}  w_g} \right) } \bigg)
=\frac{1}{n}\log \mathbb{E} \bigg( e^{-t  \sum\limits_{i=1}^n \varepsilon_{i} \left( \sum\limits_{j=1}^{d_g} \frac{u_j X_{ig_j} }{h_{g_j}  w_g} \right) } \bigg).
\end{aligned}
\end{equation*}
Since $\left\{\epsilon_{i}\right\}_{i=1}^{n} $ are i.i.d zero mean sub-Gaussian random variables with parameter $\sigma$, let $u = (u_1, \cdots,u_{d_g})^T \in \mathbb{R}^{d_g \times 1}$, $X_{i,g} = (X_{ig_1},\cdots X_{ig_{d_g}})^T \in \mathbb{R}^{d_g \times 1}$, then we have
$$
\begin{aligned}\frac{1}{n}\log \mathbb{E} \bigg( e^{-t  \sum\limits_{i=1}^n \varepsilon_{i} \left( \sum\limits_{j=1}^{d_g} \frac{u_j x_{ig_j} }{h_{g_j} w_g} \right) } \bigg)& =\frac{1}{n}\log \mathbb{E}\bigg(e^{-t\varepsilon_{1}\left( \sum\limits_{j=1}^{d_g} \frac{u_j X_{1g_j} }{h_{g_j}  w_g} \right)}\bigg)+\cdots+\frac{1}{n}\log \mathbb{E}\bigg(e^{-t \varepsilon_{n} \left( \sum\limits_{j=1}^{d_g} \frac{u_j X_{ng_j} }{h_{g_j} w_g} \right)}\bigg) \\ & \leqslant \frac{t^2 \sigma^2}{2n} \bigg(\sum\limits_{i=1}^{n} \Big(\sum\limits_{j=1}^{d_g} \frac{u_j X_{ig_j}}{ w_g h_{g_j}}\Big)^2\bigg)
\leqslant  \frac{t^2 \sigma^2}{2n} \frac{1}{w_g^2 \left(h_{\min}^{g}\right)^2} \bigg( \sum\limits_{i=1}^{n} \Big(\sum\limits_{j=1}^{d_g} u_j X_{ig_j}\Big)^2 \bigg)\\
&= \frac{t^2 \sigma^2}{2n} \frac{1}{w_g^2 \left(h_{\min}^{g}\right)^2} \bigg( \sum\limits_{i=1}^{n} \langle u,\ X_{i,g} \rangle^2\bigg) = \frac{t^2 \sigma^2}{2n} \frac{1}{w_g^2 \left(h_{\min}^{g}\right)^2} \bigg( \sum\limits_{i=1}^{n}( u^TX_{i,g}X^T_{i,g}u)\bigg)\\
&= \frac{t^2 \sigma^2}{2} \frac{1}{w_g^2 \left(h_{\min}^{g}\right)^2} \bigg( u^T\Big(\frac{1}{n}\sum\limits_{i=1}^{n}X_{i,g}X^T_{i,g}\Big)u\bigg)\\
&= \frac{t^2 \sigma^2}{2} \frac{1}{w_g^2 \left(h_{\min}^{g}\right)^2}\bigg( u^T\frac{X^T_{G_g}X_{G_g}}{n}u\bigg)\\
&\leqslant \frac{t^2 \sigma^2}{2} \frac{1}{w_g^2 \left(h_{\min}^{g}\right)^2} \bigg( \gamma_{\max}(\frac{X^T_{G_g}X_{G_g}}{n})\bigg)\\
\end{aligned}
$$

By Assumption~\ref{ass:distr}, we have
$ \gamma_{\max}(\frac{X^T_{G_g}X_{G_g}}{n}) \leqslant c^2$. Combining this with the previous proof, we have
$\frac{1}{n} \log \mathbb{E}\Big(e^{t\left\langle u, \sum\limits_{i=1}^{n} V_{i\cdot g}\right\rangle }\Big) \leqslant \frac{c^2 t^2 \sigma^2}{2 w_g^2 \left(h_{\min}^{g}\right)}$. Therefore, the random variable $\left\langle u, \sum\limits_{i=1}^{n} V_{i\cdot g}\right\rangle$ is the sub-Gaussian with the parameter at most $\sqrt{\frac{ c^2 \sigma^2}{w_g^2 \left(h_{\min}^{g}\right)}}$, and by properties of sub-Gaussian variables, we have
\begin{equation*}
    \log \mathbb{P}\Big(\Big\langle u, \sum\limits_{i=1}^{n} V_{i\cdot g}\Big\rangle \geqslant \frac{\lambda_n}{4}\Big) \leqslant -\frac{\lambda_n^2 w_g^2 h_{\min}^{g} }{32 C^2 \sigma^2}.
    \end{equation*}

 We can find a $\frac{1}{2}$ covering of $S^{d_g-1}$ in Euclidean norm:$\{u^1,u^2,\dots,u^N\}$ with $N \leq 5^{d_g}$, recall that
$\frac{1}{n} \left\|\sum_{i=1}^{n}V_{i\cdot g} \right\|_2
=\frac{1}{n}
\underset{u\in S^{d_{g-1}}}{\operatorname{sup}}
\left\langle u , \sum\limits_{i=1}^{n} V_{i\cdot g}\right\rangle,
$ so that for any   $u\in S^{d_{g-1}}$, we can find a $u^{q(u)} \in \left\{u^1, \ldots, u^N\right\}$, such that $\left\|u^{q(u)}-u\right\|_2\leqslant\frac{1}{2}$, and 
\begin{equation*}
    \begin{aligned}
    \frac{1}{n}
\underset{u\in S^{d_{g-1}}}{\operatorname{sup}}
\Big\langle u ,  \sum\limits_{i=1}^{n}V_{i\cdot g}\Big\rangle &=\frac{1}{n} \underset{u\in S^{d_{g-1}}}{\operatorname{sup}}
\Big(\Big\langle u-u^{q(u)},\sum_{i=1}^{n} V_{i\cdot g}\Big\rangle +\Big\langle u^{q(u)}, \sum_{i=1}^{n}V_{i\cdot g} \Big\rangle \Big)\\
& \leqslant \frac{1}{n} \underset{u\in S^{d_{g-1}}}{\operatorname{sup}} \Big\langle u-u^{q(u)},\sum_{i=1}^{n} V_{i\cdot g}  \Big\rangle + \frac{1}{n} \max\limits_{q\in[N]} \Big\langle u^q,  V_{i\cdot g} \Big\rangle 
\end{aligned}
\end{equation*}

By applying the Cauchy-Schwarz inequality, we have
\begin{equation*}
\frac{1}{n} \underset{u\in S^{d_{g-1}}}{\operatorname{sup}}
\Big\langle u-u^{q(u)}, \sum\limits_{i=1}^{n} V_{i\cdot g} \Big\rangle 
\leqslant \frac{\big\| u-u^{q(u)} \big\|_2}{n}  \Big\| \sum\limits_{i=1}^{n} V_{i\cdot g} \Big\|_2
\leqslant \frac{1}{2n} \Big\| \sum\limits_{i=1}^{n} V_{i\cdot g} \Big\|_2.
\end{equation*}
Hence, we obtain $\frac{1}{n} \Big\| \sum\limits_{i=1}^{n} V_{i\cdot g} \Big\|_2 \leqslant  \frac{1}{2n} \Big\| \sum\limits_{i=1}^{n} V_{i\cdot g} \Big\|_2 + \frac{1}{n} \max\limits_{q\in[N]} 
\Big\langle u^q, \sum\limits_{i=1}^{n} V_{i\cdot g} \Big\rangle,$
which indicates that 
\begin{equation*}
\frac{1}{n} \Big\|\sum\limits_{i=1}^{n} V_{i\cdot g}\Big\|_2 \leqslant 2 \max\limits_{q\in[N]} \Big\langle u^q, \frac{1}{n} \sum\limits_{i=1}^{n} V_{i\cdot g}\Big\rangle.
\end{equation*}
Consequently, we can express the probability as
\begin{align*}
\mathbb{P}\Big(\frac{1}{n} \Big\|\sum\limits_{i=1}^{n} V_{i\cdot g}\Big\|_2 \geqslant \frac{\lambda_n}{2}\Big) 
&\leqslant \mathbb{P} \Big( \max\limits_{q\in[N]}\Big\langle u^q, \frac{1}{n} \sum\limits_{i=1}^{n} V_{i\cdot g} \Big\rangle \geqslant \frac{\lambda_n}{4}\Big)\\
&\leqslant \sum_{q=1}^{N} \mathbb{P} \Big( \Big\langle u^q, \frac{1}{n} \sum\limits_{i=1}^{n} V_{i\cdot g} \Big\rangle \geqslant \frac{\lambda_n}{4}\Big)\\
&\leqslant N \exp \Big(- \frac{n \lambda_n^2 w_g^2 h_{\min}^{g}}{32 C^2 \sigma^2} \Big) \leqslant \exp \Big(- \frac{n \lambda_n^2 w_g^2 h_{\min}^{g}}{32 C^2 \sigma^2} + d_g \log{5}\Big),
\end{align*}
and by setting $\lambda_n =  \frac{8C \sigma}{\sqrt{\min\limits_{g\in[m]}(w_g^2 h_{\min}^{g})}}
\sqrt{\frac{d_{\max}\log{5}}{n}+ \frac{\log{m}}{n}+\delta}$, we get
\begin{equation*}
\begin{aligned}
\mathbb{P}\Big(\max\limits_{g \in [m]} \frac{1}{n} \Big\| \sum\limits_{i=1}^{n} V_{i\cdot g}\Big\|_2 \geqslant \frac{\lambda_n}{2}\Big) 
&\leqslant \sum_{g=1}^m \mathbb{P} \Big(\frac{1}{n} \Big\| \sum\limits_{i=1}^{n} V_{i\cdot g}\Big\|_2 \geqslant \frac{\lambda_n}{2}\Big) \\
&\leqslant \exp \Big(- \frac{n \lambda_n^2}{32 C^2 \sigma^2} \min\limits_{g \in [m]}(w_g^2 h_{\min}^{g}) + d_{\max} \log{5} + \log{m}\Big) \\
&\leqslant \exp\{-2n\delta\}.
\end{aligned}
\end{equation*}

From Proposition~\ref{pro1}, we have
\begin{equation*}
    \begin{aligned}
    \phi^{*}\Big(\frac{X^{\top} \varepsilon}{n}\Big)&\leqslant\max \limits_{g \in [m]} \frac{1}{w_g}\Big\|\Big(\frac{HX^{\top} \varepsilon}{n}\Big)_{G_{g}}\Big\|_{2} 
=\max \limits_{g \in [m]} \frac{1}{w_g}\Big\|\frac{1}{n}\sum\limits_{i=1}^n-\varepsilon_i\Big(\frac{X_{ig_1}}{h_{g_1}},\cdots,\frac{X_{ig_{d_g}}}{h_{g_{d_g}}}\Big)\Big\|_{2}=\max \limits_{g \in [m]}\Big\|\frac{1}{n} \sum\limits_{i=1}^n V_{i\cdot g} \Big\|_2.
\end{aligned}
\end{equation*}
Therefore,  $P \left(\lambda_n \geqslant 2 \phi^{*}(\frac{X^{\top} \varepsilon}{n})\right) \geqslant 1-e^{-2 n \delta}. $

\end{proof}

\begin{lemma}
\label{l2}
The group lasso regularizer \eqref{eq:glnorm} is decomposable with respect to the pair $\left\{M\left(S\right), M^{\perp}(\overline{S})\right\}$. That is, $
\phi(a+b)=\phi(a)+\phi(b), \hspace{0.2cm} \text{for all} \hspace{0.2cm} a\in M\left(S\right) \hspace{0.1cm}\textit{and} \hspace{0.2cm}  \text{for all} \hspace{0.2cm} b\in M^{\perp}(\overline{S})$.
\end{lemma}

\begin{proof}[Proof of Lemma \ref{l2}]
\begin{align*}
    \phi\left(a+b\right)	&=\sum\limits_{g=1}^m{w_g}\left\|\left(a+b\right)_{G_g}\right\|_2
	=\sum\limits_{g\in M(\overline{S})}{w_g}\left\|\left(a+b\right)_{G_g}\right\|_2+\sum\limits_{g\notin M(\overline{S})}{w_g}\left\|\left(a+b\right)_{G_g}\right\|_2\\
	&=\sum\limits_{g\in M(\overline{S})}{w_g}\left\|a_{G_g}\right\|_2+\sum\limits_{g\in M^{\perp}(\overline{S})}{w_g}\left\|b_{G_g}\right\|_2 =\sum\limits_{g\in M(S)}{w_g}\left\|a_{G_g}\right\|_2+\sum\limits_{g\in M^{\perp}(\overline{S})}{w_g}\left\|b_{G_g}\right\|_2\\
	&=\phi\left(a\right)+\phi\left(b\right) 
\end{align*}
\end{proof}

\begin{lemma}
\label{l3}
If $\lambda_n\geqslant 2 \phi^{\ast}\left(\frac{X^T\varepsilon}{n}\right)$, then  $\phi\left((\hat{\beta}-\beta^*)_{M^\perp(\overline{S})}\right)\leqslant3\phi\left((\hat{\beta}-\beta^*)_{M(\overline{S})}\right)$.
\end{lemma}

\begin{proof}[Proof of Lemma \ref{l3} (also see proposition 9.13 in \texorpdfstring{\cite{wainwright2019high}})]
From equation (\ref{c}), we have
\begin{equation*}
    \phi(\hat{\beta})-\phi\left(\beta^{\ast}\right)\geqslant  \phi\left((\hat{\beta}-\beta^{\ast})_{M^{\perp}(\overline{S})}\right)-
\phi\left((\hat{\beta}-\beta^{\ast})_{M(\overline{S})}\right),
\end{equation*}
On the other hand, by the convexity of the cost function, we have
\begin{align*}
     \frac{1}{n}\big\|Y-X\hat{\beta}\big\|_2^2
	-\frac{1}{n}\big\|Y-X{\beta^{\ast}}\big\|_2^2
    \geqslant\bigg\langle \bigtriangledown\frac{\big\|Y-X{\beta^\ast}\big\|_2^2}{n},\left(\hat{\beta}-\beta^{\ast}\right)\bigg\rangle 
	 \geqslant-\bigg\langle \bigtriangledown\frac{\big\|Y-X{\beta^\ast}\big\|_2^2}{n},\left(\hat{\beta}-\beta^{\ast}\right)\bigg\rangle.
\end{align*}
By applying Holder's inequality with the regularizer $\phi$ and its dual norm $\phi^{\ast}$, we have
$$\bigg|\bigg\langle \bigtriangledown\frac{\left\|Y-X{\beta^{\ast}}\right\|_2^2}{n},\left(\hat{\beta}-\beta^{\ast}\right)\bigg\rangle  \bigg|\leqslant \phi^{\ast}\bigg(\bigtriangledown\frac{\left\|Y-X{\beta^\ast}\right\|_2^2}{n}\bigg)
\phi\left(\hat{\beta}-\beta^{\ast}\right).
$$
Therefore,
\begin{align*}
      \frac{1}{n}\big\|Y-X\hat{\beta}\big\|_2^2
	-\frac{1}{n}\big\|Y-X{\beta^{\ast}}\big\|_2^2 &
    \geqslant-\bigg\langle \bigtriangledown\frac{\left\|Y-X{\beta^\ast}\right\|_2^2}{n},\left(\hat{\beta}-\beta^{\ast}\right)\bigg\rangle 
	\geqslant-\phi^{\ast}\bigg(\bigtriangledown\frac{\left\|Y-X{\beta^\ast}\right\|_2^2}{n}\bigg)\phi\left(\hat{\beta}-\beta^{\ast}\right)\\
    &\geqslant-\frac{\lambda_n}{2} \phi\left(\hat{\beta}-\beta^{\ast}\right) \geqslant-\frac{\lambda_n}{2}
\left(\phi(\hat{\beta}-\beta^{\ast})_{M(\overline{S})} + \phi(\hat{\beta}-\beta^{\ast})_{M^{\perp}(\overline{S})}\right),
\end{align*}
and
\begin{equation*}
\begin{aligned} 0 & =\frac{1}{n}\big\|Y-X\hat{\beta}\big\|_2^2-\frac{1}{n}\big\|Y-X\beta^{\ast}\big\|_2^2+\lambda_n\left(\phi(\hat{\beta})-\phi(\beta^{\ast})\right) \\
& \geq \lambda_n\left(\phi\left((\hat{\beta}-\beta^{\ast})_{M^{\perp}(\overline{S})}\right)-
\phi\left((\hat{\beta}-\beta^{\ast})_{M(\overline{S})}\right)
-2\phi(\beta_{M^{\perp}(S)}^{\ast})\right) -\frac{\lambda_n}{2}
\left(\phi(\hat{\beta}-\beta^{\ast})_{M(\overline{S})} + \phi(\hat{\beta}-\beta^{\ast})_{M^{\perp}(\overline{S})}\right) \\ &
=\frac{\lambda_n}{2}\left(\phi\left((\hat{\beta}-\beta^{\ast})_{M^{\perp}(\overline{S})}\right)-3\phi\left(\hat{\beta}-\beta^{\ast})_{M(\overline{S})}\right)\right),\end{aligned}
\end{equation*}
from which the claim follows.
\end{proof}

\subsection{Proof of Theorem \ref{the:rsccondtion}}

\begin{proof}[Proof of Theorem~\ref{the:rsccondtion} Part 1]

By Lemma~\ref{lem:lemma1inlarent}, we have 
 
 $$ \mathbb{P}\bigg(\frac{|||\frac{X_{G_g}^TX_{G_g}}{n} - \Theta_{G_g,G_g}|||_2}{|||\Theta|_{G_g,G_g}||_2} \leqslant c_5(\sqrt{\frac{d_g}{n}}+\frac{d_g}{n}) + \delta \bigg) >1 - c_4e^{-c_2n\delta^2}$$

By triangle inequality, since $X_{G_g}^TX_{G_g}$ is a positive semi-definite, we have 
\begin{equation*}
    \begin{aligned}
        \gamma_{\max}(\frac{X_{G_g}^TX_{G_g}}{n}) & = |||\frac{X_{G_g}^TX_{G_g}}{n}|||_2 = |||\frac{X_{G_g}^TX_{G_g}}{n} - \Theta_{G_g,G_g}|||_2 + |||\Theta_{G_g,G_g}|||_2\\
        &\leqslant (1+c_5(\sqrt{\frac{d_g}{n}}+\frac{d_g}{n}) + \delta) |||\Theta_{G_g,G_g}|||_2,
    \end{aligned}
\end{equation*}
with probability at least $1 - c_4e^{-c_2n\delta^2}$. Because $|||\Theta_{G_g,G_g}|||_2 \le |||\Theta|||_2 \leqslant c_1$ for some constant $c_1$ and $d_g \leqslant n$, we have $ \gamma_{\max}(\frac{X_{G_g}^TX_{G_g}}{n})  \leqslant c + \delta$ for some constant $c$, with probability at least $1 - e^{-c_2n\delta^2}$. Taking the union probability for all $m$ groups, we have
$$\max_{g\in [m]}\gamma_{\max}(\frac{X_{G_g}^TX_{G_g}}{n}) \le c+\delta$$
with probability at least $1-\exp( -c'2n\delta^2)$ for some constant $c'>0$ as long as
$$\log m \ll n\delta^2.$$
For simplicity, we take $\delta$ as a constant.
\end{proof}

\begin{proof}[Proof of Theorem~\ref{the:rsccondtion} Part 2]
First note that we must have $\rho(\Theta) \le \gamma_{\max}(\Theta) \le c_1$ by Assumptions~\ref{ass:distribution_sub_noise},\ref{ass:distribution_normal_rd}, and \ref{ass:distribution_groupstructure}. By applying Minkowski  inequality, we have 
\begin{align*}
    &\phi(\beta) =  \sum\limits_{g=1}^{m} w_g \left\|\beta_{G_g}\right\|_2
\leqslant  \sqrt{m} \sqrt{\sum\limits_{g=1}^{m} w_g^2 \left\|\beta_{G_g}\right\|^2_2}
\leqslant  \sqrt{m} \sqrt{ \max\limits_{g\in[m]}  w_g^2 h_{\max}^{g} \left\| \beta \right\|^2_2 }
\end{align*}
Let $\beta=\beta^{\ast}-\bar{\beta}$, we now want to prove that 
$\phi \left( \beta_{M^\perp(\bar{S})}\right) \leqslant 3 \phi \left( \beta_{M(\bar{S})}  \right)$ implies $\frac{\left\| X \beta \right\|_2^2}{n} \geqslant \frac{\gamma_{\min}}{64} \left\| \beta \right\|_2^2.$\\
Since $\phi \left( \beta_{M^\perp(\bar{S})}\right) \leqslant 3 \phi \left( \beta_{M(\bar{S})}  \right)$, combining with triangle inequality, we have 
\begin{align*}
    \phi(\beta) & =  \phi \left( \beta_{M(\bar{S})}\right) + \phi \left( \beta_{M^\perp(\bar{S})} \right) \leqslant 4 \phi \left( \beta_{M(\bar{S})} \right) 
\leqslant 4 \sqrt{\overline{s_g}} \sqrt{ \max\limits_{g\in\bar{S}}  w_g^2 h_{\max}^{g}} \left\| \beta_{M(\bar{S})} \right\|_2 \\
 & \leqslant 4 \sqrt{\overline{s_g}} \sqrt{ \max\limits_{g\in\bar{S}} w_g^2 h_{\max}^{g}} \left\| \beta \right\|_2 
\end{align*}
From Lemma~\ref{ll1}, we have 
\begin{equation*}
\begin{split}
\frac{\left\|X\beta\right\|_2}{\sqrt{n}} 
\geqslant & \frac{1}{4} \left\| \Theta^{\frac{1}{2}} \beta \right\|_2 -8\rho(\Theta)\max\limits_{g\in[m]} \frac{1}{w_g \sqrt{h_{\min}^{g}}}
\sqrt{\frac{2  (\log m + d_{\max} \log 5)}{n}} \phi(\beta)\\
\geqslant & \frac{1}{4\sqrt{c_1}} \left\| \beta \right\|_2 - 32\rho(\Theta)\max\limits_{g\in[m]} \frac{1}{w_g \sqrt{h_{\min}^{g}}}
\sqrt{\frac{2  (\log m + d_{\max} \log 5)}{n}} \sqrt{\overline{s_g}} \sqrt{ \max\limits_{g\in\bar{S}} w_g^2 h_{\max}^{g}} \left\| \beta \right\|_2 \\
\geqslant &  \frac{1}{64\sqrt{c_1}} \left\| \beta \right\|_2,
\end{split} 
\end{equation*} 
where the last step is valid due to Assumption \ref{ass:distribution_normal_rd} and \ref{ass:distribution_groupstructure}.
\end{proof}

\subsection{Lemmas for the proof of Theorem \ref{the:rsccondtion}}

\begin{lemma}(Theorem 6.5 in \citep{wainwright2019high})
\label{lem:lemma1inlarent}
\newline
Let $|||.|||_2$ be the spectral norm of a matrix. There are universal constants $c_2, c_3,c_4,c_5$ such that, for any matrix $A \in \mathbb{R}^{n \times p} $, if all rows are drawn i.i.d from $N(0,\Theta)$, then the sample covariance matrix $\hat{\Theta}$ satisfies the bound

\begin{equation*}
    \mathbb{E}\left(e^{t|||\hat{\Theta}-\Theta|||_2}\right) \leqslant e^{c_3\frac{t^2\theta^2}{n} + 4p} \hspace{0.3cm} \text{for all} \hspace{0.2cm} |t| <
    \frac{n}{64e^2|||\Theta|||_2},
\end{equation*}

and hence for all $\delta \in [0,1]$

\begin{equation}
    \label{eq:expspctral}
    \mathbb{P}\left(\frac{|||\hat{\Theta}-\Theta|||_2}{|||\Theta|||_2} \leqslant c_5(\sqrt{\frac{p}{n}}+\frac{p}{n}) + \delta \right) > 1 - c_4e^{-c_2n\delta^2} 
\end{equation}
\end{lemma}
\begin{lemma}
\label{ll1}
Under Assumptions~\ref{ass:distribution_sub_noise},\ref{ass:distribution_normal_rd}, and \ref{ass:distribution_groupstructure}, and use $\rho (\Theta)$ to denote the maximum diagonal of a covariance matrix $\Theta$. For any vector $\beta \in \mathbb{R}^p$ and a given group structure with $m$ groups, we have

\begin{equation}
\label{eq:lemrsc}
\frac{\left\|X\beta\right\|_2}{\sqrt{n}} 
\geq \frac{1}{4} \left\|\Theta^{\frac{1}{2}}\beta\right\|_2
-  8\rho(\Theta) \left( \max\limits_{g\in [m]} \frac{1}{w_g \sqrt{h_{\min}^{g}}}\right)
\sqrt{\frac{2  (\log m + d_{\max} \log 5)}{n}} \phi(\beta),
\end{equation}
with probability at least $1- \frac{e^{-\frac{n}{32}}}{1-e^{-\frac{n}{64}}}.$
\end{lemma}
\begin{proof}[Proof of Lemma~\ref{ll1}]
To begin with, for a vector $\beta \in \mathbb{R}^p$ with a fixed group structure, we define the set  $S^{p-1}(\Theta)= \left\{ \beta \in \mathbb{R}^p \middle| \left\|\Theta^{\frac{1}{2}}\beta\right\|_2 =1\right\}$, the function $$
g(t) = 4\rho(\Theta)\max\limits_{g\in[m]} \frac{1}{w_g \sqrt{h_{\min}^{g}}}
\sqrt{\frac{2  (\log m + d_{\max} \log 5)}{n}} \cdot t
$$
and the event
$$
    \mathcal{E}\left( S^{p-1}(\Theta)\right) = \left\{ X\in \mathbb{R}^{n\times p} \middle| \inf\limits_{\beta\in S^{p-1} (\Theta)} \frac{ \left\| X \beta \right\|_2 }{ \sqrt{n}}+2 g (\phi(\beta)) \leqslant \frac{1}{4}  \right\}.
$$
\noindent
where $\phi(.)$ is the overlapping group lasso regularizer. In addition, given $0 \leqslant r_\ell \leqslant r_u$, we define the set $$\mathbb{K} \left( r_\ell, r_u \right) = \left\{ \beta \in S^{p-1} (\Theta) \middle| g \left( \phi(\beta)\right) \in \left[r_\ell, r_u \right] \right\},$$ and the event:
$$
 \mathcal{A} \left(r_{\ell}, r_{u}\right)=
\left\{ X \in \mathbb{R}^{n \times p} \middle|\inf\limits_{\beta \in \mathbb{K} \left(r_{\ell}, r_{u}\right)} \frac{\left\| X \beta  \right\|_2}{\sqrt{n}} \leqslant \frac{1}{2} - r_u\right\}.
$$\\

Based on lemma \ref{ll1}.1 and lemma \ref{ll1}.2, we have
$$\mathbb{P}\left( X \in \mathcal{E} \right) \leqslant \mathbb{P} \left( \mathcal{A} (0, \upsilon)\right) + \sum\limits_{\ell=1}^{\infty} \mathbb{P} \left(\mathcal{A} ( 2^{\ell-1} \upsilon, 2^{\ell} \upsilon)\right) \leqslant e^{-\frac{n}{32} }  \left\{ \sum\limits_{t=0}^{\infty} e^{ -\frac{n}{8} 2^{2\ell} \upsilon^2} \right\}.$$
Since $\upsilon=\frac{1}{4}$ and $2^{2\ell} \geqslant 2\ell$, we have $\mathbb{P} \left( X \in \mathcal{E}  \right) \leqslant e^{-\frac{n}{32}} \sum\limits_{\ell=0}^{\infty} e^{-\frac{n}{8} 2^{2\ell} \upsilon^2} \leqslant e^{-\frac{n}{32}} \sum\limits_{\ell=0}^{\infty} e^{-n \frac{\ell}{4} \upsilon^2} \leqslant \frac{e^{-\frac{n}{32}}}{1- e^{-\frac{n}{64}}}$.\\
We just get upper bound of $\mathbb{P}\left(X\in \mathcal{E}\right)$. We next show that the bound in \eqref{eq:lemrsc} always hold on the complementary set $\mathcal{E} ^c.$\\
If $ X\notin \mathcal{E}$, based on the definition of $\mathcal{E}$, we have
$
\inf\limits_{\beta \in S^{p-1} (\Theta)} \frac{\left|X\beta\right\|_2}{\sqrt{n}} \geqslant \frac{1}{4} - 2 g \left(\phi(\beta)\right).
$  That is $\forall \beta \in S^{p-1} (\Theta)$. 
$\frac{\left\|X\beta\right\|_2}{\sqrt{n}} \geqslant \frac{1}{4} - 2 g \left( \phi (\beta)\right)$.
Therefore, for any $\beta'\in \{\beta'\in \mathbb{R} | \frac{\beta'}{\left\|\Theta^{\frac{1}{2}}\beta'\right\|_2} \in S^{p-1} (\Theta)\}$, we have 
\begin{align*}
    &\frac{\Big\|X\frac{\beta'}{\big\|\Theta^{\frac{1}{2}}\beta'\big\|_2}\Big\|_2}{\sqrt{n}}
\geqslant \frac{1}{4} - 2g \bigg( \phi \bigg( \frac{ \beta' }{ \big\| \Theta^{\frac{1}{2}}\beta' \big\|_2} \bigg) \bigg) \\
 &\frac{\Big\|X\beta'\Big\|_2}{\sqrt{n}}
\geqslant
\frac{1}{4} \big\|\Theta^{\frac{1}{2}}\beta'\big\|_2
- 2g \Big( \phi (\beta')\Big),\\
\end{align*}

We finish the proof by substituting the definition of $g(\phi(\beta))$. 
\end{proof}

\noindent\textbf{Lemma \ref{ll1}.1} For $\upsilon=\frac{1}{4}$, we have
$
\mathcal{E}\subseteq
\mathcal{A} (0, \upsilon) \cup 
\left(\bigcup_{\ell=1}^{\infty} \mathcal{A} \left( 2^{\ell-1}\upsilon,2^{\ell}\upsilon\right)\right).
$\\
\textbf{Lemma \ref{ll1}.2}  For any pair $\left(r_\ell, r_u \right)$, where $0\leqslant r_\ell \leqslant r_u$, we have $\mathbb{P}\left(\mathcal{A} \left(r_\ell, r_u \right) \right) \leqslant {e^{-\frac{n}{32}} e^{-\frac{n}{8} r_u^2}}.$
\begin{proof}

\begin{proof}[Proof of Lemma \ref{ll1}.1]
By definition, $\mathbb{K}(0,\upsilon) \cup \left(\bigcup_{\ell=1}^{\infty} \mathbb{K} \left( 2^{\ell-1}\upsilon, 2^{\ell}\upsilon\right)\right)$ is a cover of  $S^{p-1} (\Theta)$. Therefore, for any $\beta$, it either belongs to $\mathbb{K}(0,\upsilon)$ or $\mathbb{K} \left( 2^{\ell-1} \upsilon, 2^{\ell} \upsilon \right)$.\\
\textbf{Case 1} If $\beta \in \mathbb{K}(0,\upsilon)$, by definition, we have $ g \left( \phi(\beta)\right) \in \left[0, \upsilon \right]$ and   
\begin{equation*}
\frac{ \left\| X \beta \right\|_2 }{ \sqrt{n} }
\leqslant \frac{1}{4} - 2 g \left(\phi(\beta)\right) 
\leqslant \frac{1}{4} = \frac{1}{2} - \upsilon.\\
\end{equation*}
Therefore, the event $\mathcal{A} (0,\upsilon)$ must happen in this case.\\
\textbf{Case 2}: If $ \beta \notin \mathbb{K} (0,\upsilon) $, we must have $ \beta \in \mathbb{K} \left( 2^{\ell-1} \upsilon, 2^{\ell} \upsilon \right) $ for some $\ell=1,2,\cdots$, and moreover
\begin{equation*}
\frac{ \left\| X \beta \right\|_2 }{ \sqrt{n} }
\leqslant \frac{1}{4} - 2 g \left(\phi(\beta)\right) 
\leqslant \frac{1}{4} - 2 \cdot \left( 2^{\ell-1}\upsilon \right)
\leqslant \frac{1}{2} - \left(2 \cdot 2^{\ell-1}\right) \upsilon
\leqslant \frac{1}{2} - 2^{\ell} \upsilon.\\
\end{equation*}
So that the event $\mathcal{A} \left( 2^{\ell-1}\upsilon, 2^{\ell}\upsilon \right)$ must happen. Therefore, $\mathcal{E}\subseteq \mathcal{A}(0,\upsilon) \cup \left(\bigcup\limits_{\ell=1}^{\infty} \mathcal{A} \left( 2^{\ell-1}\upsilon,2^{\ell}\upsilon\right)\right). $ 
\end{proof}

\begin{proof}[Proof of Lemma~\ref{ll1}.2]
To prove Lemma~\ref{ll1}.2, we define and bound the random variable $T \left(r_\ell, r_u \right) = - \inf\limits_{\beta \in \mathbb{K} \left(r_{\ell}, r_{u}\right)} \frac{\left\| X \beta\right\|_2}{\sqrt{n}}$. 
Let $S^{n-1}$ be a unit ball on $\mathbb{R}^n$, by the variational representation of the $\ell_2$-norm, we have
\begin{equation*}
\begin{split}
T \left(r_\ell, r_u \right) 
= & - \inf\limits_{\beta \in \mathbb{K} \left(r_{\ell}, r_{u}\right)} \frac{\left\| X \beta\right\|_2}{\sqrt{n}}
= - \inf\limits_{\beta \in \mathbb{K} \left(r_{\ell}, r_{u}\right)} \sup\limits_{u \in S^{n-1}}
\frac{ \left\langle u, X \beta \right\rangle }{\sqrt{n}}
= \sup\limits_{\beta \in \mathbb{K} \left(r_{\ell}, r_{u}\right)} \inf\limits_{u \in S^{n-1}}
\frac{ \left\langle u, X \beta \right\rangle }{\sqrt{n}}.\\
\end{split}
\end{equation*}

Let $ X = W \Theta^{\frac{1}{2}}$, where $W \in \mathbb{R}^{n\times p}$ is a standard Gaussian matrix, and define the transformed vector $v = \Theta^{\frac{1}{2}} \beta$, then
\begin{equation*}
T \left(r_\ell, r_u \right)= \sup\limits_{\beta \in \mathbb{K} \left(r_{\ell}, r_{u}\right)} \inf\limits_{u \in S^{n-1}}
\frac{ \left\langle u, X \beta \right\rangle }{\sqrt{n}} 
= \sup\limits_{v \in \bar{\mathbb{K}} \left(r_{\ell}, r_{u}\right)}
\inf\limits_{u \in S^{n-1}} 
\frac{ \left\langle u, Wv \right\rangle }{\sqrt{n}}, 
\end{equation*}
where $\bar{\mathbb{K}} \left( r_{\ell}, r_{u} \right) = \left\{ v \in \mathbb{R}^p \middle| \left\|v\right\|_2=1, g \left(\phi( \Theta^{-\frac{1}{2}} v) \right) \in \left[r_\ell, r_u \right] \right\}$.

Define $Z_{u,v} = \frac{ \left\langle u, Wv \right\rangle }{\sqrt{n}}$, since $(u, v)$ range over a subset of $S^{n-1} \times S^{p-1}$, each variable $Z_{u,v}$ is zero-mean Gaussian with variance $n^{-1}$. We compare the Gaussian process $Z_{u,v}$ to the zero-mean Gaussian process $Y_{u,v}$ which defined as:
\begin{equation*}
Y_{u,v}
=
\frac{ \left\langle \zeta, u \right\rangle }{\sqrt{n}}
+ \frac{ \left\langle \xi, v \right\rangle }{\sqrt{n}}
~~~~~~~~\text{where }
\zeta \in \mathbb{R}^n , \xi \in \mathbb{R}^p,\text{have i.i.d $N(0,1)$ entries.}
\end{equation*}

Next, we show that the $Y_{u,v}$ and $Z_{u,v}$ defined above satisfy conditions in Gordon's inequality. By definition, we have 
\begin{equation}\label{d}
\begin{aligned}
     \mathbb{E} \left(Z_{u,v} - Z_{u',v'} \right)^2 &=  \mathbb{E} \left( \frac{\left\langle u, Wv\right\rangle}{\sqrt{n}} - \frac{\left\langle u', Wv'\right\rangle}{\sqrt{n}} \right)^2 
     =\frac{1}{n} \sum_{i=1}^n \sum_{j=1}^p \left( u_i v_j - u'_i v'_j \right)^2 \\&
     =\frac{1}{n} \sum_{i=1}^n \sum_{j=1}^p \left( u_i v_j - u'_i v_j + u'_i v_j - u'_i v'_j \right)^2 \\&
     =\frac{1}{n} \left( \left\|v\right\|_2^2 \left \| u-u'\right\|_2^2 + \left\|u'\right\|_2^2 \left \|   v-v'\right\|_2^2 + 2\left( \left\|v\right\|_2^2 - \left\langle v, v' \right\rangle\right)\left( \left\langle u, u' \right\rangle - \left\|u\right\|_2^2 \right) \right),
     \end{aligned}
\end{equation}

On one hand, since $\left\|v\right\|_2^2 \leqslant 1$, $\left\|u'\right\|_2^2 \leqslant 1$, $(7) \leqslant \frac{1}{n} \left( \left\|u-u'\right\|_2^2 + \left\|v-v'\right\|_2^2\right)$.

On the other hand, we have
\begin{equation}
\label{e}
\begin{aligned}
     \mathbb{E}\left(Y_{u,v} - Y_{u',v'} \right)^2 &=  \mathbb{E} \left( \frac{\left\langle \zeta, u - u' \right \rangle}{\sqrt{n}} + \frac{\left\langle \xi, v - v' \right \rangle}{\sqrt{n}} \right)^2 \\&
     = \frac{1}{n} \left(\sum_{i = 1}^n\sum_{j = 1}^p (u - u')^2 + \sum_{i = 1}^n\sum_{j = 1}^p (v - v')^2 \right) 
     = \frac{1}{n} \left( \left\|u-u'\right\|_2^2 + \left\|v-v'\right\|_2^2\right).
\end{aligned}
\end{equation}

Taking  equation (\ref{d}) and (\ref{e}) together, we have $$ \mathbb{E} \left(Z_{u,v} - Z_{u',v'} \right)^2 \leqslant \frac{1}{n} \left( \left\|u-u'\right\|_2^2 + \left\|v-v'\right\|_2^2\right) =  \mathbb{E} \left(Y_{u,v} - Y_{u',v'} \right)^2.$$ 

If $V=V^{\prime}, \text { then } n  \mathbb{E}\left(\left(Z_{u, v}-Z_{u^{\prime}, v^{\prime}}\right)^{2}\right)=\left\|u-u^{\prime}\right\|_{2}=n  \mathbb{E}\left(\left(Y_{u, v}-Y_{u^{\prime}, v^{\prime}}\right)^{2}\right).$

By applying Gordon's inequality, we have $$ \mathbb{E}\left( \sup\limits_{v \in \tilde{\mathbb{K}} \left(r_{\ell}, r_{u}\right)}  \inf\limits_{u \in S^{n-1}} Z_{u,v}\right) \leqslant  \mathbb{E}\left( \sup\limits_{v \in \tilde{\mathbb{K}} \left(r_{\ell}, r_{u}\right)}  \inf\limits_{u \in S^{n-1}} Y_{u,v}\right).$$ Therefore,
\begin{equation*}
\begin{split}
 \mathbb{E} \left( T \left( r_{\ell}, r_{u} \right) \right)
&=  
 \mathbb{E} \left( \sup\limits_{v \in \tilde{\mathbb{K}} \left(r_{\ell}, r_{u}\right)}  \inf\limits_{u \in S^{n-1}} \frac{ \left\langle u, Wv \right\rangle }{\sqrt{n}} \right) 
\leqslant 
 \mathbb{E} \left( \sup\limits_{v \in \tilde{\mathbb{K}} \left(r_{\ell}, r_{u}\right)} \inf\limits_{u \in S^{n-1}} \left( \frac{ \left\langle \xi, v \right\rangle }{\sqrt{n}} + \frac{ \left\langle \zeta, u \right\rangle }{\sqrt{n}} \right)\right)
\\
& = 
 \mathbb{E} \left( \sup\limits_{\beta \in \mathbb{K} \left(r_{\ell}, r_{u}\right)} \frac{ \left\langle \Sigma^{\frac{1}{2}}\xi, \beta \right\rangle }{\sqrt{n}} \right)
 -  \mathbb{E} \left( \frac{ \left\| \zeta \right\|_2 }{\sqrt{n}} \right)
\end{split} 
\end{equation*}

Next, we bound these two terms. For the second term, we have $ \mathbb{E} \left( \frac{ \left\| \zeta \right\|_2 }{\sqrt{n}} \right) =  \mathbb{E} \left( \sqrt{\frac{\xi_1^2 + \ldots + \xi_n^2}{n}} \right) \geqslant  \mathbb{E} \left( \frac{|\xi_1| + \ldots + |\xi_n|}{n} \right)= \sqrt{\frac{2}{\pi}}$.  For the first term,  we have
$ \mathbb{E} \left( \sup\limits_{\beta \in \mathbb{K} \left(r_{\ell}, r_{u}\right)} \frac{ \left\langle \Theta^{\frac{1}{2}}\xi, \beta \right\rangle }{\sqrt{n}} \right) \leqslant 
 \mathbb{E} \left( \sup\limits_{\beta \in \mathbb{K} \left(r_{\ell}, r_{u}\right)} \frac{\phi(\beta) \phi^{\ast}(\Theta^{\frac{1}{2}}\xi) }{\sqrt{n}} \right)$, where $\phi^{\ast}(\Theta^{\frac{1}{2}}\xi)$ is the the dual norm defined before. Since $ \beta \in \mathbb{K} \left( r_{\ell}, r_{u} \right) $, $ g \left( \phi (\beta) \right) \leqslant r_u$, by the definition of $g(t)$, we have 
\begin{equation}
\label{f}
    \phi(\beta) \leqslant \frac{r_u}{ \left(4\rho(\Theta)\max\limits_{g\in[m]} \frac{1}{w_g \sqrt{h_{\min}^{g}}}
\sqrt{\frac{2  (\log m + d_{\max} \log 5)}{n}}\right)}.
\end{equation}

Let $\eta_{G_g} = (\Theta^{\frac{1}{2}}\xi)_{G_g} $, to bound $ \mathbb{E} \left( \max\limits_{g} \left\|(\Theta^{\frac{1}{2}}\xi)_{G_g}\right\|_2 \right) =  \mathbb{E} \left( \max\limits_{g} \left\|\eta_{G_g}\right\|_2 \right) $. Since $\Theta^{\frac{1}{2}} \xi \sim N(0,\Theta)$, by the properties of normal distribution, its corresponding marginal distribution of $j\text{th}$ variable $(\Theta^{\frac{1}{2}} \xi)_{j}$ also follows zero mean normal distribution with covariance matrix $\Theta_{jj}$, which is the $j\text{th}$ diagonal elements of $\Theta$. Therefore, any subset of $\Theta^{\frac{1}{2}} \xi$ is a zero-mean sub-Gaussian random sequence with parameters at most $\rho(\Theta)$. By equation (\ref{f}) and Lemma~\ref{ll1}.2.3, we have
\begin{align*}
      \mathbb{E} \bigg( \sup\limits_{\beta \in \mathbb{K} \left(r_{\ell}, r_{u}\right)} \frac{ \phi (\beta) \phi^{\ast}  \Theta^{\frac{1}{2}} \xi }{\sqrt{n}} \bigg)
&\leqslant  \mathbb{E}  \bigg( \sup\limits_{\beta \in \mathbb{K} \left(r_{\ell}, r_{u}\right)} \frac{r_u}{  \Big(4\rho(\Theta)\big(\max\limits_{g\in[m]} \frac{1}{w_g h_{\min}^{g}}\big)
\sqrt{\frac{2  (\log m + d_{\max} \log 5)}{n}} \Big)}  \frac{\phi^{\ast} \big(\Theta^{\frac{1}{2}} \xi\big)}{ \sqrt{n} }  \bigg)\\
&= \frac{r_u}{ \Big(4\rho(\Theta)\big(\max\limits_{g\in[m]} \frac{1}{ w_g h_{\min}^{g}}\big)
\sqrt{\frac{2  (\log m + d_{\max} \log 5)}{n}}\Big)}
 \mathbb{E} \bigg(  \frac{\phi^{\ast} \left(\Theta^{\frac{1}{2}} \xi \right)}{ \sqrt{n} }\bigg)\\
 &\leqslant \frac{r_u}{ \Big(4\rho(\Theta)\big(\max\limits_{g\in[m]} \frac{1}{w_g h_{\min}^{g}}\big)
\sqrt{\frac{2  (\log m + d_{\max} \log 5)}{n}}\Big)}
 \mathbb{E} \left( \max\limits_{g\in[m]} \frac{1}{\sqrt{n} w_g} \left\|H \left(\Theta^{\frac{1}{2}} \xi\right)_{G_g}\right\|_2\right)\\
&\leqslant\frac{r_u}{ \Big(4\rho(\Theta)\big(\max\limits_{g\in[m]} \frac{1}{w_g h_{\min}^{g}}\big)
\sqrt{\frac{2  (\log m + d_{\max} \log 5)}{n}}\Big)}
 \mathbb{E} \left( \max\limits_{g\in[m]} \frac{1}{\sqrt{n} w_g h_{\min}^{g}} \left\| \left(\Theta^{\frac{1}{2}} \xi\right)_{G_g}\right\|_2\right)\\
&\leqslant\frac{r_u}{ \Big(4\rho(\Theta)
\sqrt{\frac{2  (\log m + d_{\max} \log 5)}{n}}\Big)}
 \mathbb{E}\left(\left\| \max\limits_{g\in[m]} \left(\Theta^{\frac{1}{2}} \xi\right)_{G_g}\right\|_2\right)\\
& \leqslant\frac{r_u}{ \Big(4\rho(\Theta)
\sqrt{\frac{2  (\log m + d_{\max} \log 5)}{n}}\Big)}
\left( 2\rho(\Theta) \sqrt{ \left(\log m + d_{\max} \log 5\right) 2 \sigma^2}\right)\leqslant \frac{r_u}{2}
\end{align*}

Therefore, $ \mathbb{E}\left[T\left(r_{\ell},r_u\right)\right] \leqslant -\sqrt{\frac{2}{\pi}} + \frac{r_u}{2}$. Next we want to bound  $\mathbb{P} \left( T \left( r_{\ell}, r_u \right) \geqslant - \frac{1}{2} +  r_u \right)$ based on the bound of this expectation. To apply Lemma~\ref{ll1}.2.4, we first show that, the $f=T(r_l,r_u)$, a function of the random variable $W$ is a $\frac{1}{\sqrt{n}}$-Lipschitz function and without making confusion, we denote the corresponding function as $T(W)$. For any standard Gaussian matrix $W_1$ and $W_2$, we have
\begin{align*}
    \left| T(W_1) - T(W_2) \right| & = \left| \sup\limits_{v \in \tilde{\mathbb{K}} \left(r_{\ell}, r_{u}\right)}  \inf\limits_{u \in S^{n-1}} \frac{\left\langle u,W_1 v \right\rangle }{\sqrt{n}} - \sup\limits_{v \in \tilde{\mathbb{K}} \left(r_{\ell}, r_{u}\right)}  \inf\limits_{u \in S^{n-1}} \frac{\left\langle u, W_2 v\right\rangle }{\sqrt{n}} \right| \\
    & =  \left| \sup\limits_{v \in \tilde{\mathbb{K}} \left(r_{\ell}, r_{u}\right)}  \left(-\frac{\left\|W_1v\right\|_2 }{\sqrt{n}}\right) - \sup\limits_{v \in \tilde{\mathbb{K}} \left(r_{\ell}, r_{u}\right)}  \left(-\frac{\left\|W_2v\right\|_2 }{\sqrt{n}}\right) \right| \\
     & =  \left|  \left(-\inf\limits_{v \in \tilde{\mathbb{K}} \left(r_{\ell}, r_{u}\right)}\frac{\left\|W_1v\right\|_2 }{\sqrt{n}}\right) -  \left(-\inf\limits_{v \in \tilde{\mathbb{K}} \left(r_{\ell}, r_{u}\right)}\frac{\left\|W_2v\right\|_2 }{\sqrt{n}}\right) \right| \\
      & =  \left|  \inf\limits_{v \in \tilde{\mathbb{K}} \left(r_{\ell}, r_{u}\right)}\frac{\left\|W_2v\right\|_2 }{\sqrt{n}} - \inf\limits_{v \in \tilde{\mathbb{K}} \left(r_{\ell}, r_{u}\right)}\frac{\left\|W_1v\right\|_2 }{\sqrt{n}} \right|.
\end{align*}
\noindent

Suppose that  $\frac{\left\|W_1v_1\right\|_2 }{\sqrt{n}} = \inf\limits_{v \in \tilde{\mathbb{K}} \left(r_{\ell}, r_{u}\right)}\frac{\left\|W_1v\right\|_2 }{\sqrt{n}}$ and     
 $\frac{\left\|W_2v_2\right\|_2 }{\sqrt{n}} = \inf\limits_{v \in \tilde{\mathbb{K}} \left(r_{\ell}, r_{u}\right)}\frac{\left\|W_2v\right\|_2 }{\sqrt{n}}$.
 
\textbf{Case 1} If $\left\|W_1v_1\right\|_2 > \left\|W_2v_2\right\|_2$, then we have
\begin{align*}
    \left| T(W_1) - T(W_2) \right| & =  \left|  \inf\limits_{v \in \tilde{\mathbb{K}} \left(r_{\ell}, r_{u}\right)}\frac{\left\|W_2v\right\|_2 }{\sqrt{n}} - \inf\limits_{v \in \tilde{\mathbb{K}} \left(r_{\ell}, r_{u}\right)}\frac{\left\|W_1v\right\|_2 }{\sqrt{n}} \right|\\
     & =  \frac{\left\|W_1v_1\right\|_2 - \left\|W_2v_2\right\|_2}{\sqrt{n}} \leqslant
      \frac{\left\|W_1v_2\right\|_2 - \left\|W_2v_2\right\|_2}{\sqrt{n}}\\
      &\leqslant \frac{\left\|(W_1-W_2)v_2\right\|_2}{\sqrt{n}} \leqslant \frac{\left\|W_1 - W_2\right\|_F}{\sqrt{n}}\\.
\end{align*}

\textbf{Case 2} If $\left\|W_1v_1\right\|_2 \leqslant \left\|W_2v_2\right\|_2$, then we have
\begin{align*}
    \left| T(W_1) - T(W_2) \right| & =  \left|  \inf\limits_{v \in \tilde{\mathbb{K}} \left(r_{\ell}, r_{u}\right)}\frac{\left\|W_2v\right\|_2 }{\sqrt{n}} - \inf\limits_{v \in \tilde{\mathbb{K}} \left(r_{\ell}, r_{u}\right)}\frac{\left\|W_1v\right\|_2 }{\sqrt{n}} \right|\\
     & =  \frac{\left\|W_2v_2\right\|_2 - \left\|W_1v_1\right\|_2}{\sqrt{n}} \leqslant
      \frac{\left\|W_2v_1\right\|_2 - \left\|W_1v_1\right\|_2}{\sqrt{n}}\\
      &\leqslant \frac{\left\|(W_1-W_2)v_1\right\|_2}{\sqrt{n}} \leqslant \frac{\left\|W_1 - W_2\right\|_F}{\sqrt{n}}\\.
\end{align*}
where $\left\|.\right\|_F$ represent the Frobenious norm of a matrix. Thus under the Euclidean norm, $T(W)$ is a $\frac{1}{\sqrt{n}}$-Lipschitz function. Therefore, by lemma \ref{ll1}.2.3, we have
\begin{equation*}
    \mathbb{P}(T(r_l,r_u)-  \mathbb{E}(T(r_l,r_u)) \geqslant t) \leqslant e^{-nt^2/2}, \forall t\geqslant 0 
\end{equation*}
  
Set t = $\sqrt{\frac{2}{\pi}} - \frac{1}{2}+\frac{r_u}{2} \geqslant \frac{1}{4} + \frac{r_u}{2}$,
we have, $ \mathbb{E}(T(r_l, r_u)) + t \leqslant -\frac{1}{2} + r_u$ and $\mathbb{P} \left[ T \left( r_{\ell}, r_u \right) \geqslant - \frac{1}{2} +  r_u \right] \leqslant e^{-\frac{n}{32}} e^{-\frac{n}{8} r_u^2},$ which is actually the Lemma~\ref{ll1}.2 
\end{proof}

\noindent\textbf{Lemma \ref{ll1}.2.1  (Gordon's Inequality)}
Let $\left\{Z_{u, v}\right\}_{u \in U, v \in V}$ and $\left\{Y_{u, v}\right\}_{u \in U, v \in V}$ be zero-mean Gaussian process indexed by a non-empty index set $I=U \times V$. If 

1. $ \mathbb{E}\left(\left(Z_{u, v}-Z_{u^{\prime} v^{\prime}}\right)^{2}\right) \leq  \mathbb{E}\left(\left(Y_{u, v}-Y_{u^{\prime},v^{\prime}}\right)^{2}\right) \text { for all pairs }(u, v) \operatorname{and}\left(u^{\prime}\, v^{\prime}\right) \in I$

2. $ \mathbb{E}\left(\left(Z_{u, v}-Z_{u^{\prime} v}\right)^{2}\right)= \mathbb{E}\left(\left(Y_{u, v}-Y_{u^{\prime},v}\right)^{2}\right),$\\
then we have $ \mathbb{E}(\max\limits_{v \in V} \min\limits_{u \in U} Z_{u,v}) \leq  \mathbb{E}(\max\limits_{v \in V} \min\limits_{u \in U} Y_{u,v}).$\\
\textbf{Lemma \ref{ll1}.2.2} Suppose that $\alpha = (\alpha_{1},...,\alpha_{d}),$ where each $\alpha_{i}, i \in [d]$   is a zero-mean sub-Gaussian random variable with parameter at most $\sigma^2$, then for any $t\in \mathbb{R}$, we have  $ \mathbb{E} \left( \exp\left(t\left\|\alpha\right\|_2\right) \right) \leqslant 5^d \exp \left(2 t^2 \sigma^2\right).$\\
\textbf{Lemma \ref{ll1}.2.3} Suppose that $\alpha = (\alpha_{1},...,\alpha_{d}),$ where each $\alpha_{i}, i \in [d]$   is a zero-mean sub-Gaussian random variable with parameter at most $\sigma^2$, and for a given group structure $G$, let $\left\|\alpha_{G_g} \right\| $ be the corresponding group norm, $m$ be the number of groups and $d_{max}$ be the maximum group size, then 
\begin{equation*}
     \mathbb{E}\left(\max\limits_{g} \left\|\alpha_{G_g} \right\| \right) \leqslant 2 \sqrt{2 \sigma^2 \left( \log m +d_{max} \log 5\right)}\\
\end{equation*}

\textbf{Lemma \ref{ll1}.2.4 (Theorem 2.26 in \citep{wainwright2019high}):} Let $x = \left(x_1, \cdots, x_n\right)$ be a vector of i.i.d standard Gaussian variable, and $ f: \mathbb{R}^n \to \mathbb{R} $ be a $L$-Lipschitz, with respect to the Euclidean norm, then $f(x)- \mathbb{E} f(x)$ is sub-Gaussian with parameter at most $L$, and hence $ \mathbb{P}\left( \left(f(x) - { \mathbb{E}}\left[f(x)\right)\right) \geqslant t \right] \leqslant e^{-\frac{t^2}{2L^2}}$, $\forall t\geqslant 0$.

\subsubsection{Proof of Lemma~\ref{ll1}.2.2}

We can find a $\frac{1}{2}$ - cover  of $S^{d-1}$, and for any $u\in S^{d-1}$ in the Euclidean norm with cardinally at most $N \leqslant 5^d$, say there exists $u^{q(u)} \in \left\{u^1, \ldots, u^N\right\}$, such that $\left\|u^{q(u)}-u\right\|_2\leqslant\frac{1}{2}$. 

By the variational representation of the $\ell_2$ norm, we have $\left\|\alpha\right\|_2  = \max\limits_{u \in S^{d-1}} \left\langle u,\alpha\right\rangle \leqslant \max\limits_{q(u)\in [N]} \left\langle u^{q(u)},\alpha \right\rangle + \frac{1}{2} \left\|\alpha\right\|_2 $. Therefore, $\left\|\alpha\right\|_2 \leqslant 2 \max\limits_{q(u)\in [N]} \left\langle u^{q(u)},\alpha \right\rangle$.  Consequently, 
\begin{align*}
    \mathbb{E}\left(\exp\left(t\left\|\alpha\right\|_2 \right)\right)  &\leqslant  \mathbb{E} \left(\exp \left(2 t \max\limits_{q\in [N]} \left\langle u^{q},\alpha \right\rangle \right)\right) 
    =  \mathbb{E} \left(\max\limits_{q\in [N]}\exp \left(2 t  \left\langle u^{q},\alpha \right\rangle \right)\right) \\
    & \leqslant  \sum\limits_{q=1}^{N}  \mathbb{E} \left( \exp \left( 2 t \left\langle u^{q}, \alpha \right\rangle \right)\right) \leqslant 5^d \exp \left( \frac{ 4 t^2 \sigma^2 }{ 2 } \right)
      \leqslant 5^d \exp \left(2 t^2 \sigma^2\right).  \\
\end{align*} 
\subsubsection{Proof of Lemma~\ref{ll1}.2.3}

For any $t > 0$, by Jensen's inequality, we have $\exp\left(t  \mathbb{E} \left( \max\limits_{g} \left\| \alpha_{G_g} \right\| \right)\right) \leqslant  \mathbb{E} \left( \exp \left( t \max\limits_{g} \left\| \alpha_{G_g} \right\|_2 \right)\right) $\\

$ =  \mathbb{E} \left( \max\limits_{j} \exp \left( t \left\| \alpha_{G_g} \right\|_2 \right) \right)  \leqslant \sum\limits_{j=1}^{m}  \mathbb{E} \left( \exp \left( t \left\|\alpha_{G_g}\right\|_2\right)\right) \leqslant \sum\limits_{j=1}^{m} 5^{d_g} \exp \left( 2 t^2 \sigma^2 \right) \leqslant m \cdot 5^{d_{\max}} \cdot \exp(2 t^2 \sigma^2) $.\\

By taking $\log$ at both sides, we have
$  t \mathbb{E} \left( \max\limits_{g} \left\| \alpha_{G_g}\right\| \right) \leqslant \log m + d_{\max} \log 5 + 2 t^2 \sigma^2$. That is $ \mathbb{E} \left( \max\limits_{g} \left\| \alpha_{G_g}\right\| \right) \leqslant \frac{ \log m + d_{\max} \log 5 + 2t^2 \sigma^2 }{ t }$.

Let $ t = \sqrt{\frac{ \log m + d_{\max} \log 5 }{ 2 \sigma^2}} $, we have $   \mathbb{E} \left( \max\limits_{g} \left\| \alpha_{G_g}\right\| \right) \leqslant 2 \sqrt{ \left(\log m + d_{\max} \log 5\right) 2 \sigma^2}$. 

\end{proof}

\newpage

\subsection{Proof of Theorem~\ref{the:lowerbound}}

The two lemmas below are integral to the proof:
\begin{lemma}[Packing Number for Binary Sets]
\label{lem:lowerbound1}
Consider a set $A$ defined for real numbers $m, s_g$ as
\[
A=\left\{a \in\{0,1\}^{m} \mid \sum_{j=1}^{m} a_{j} \leq s_{g}\right\}.
\]
Then the $\sqrt{\frac{s_g}{2}}$-packing number of set $A \geqslant
 \frac{\left(\begin{array}{l}m \\ s_{g}\end{array}\right) - \hspace{0.2cm}2}{\left(\begin{array}{c}
	m \\
	{\left\lfloor\frac{s_{g}}{2}\right\rfloor}
	\end{array}\right) \cdot 2^{\frac{s_{g}}{2}}}
$, and $$\log \Bigg( \frac{\left(\begin{array}{l}m \\ s_{g}\end{array}\right) - \hspace{0.2cm}2}{\left(\begin{array}{c}
	m \\
	{\left\lfloor\frac{s_{g}}{2}\right\rfloor}
	\end{array}\right) \cdot 2^{\frac{s_{g}}{2}}} \Bigg) \asymp s_g \log (\frac{m}{s_g}).$$
\end{lemma}

\begin{lemma}[Packing Number for Sparse Group Vectors]
\label{lem:lowerbound2}
For the set $\Omega(G,s_g)$,  the $\sqrt{\frac{2ds_g}{5}}$-packing number $\gtrsim \frac{\left(\begin{array}{l}m \\ s_{g}\end{array}\right) - 2}{\left(\begin{array}{c}
	m \\
	{\left\lfloor\frac{s_{g}}{2}\right\rfloor}
	\end{array}\right) \cdot 2^{\frac{s_{g}}{2}}} \cdot (\sqrt{2})^{ds_g},$ and $$\log \Bigg( \frac{\left(\begin{array}{l}m \\ s_{g}\end{array}\right) - 2}{\left(\begin{array}{c}
	m \\
	{\left\lfloor\frac{s_{g}}{2}\right\rfloor}
	\end{array}\right) \cdot 2^{\frac{s_{g}}{2}}} \cdot (\sqrt{2})^{ds_g} \Bigg) \asymp s_g(d + \log (\frac{m}{s_g})).$$
\end{lemma}

\begin{proof}[Proof of Theorem~\ref{the:lowerbound}]

	First, select $N$ points $\omega^{(1)}, \ldots, \omega^{(N)}$ from $\Omega(G,s_g)$ such that $\left\|\omega^{(i)}-\omega^{(j)}\right\| > \sqrt{\frac{2ds_g}{5}}$ for all distinct $i,j$. Clearly, $\left\|\omega^{(i)}-\omega^{(j)}\right\| \leqslant \sqrt{4s_g d}$.

 Define $\beta^{(i)}=r \omega^{(i)}$ for each $i$. This results in
\[
\frac{2ks_gr^2}{5} \leq\left\|\beta^{(i)}-\beta^{(j)}\right\|_{2}^{2} \leqslant 4 s_g d r^{2}.
\]

Next, let $y^{(i)}=X \beta^{(i)}+\varepsilon$ for $1 \leqslant i \leqslant N$. Consider the Kullback-Leibler divergence between different distribution pairs: 
\[
D_{K L}\left((y^{(i)}, X),(y^{(j)}, X)\right)=\mathbb{E}_{(y^{(j)}, X)}\left[\log \left(\frac{p\left(y^{(i)}, X\right)}{p\left(y^{(j)}, X\right)}\right)\right].
\]
where $p\left(y^{(i)}, X\right)$ is the probability density of $\left(y^{(i) }, X\right)$. Conditioning on $X$, we have
\[
\mathbb{E}_{(y^{(j)},X)}\left[\log\left(\frac{p\left(y^{(i)}, X\right)}{p\left(y^{(j)}, X\right)}\right) \mid X\right]=\frac{\|X(\beta^{(i)}-\beta^{(j)}) \|_2^2}{2\sigma ^2}.
\]

Thus, for $1 \leq i \neq j \leq N,$
\[
\begin{aligned}
&D_{K L}\left(\left(y^{(i)}, X\right),\left(y^{(j)}, X\right)\right)=\mathbb{E}_{X} \frac{\left\|X\left(\beta^{(i)}-\beta^{(j)}\right)\right\|_{2}^{2}}{2 \sigma ^{2}}=\frac{n (\beta^{(i)}-\beta^{(j)})^{\top} \Sigma (\beta^{(i)}-\beta^{(j)})}{2 \sigma ^{2}} \\
&\leq \frac{3c_1\left\|\beta^{(i)}-\beta^{(j)}\right\|_{2}^{2}}{2 \sigma ^{2}} \leq \frac{2c_1 n d r^{2} s_g}{\sigma ^{2}}.
\end{aligned}
\]
	
From Lemma~\ref{lem:lowerbound2}, $\log N \asymp s_g\left(d+\log \frac{m}{s_g}\right)$. Setting $\frac{\frac{n d r^{2} s_g}{\sigma^{2}}+\log 2}{\log N} = \frac{1}{2}$, we obtain 
\[
r \gtrsim \sqrt{\frac{\left(d +\log \frac{m}{s_g}\right) \sigma^{2}}{3n d}}.
\]

By generalized Fano's Lemma, $\inf\limits_{\hat{\beta}}\sup\limits_{\beta} \mathbb{E}\|\hat{\beta}-\beta\|_{2} \geqslant \sqrt{\frac{2r^2ks_g}{5}}\left(1-\frac{\frac{n d r^{2} s_g}{\sigma^{2}}+\log 2}{\log N}\right)$. Consequently,
\[
\inf  \sup \mathbb{E}\|\hat{\beta}-\beta\|_{2}^{2} \geq \left(\inf\sup \mathbb{E}\|\hat{\beta}-\beta\|_{2}\right)^{2} \gtrsim \frac{\sigma^2\left(s_g(d +\log (\frac{m}{s_g}))\right)}{n}.
\]	

\end{proof}

\begin{proof}[Proof of Lemma~\ref{lem:lowerbound1}]
	
Notice that the cardinality of $A$ is $\left(\begin{array}{l}m \\ s_{g}\end{array}\right)$. Denote the hamming distance between any two points $x, y \in A$ by
$$
h(a, b)=|\left\{j: a_{j} \neq b_{j}\right\}|.
$$
Then, for a fixed point a $\in A$, 
$$|\left\{b \in A, h(a, b) \leq \frac{s_{g}}{2}\right\} = \left(\begin{array}{c}
	m \\
	\lfloor\frac{s_g}{2} \rfloor
	\end{array}\right)\cdot 2^{	\lfloor\frac{s_g}{2} \rfloor}|.$$

In fact, all elements $b \in A$ with $ h(a, b) \leq \frac{s_{g}}{2}$ can be obtained as follows. First, take any subset $J \subset [m]$ of cardinality  $\left\lfloor\frac{s_g}{2}\right\rfloor$, then set $a_j = b_j$ for $j \notin J$ and choose $b_j \in \{0,1\}$ for $j \in J$.

Now let $A_{s}$ be any subset of $A$ with cardinality at most
$
T=\frac{\left(\begin{array}{l}m \\ s_{g}\end{array}\right) - 2}{\left(\begin{array}{c}
	m \\
	{\left\lfloor\frac{s_{g}}{2}\right\rfloor}
	\end{array}\right) \cdot 2^{\frac{s_{g}}{2}}},
$
then we have 
$$|\left\{b \in A \mid \text{there exist}\hspace{0.2cm} a \in A_{s}\right. \text{with} \left.h(a, b) \leq \frac{s_{g}}{2}\right\} \leq (|A_s|) \cdot \left(\begin{array}{c}
	m \\
	{\left\lfloor\frac{s_{g}}{2}\right\rfloor}
	\end{array}\right) \cdot 2^{\frac{s_{g}}{2}}| <  |A|.$$
	
It implies that one can find an element $b \in  A \hspace{0.2cm} \text{with}\hspace{0.2cm}  h(a, b) > \frac{s_g}{2} \hspace{0.2cm} \text{for all}\hspace{0.2cm}  a \in A_s$.  Therefore one can construct a subset $A_s$  with $|A_s| \geq T$ and the property $h(a, b) > \frac{s_g}{2}$ for any two distinct elements $a, b \in A_s$.

On the other hand, $h(a, b)>\frac{s_{g}}{2}$ implies  $\|a-b\|>\sqrt{\frac{s_g}{2}}$. Therefore, there exist at least $T$ points in $A$ such that the distance between any two points is greater than $\sqrt{\frac{s_g}{2}}$.

Moreover, since
$
\frac{\left(\begin{array}{c}
	m \\
	s_g
	\end{array}\right)}{\left(\begin{array}{c}
	m \\
	{\left\lfloor\frac{s_g}{2}\right\rfloor}
	\end{array}\right)}=\frac{\left\lfloor\frac{s_g}{2}\right\rfloor!\left(m-\lfloor\frac{s_g}{2}\rfloor\right)!}{s_g !(m-s_g) !}=\frac{\left(m-s_{g}+1\right) \cdots m-\left\lfloor\frac{s_{g}}{2}\right\rfloor}{\left(\left\lfloor\frac{s_g}{2}\right\rfloor+1\right) \cdots s_{g}}=\prod_{j=1}^{\lceil\frac{s_g}{2}\rceil} \frac{m-s _g+j}{\left\lfloor\frac{s_g}{2}\right\rfloor+j}
$, we have $$
\left(\frac{\left. m-\lfloor\frac{s_g}{2}\right\rfloor}{2 s _g}\right)^{\left\lfloor \frac{s_g}{2}\right\rfloor} \left.\leqslant \frac{\left(\begin{array}{c}
	m \\
	s_g
	\end{array}\right)}{\left(\begin{array}{c}
	m \\
	\left\lfloor\frac{s_g}{2}\right\rfloor
	\end{array}\right) 2^{\frac{s_g}{2}}} \leqslant\left(\frac{m-s_g+1}{\lceil s_g \rceil}\right)^{\lceil \frac{s_g}{2}  \rceil}\right.,$$
	and therefore we can find $C_1, C_2,$ such that $
C_{1} s_g \log (\frac{m}{s_g}) \leqslant \log T \leqslant {C_2s_g \log (\frac{m}{s_g}})$, so that $$\log \Bigg( \frac{\left(\begin{array}{l}m \\ s_{g}\end{array}\right) - \hspace{0.2cm}2}{\left(\begin{array}{c}
	m \\
	{\left\lfloor\frac{s_{g}}{2}\right\rfloor}
	\end{array}\right) \cdot 2^{\frac{s_{g}}{2}}} \Bigg) \asymp s_g \log (\frac{m}{s_g}).$$
\end{proof}

	\begin{proof}[Proof of Lemma~\ref{lem:lowerbound2}]

	Given a group support $a \in A$,  define $k_{a} = \Bigg|\bigg\{i \mid i \in \left(\bigcup\limits_{\{g \mid a_g = 0\}} G_g \right)^{c}   \bigg\}\Bigg|$, and the set 
		$$\Omega^{(a)} =\bigg\{\omega \in \mathbb{R}^p \mid \omega_i =0 \text{ if } i \in \bigcup\limits_{\{g \mid a_g=0\}}G_g, \omega_i \in \{-1,1\} \text{ if }   i \in \Big(\bigcup\limits_{\{g \mid a_g = 0\}} G_g \Big)^{c} \bigg\}.$$
		
    Notice that $ \Omega^{(a)} \subseteq  \Omega(G,s_g)$, and  $|\Omega^{(a)}| = 2^{k_a}$. Also denote the hamming distance between $x, y \in \Omega^{(a)}$ by
	$$
	h(x, y)=|\left\{j: x_{j} \neq y_{j}\right\}|.
    $$
    Then for any fixed $x \in \Omega_G^{(a)}$, we have
	$$\left|\{y \in \Omega^{(a)}, h(x, y) \leq \frac{k_a}{10}\}\right| =\sum\limits_{j = 0}^{\lfloor\frac{k_a}{10} \rfloor}\left(\begin{array}{c}
   k_a  \\
   j 
	\end{array}\right)$$
	
	Let $\Omega_{s}^{(a)}$ be any subset of $\Omega^{(a)}$ with cardinality at most 
	$
	N^{(a)}=\frac{2^{k_a} - 2}{\sum\limits_{j = 0}^{\lfloor\frac{k_a}{10} \rfloor}\left(\begin{array}{c}
   k_a  \\
   j 
	\end{array}\right)}
	$.
 Then,
\[
\left|\{y \in \Omega^{(a)} \mid \exists x \in \Omega_{s}^{(a)} \text{ with } h(x, y) \leq  \frac{k_a}{10}  \}\right| < |\Omega^{(a)}|.
\]

	On the other hand, $h(x,y) > \frac{k_a}{10}$ implies $\|x-y\| \geq \sqrt{\frac{2k_a}{5}}$. Thus, there are at least $N^{(a)}$ points in $\Omega^{(a)}$ with pairwise distances greater than $\sqrt{\frac{2k_a}{5}}$.

 From Chapter 9 in \citet{cma},
\[
\sum\limits_{j \leq \lfloor\frac{k_a}{10} \rfloor}\left(\begin{array}{c}
   k_a  \\
   j
\end{array}\right) < \frac{9}{8}\left(\begin{array}{c}
   k_a  \\
  \lfloor\frac{k_a}{10} \rfloor
\end{array}\right) \leq \frac{9}{8}(10e)^{\frac{k_a}{10}} \leq \frac{9}{8}2^{\frac{k_a}{2}}.
\]
Consequently, we have $N^{(a)} >  \frac{8}{9}2^{\frac{k_a}{2}} \gtrsim (\sqrt{2})^{k_a}.$

The value of $k_a$ depends on the predefined groups and group support $a$ and spans a range from $0$ to $s_gd$. Lemma~\ref{lem:lowerbound2} seeks a lower bound for all conceivable overlapping patterns, necessitating an analysis of the maximum value of $k_a$.

Furthermore, according to Lemma~\ref{lem:lowerbound1}, we can identify at least $T$ points in $A$ where the distance between any two points exceeds $\sqrt{\frac{s_g}{2}}$. For $\{a_1,\cdots,a_T\}$ group supports, if there is a group structure such that we could find at least $\frac{8}{9}(\sqrt{2})^{s_gd}$ on each group support, and the distance between every pair of these points is greater than $\sqrt{\frac{2s_gd}{5}}$, then Lemma~\ref{lem:lowerbound2} is proved.

Considering $m$ non-overlapping groups, $k_a = s_gd$ for each group support $a$. In addition,  given any two group support $a, b $ with $\|a - b\|>\sqrt{\frac{s_{g}}{2}},\|x-y\|>\sqrt{\frac{ds_g}{2}}>\sqrt{\frac{2ds_g}{5}}$ for any $x \in \Omega^{(a)}$ and $y \in \Omega^{(b)}$. Thus, considering all possible overlapping patterns, we can find at least $  \frac{\left(\begin{array}{l}m \\ s_{g}\end{array}\right) - 2}{\left(\begin{array}{c}
	m \\
	{\left\lfloor\frac{s_{g}}{2}\right\rfloor}
	\end{array}\right) \cdot 2^{\frac{s_{g}}{2}}} \cdot \frac{8}{9}(\sqrt{2})^{ds_g}$ point in  $\Omega(G,s_g)$, such that the distance between every pair of points is greater than  $\sqrt{\frac{2ds_g}{5}}$.

\end{proof}

\newpage

\subsection{Proof of Theorem~\ref{pattern}}
This proof consists of  parts: Parts \rom{1}-\rom{4} dedicated to Theorem~\ref{pattern}.\ref{the6part1}, and Part \rom{5} is for Theorem~\ref{pattern}.\ref{the6part2}. To be more specific, Part \rom{1} provides some additional concepts, Part \rom{2} introduces the reduced problem, Part \rom{3} shows the successful selection of the correct pattern under favorable conditions, and Part \rom{4} establishes that certain conditions are satisfied with high probability.
\subsubsection{Part I}

Recall that $\mathbf{S} = \textit{supp}(\beta^*)$.  With \(\mathbf{S}\), we define the norm $\phi_{\mathbf{S}}$ for any $\beta \in \mathbb{R}^{p}$ as 
$$ \phi_{\mathbf{S}}(\beta_{\mathbf{S}}) = \sum_{g \in \mathsf{G}_{\mathbf{S}}} w_g \|\beta_{\mathbf{S} \cap G_g  }\|_2,$$
along with its dual norm $(\phi_{\mathbf{S}})^*[u] = \sup_{\phi_{\mathbf{S}}(\beta_{\mathbf{S}}) \leq 1} \beta_{\mathbf{S}}^\top u$. Similarly, for  $\mathbf{S}^c = [p]\setminus \mathbf{S}$, we define the norm $\phi_{\mathbf{S}}^c$ for any $\beta \in \mathbb{R}^{p}$ as
$$\phi_{\mathbf{S}}^c(\beta_{\mathbf{S}^c}) = \sum_{g \in [m] \setminus \mathsf{G}_{\mathbf{S}}} w_g \|\beta_{ \mathbf{S}^c \cap G_g}\|_2,$$ 
accompanied by its corresponding dual norm $(\phi_{\mathbf{S}}^c)^*[u] = \sup_{\phi_{\mathbf{S}}^c(\beta_{\mathbf{S}^c}) \leq 1} \beta_{\mathbf{S}^c}^\top u $.

We also introduce equivalence parameters \(a_{\mathbf{S}}, A_{\mathbf{S}}, a_{\mathbf{S}^c}, A_{\mathbf{S}^c}\) as follows:
\begin{align}
\label{eq:equalnormdef}
    \forall \beta \in \mathbb{R}^{p},\, a_{\mathbf{S}}\|\beta_{\mathbf{S}}\|_1 \leqslant \phi_{\mathbf{S}}(\beta_{\mathbf{S}}) \leqslant A_{\mathbf{S}}\|\beta_{\mathbf{S}}\|_1, \\
    \forall \beta \in \mathbb{R}^{p},\, a_{\mathbf{S}^c}\|\beta_{\mathbf{S}^c}\|_1 \leqslant \phi_{\mathbf{S}}^c(\beta_{\mathbf{S}^c}) \leqslant A_{\mathbf{S}^c}\|\beta_{\mathbf{S}^c}\|_1.
\end{align}

We now study the equivalence parameters from two aspects. First, since
\begin{equation*}
   \sup\limits_{ a_{\mathbf{S}}\|\beta_{\mathbf{S}}\|_1 \leqslant 1} \beta_{\mathbf{S}}^\top u \geqslant  \sup\limits_{ \phi_{\mathbf{S}}(\beta_{\mathbf{S}})  \leqslant 1} \beta_{\mathbf{S}}^\top u  \geqslant \sup\limits_{ A_{\mathbf{S}}\|\beta_{\mathbf{S}}\|_1 \leqslant 1} \beta_{\mathbf{S}}^\top u,
\end{equation*}
by the definition of dual norm, we have 

\begin{equation}
\label{eq:orderre1}
    \forall u \in \mathbb{R}^{|\mathbf{S}|}, A_{\mathbf{S}}^{-1}\|u\|_{\infty} \leqslant (\phi_{\mathbf{S}})^*[u] \leqslant a_{\mathbf{S}}^{-1}\|u\|_{\infty}.
\end{equation}
Similarly, by order-reversing,
\begin{equation}
\label{eq:orderre}
    \forall u \in \mathbb{R}^{|\mathbf{S}^c|}, A_{\mathbf{S}^c}^{-1}\|u\|_{\infty} \leqslant (\phi_{\mathbf{S}}^c)^*[u] \leqslant a_{\mathbf{S}^c}^{-1}\|u\|_{\infty}.
\end{equation}

Second, by the Cauchy-Schwarz inequality, for any $\beta \in \mathbb{R}^{p}$ and $g \in \mathsf{G}_{\mathbf{S}}$,
$$\frac{w_g}{\sqrt{d_g}}\|\beta_{\mathbf{S} \cap G_g  }\|_1 \leqslant w_g \|\beta_{\mathbf{S} \cap G_g  }\|_2 \leqslant \max\limits_{g \in \mathsf{G}_{\mathbf{S}}}w_g\|\beta_{\mathbf{S} \cap G_g  }\|_1. $$
Consequently, we have 
$$ \min\limits_{g \in \mathsf{G}_{\mathbf{S}}}\frac{w_g}{\sqrt{d_g}}\|\beta_{\mathbf{S}}\|_1 \leqslant  \phi_{\mathbf{S}}(\beta_{\mathbf{S}}) \leqslant  h_{\max}(\mathbf{G_{S}})\max\limits_{g \in \mathsf{G}_{\mathbf{S}}}w_g\|\beta_{\mathbf{S}}\|_1,$$
Therefore, we can set $a_{\mathbf{S}} = \min\limits_{g \in \mathsf{G}_{\mathbf{S}}}\frac{w_g}{\sqrt{d_g}}$ and $A_{\mathbf{S}} = h_{\max}(\mathbf{G_{S}})\max\limits_{g \in \mathsf{G}_{\mathbf{S}}}w_g$.
With an trivial extension, we can set $ a_{\mathbf{S}^c} = \min\limits_{g \in \mathsf{G}_{\mathbf{S^c}}}w_g/\sqrt{d_g}$.

\subsubsection{Part II}

\textbf{From the full problem to the reduced problem}

Recall that the group lasso estimator in \eqref{eq:OGL-est} is defined as

\begin{equation}
\label{eq:op1}
\hat{\beta}^G = \argmin_{\beta \in \bR^p}~~\frac{1}{2n}\norm{Y-X\beta}_2^2 + \lambda_n  \phi^G(\beta).
\end{equation}

Now we write $ \phi^G(\beta) = \phi(\beta) $ and $ L(\beta) = \frac{1}{2n}\norm{Y-X\beta}_2^2 $ for ease of notation. Following \citet{svsw,wainwright2009sharp}, we consider the following restricted problem

\begin{equation}
\label{eq:resOGL-est}
\begin{aligned}
  \hat{\beta}^R & = \argmin_{\beta \in \bR^p, \beta_{\mathbf{S}^c}=0}  L(\beta) + \lambda_n \phi (\beta) = \argmin_{\beta \in \bR^p, \beta_{\mathbf{S}^c}=0}  L(\beta) + \lambda_n \sum_{g\in \mathsf{G}_{\mathbf{S}}} w_g \left\| \beta_{\mathbf{S} \cap G_g } \right\|_2  \\
  &:= \argmin_{\beta \in \bR^p, \beta_{\mathbf{S}^c}=0}  L(\beta) + \lambda_n \phi_{\mathbf{S}}(\beta_{\mathbf{S}}) .
\end{aligned}
\end{equation}

Let $L_{\mathbf{S}}(\beta_{\mathbf{S}})  = \frac{1}{2n}
\norm{Y -X_{\mathbf{S}}\beta_{\mathbf{S}}}_2^2$. Due to the restriction of $\hat{\beta}^R$, we can obtain $\hat{\beta}^R$ by first solving the following reduced problem 
\begin{equation}
\label{eq:redOGL-est}
\begin{aligned}
    \hat{\beta}_{\mathbf{S}} &=  \argmin_{\beta_{\mathbf{S}} \in \bR^{|\mathbf{S}|}}~~\frac{1}{2n}
\norm{Y -X_{\mathbf{S}}\beta_{\mathbf{S}}}_2^2 + \lambda_n \sum_{g\in \mathsf{G}_{\mathbf{S}}} w_g \left\| \beta_{\mathbf{S} \cap G_g} \right\|_2 \\& =  \argmin_{\beta_{\mathbf{S}} \in \bR^{|\mathbf{S}|}}  L_{\mathbf{S}}(\beta_{\mathbf{S}}) + \lambda_n \phi_{\mathbf{S}}(\beta_{\mathbf{S}}) 
\end{aligned}
\end{equation}
and then padding $\hat{\beta}_{\mathbf{S}}$ with zeros on $\mathbf{S}^c$. In addition,  

\begin{equation*}
\begin{aligned}
     L_{\mathbf{S}}(\hat{\beta}_{\mathbf{S}}) &= \frac{1}{2n}\norm{Y -X_{\mathbf{S}}\hat{\beta}_{\mathbf{S}}}_2^2
     \\& =  \frac{1}{2n}\left(Y^\top Y - 2Y^\top X_{\mathbf{S}}\hat{\beta}_{\mathbf{S}} + (X_{\mathbf{S}}\hat{\beta}_{\mathbf{S}})^\top X_{\mathbf{S}}\hat{\beta}_{\mathbf{S}}\right)\\
     & =  \frac{1}{2n}\left(Y^\top Y - 2(X\beta^* + \epsilon)^\top X_{\mathbf{S}}\hat{\beta}_{\mathbf{S}} + (X_{\mathbf{S}}\hat{\beta}_{\mathbf{S}})^\top X_{\mathbf{S}}\hat{\beta}_{\mathbf{S}}\right)\\
     & =  \frac{1}{2n}\left(Y^\top Y - 2(X_{\mathbf{S}}\beta^*_{\mathbf{S}})^\top X_{\mathbf{S}}\hat{\beta}_{\mathbf{S}} - 2 \epsilon^\top X_{\mathbf{S}}\hat{\beta}_{\mathbf{S}} + (X_{\mathbf{S}}\hat{\beta}_{\mathbf{S}})^\top X_{\mathbf{S}}\hat{\beta}_{\mathbf{S}}\right),
\end{aligned}
\end{equation*}

and consequently,

\begin{equation}
\label{eq:derivation}
\begin{aligned}
 \nabla L_{\mathbf{S}}(\hat{\beta}_{\mathbf{S}})&= \frac{1}{n}X_{\mathbf{S}}^\top X_{\mathbf{S}} \hat{\beta}_{\mathbf{S}} - \frac{1}{n}X_{\mathbf{S}}^\top X_{\mathbf{S}} \beta^*_{\mathbf{S}} - \frac{1}{n}\epsilon^\top X_{\mathbf{S}}\\
 &:= Q_{\mathbf{S}\mathbf{S}}(\hat{\beta}_{\mathbf{S}}-\beta^*_{\mathbf{S}})-q_{\mathbf{S}},
\end{aligned}
\end{equation}
where  $Q=\frac{1}{n} X^{\top}X$, $q  = \frac{1}{n} \sum\limits_{i=1}^n \epsilon_i x_i$.

\subsubsection{Part III}
Part \rom{3} mostly follows the proof in Theorem 7 of \cite{svsw}. Here we aim to show that $\textit{supp}(\hat{\beta}^G) = \mathbf{S}$ under certain conditions.

To begin with, Given \(\beta \in \mathbb{R}^p\), we define \(J^G(\beta)\)  as:
\[ J^G(\beta) = [p] \setminus \Big\{\bigcup_{G_g \cap \textit{supp}(\beta) = \emptyset} G_g\Big\}. \]
\(J^G(\beta)\) is called the adapted hull of the support of $\beta$ in \citet{svsw}. For simplicity, we write $J^G(\beta) = J(\beta)$. Notice that by assumption we have 
$$J(\beta^*) =  [p] \setminus \Big\{\bigcup_{G_g \cap \textit{supp}(\beta^*) = \emptyset} G_g\Big\} = \mathbf{S}.$$

Now we consider the reduced problem \eqref{eq:redOGL-est}, and we want to show that for all $g \in \mathsf{G}_{\mathbf{S}}$, $\left\|\hat{\beta}_{\mathbf{S}\cap G_g} \right\|_{\infty}>0$. That is, no active group is missing.

\begin{lemma}(Lemma 14 of \citet{svsw})
\label{lem:jenatton14}

For the loss $L(\beta)$ and norm $\phi$ in \eqref{eq:op1}, $\hat{\beta} \in \bR^p$ is a solution of 
\begin{equation}
\label{eq:jenatton14}
    \min\limits _{\beta \in \mathbb{R}^p} L(\beta)+\lambda_n\phi(\beta)
\end{equation}
if and only if
\begin{equation}
\label{ref1}
    \begin{aligned}
    \left\{\begin{array}{l}
        \nabla L(\hat{\beta})_{J(\hat{\beta})}+\lambda_n r(\hat{\beta})_{J(\hat{\beta})} = \mathbf{0} \\
        (\phi_{J(\hat{\beta})}^{c})^*\left[\nabla L(\hat{\beta})_{J(\hat{\beta})^c} \right] \leqslant \lambda_n.
        \end{array}\right.
\end{aligned}
\end{equation}

In addition, the solution $\hat{\beta}$ satisfies

\begin{equation}
\label{eq:jenatton141}
    \phi^*[\nabla L(\hat{\beta})] \leqslant \lambda_n.
\end{equation}

\end{lemma}

As $\hat{\beta}_{\mathbf{S}}$ is the solution of  \eqref{eq:redOGL-est},  Equation~\eqref{eq:jenatton141} in Lemma~\ref{lem:jenatton14} implies that 

\begin{equation}
\label{ref2}
     (\phi_{\mathbf{S}})^*\left[\nabla L_{\mathbf{S}}(\hat{\beta}_{\mathbf{S}})\right]\stackrel{\text{\eqref{eq:derivation}}}{=} (\phi_{\mathbf{S}})^*\left[Q_{\mathbf{S}\mathbf{S}}\left(\hat{\beta}_{\mathbf{S}}-\beta_{\mathbf{S}}\right)-q_{\mathbf{S}}\right] \leqslant \lambda_n.
\end{equation}

By the property of the equivalent parameters,  we have
\begin{equation}
\label{eq:ineq1}
A_{\mathbf{S}}^{-1}\left\|Q_{\mathbf{S}\mathbf{S}}\left(\hat{\beta}_{\mathbf{S}}-\beta_{\mathbf{S}}\right)-q_{\mathbf{S}}\right\|_{\infty} \stackrel{\text{\eqref{eq:orderre1}}}{\leqslant}(\phi_{\mathbf{S}})^*\left[Q_{\mathbf{S}\mathbf{S}}\left(\hat{\beta}_{\mathbf{S}}-\beta_{\mathbf{S}}\right)-q_{\mathbf{S}}\right]  \stackrel{\text{\eqref{ref2}}}{\leqslant} \lambda_n.
\end{equation}

If 
\begin{equation}
\label{eq:fc1}
    \lambda_n \leqslant \frac{\gamma_{\min}\left(Q_{\mathbf{S}\mathbf{S}}\right) \beta^*_{\min}}{3|\mathbf{S}|^{\frac{1}{2}} A_{\mathbf{S}}},
\end{equation}
and 
\begin{equation}
\label{eq:fc2}
   \left\|q_{\mathbf{S}}\right\|_{\infty} \leqslant \frac{\gamma_{\min}\left(Q_{\mathbf{S}\mathbf{S}}\right) \beta^*_{\min}}{3|\mathbf{S}|^{\frac{1}{2}}},
\end{equation}
then we have 
\begin{equation}
    \label{eq:active}
    \begin{aligned}
       \left\|\hat{\beta}_{\mathbf{S}}-\beta_{\mathbf{S}}^*\right\|_{\infty} &= \left\|Q_{\mathbf{S}\mathbf{S}}^{-1} Q_{\mathbf{S}\mathbf{S}} \left(\hat{\beta}_{\mathbf{S}}-\beta_{\mathbf{S}}^*\right)\right\|_{\infty} \\
    &\leqslant \left\|Q_{\mathbf{S}\mathbf{S}}^{-1}\right\|_{\infty,\infty}\left\| Q_{\mathbf{S}\mathbf{S}}\left(\hat{\beta}_{\mathbf{S}}-\beta_{\mathbf{S}}^*\right)\right\|_{\infty} \\
    &\leqslant |\mathbf{S}|^{\frac{1}{2}} \gamma_{\max }\left(Q_{\mathbf{S}\mathbf{S}}^{-1}\right)\left\|Q_{\mathbf{S}\mathbf{S}}\left(\hat{\beta}_{\mathbf{S}}-\beta_{\mathbf{S}}^*\right)\right\|_{\infty} \\
    &\leqslant |\mathbf{S}|^{\frac{1}{2}} \gamma^{-1}_{\min} \left(Q_{\mathbf{S}\mathbf{S}}\right)\left(\left\|Q_{\mathbf{S}\mathbf{S}}\left(\hat{\beta}_{\mathbf{S}}-\beta_{\mathbf{S}}\right)-q_{\mathbf{S}}\right\|_{\infty}+\left\|q_{\mathbf{S}}\right\|_{\infty}\right) \\
    &\stackrel{\text{\eqref{eq:ineq1}}}{\leqslant } |\mathbf{S}|^{\frac{1}{2}} \gamma^{-1}_{\min} \left(Q_{\mathbf{S}\mathbf{S}}\right)\left(\lambda_n A_{\mathbf{S}}+\left\|q_{\mathbf{S}}\right\|_{\infty}\right) \\
    &\leqslant  |\mathbf{S}|^{\frac{1}{2}} \gamma^{-1}_{\min} \left(Q_{\mathbf{S}\mathbf{S}}\right) \lambda_n A_{\mathbf{S}}+ |\mathbf{S}|^{\frac{1}{2}} \gamma^{-1}_{\min} \left(Q_{\mathbf{S}\mathbf{S}}\right)\left\|q_{\mathbf{S}}\right\|_{\infty}
    \\
    &\leqslant \frac{2}{3} \beta^*_{\min}.  
    \end{aligned}
\end{equation}

If there exist a group $g \in \mathsf{G}_{\mathbf{S}}$ such that $\left\|\hat{\beta}_{\mathbf{S} \cap G g}\right\|_{\infty} < \frac{\beta^*_{\min}}{3} $, then  $$\left\|\hat{\beta}_{\mathbf{S}}-\beta^*_{\mathbf{S}}\right\|_{\infty}>\beta^*_{\min}-\frac{\beta^*_{\min}}{3}=\frac{2\beta^*_{\min}}{3}.$$ 

Thus, Equation \eqref{eq:active} implies that for all $g \in \mathsf{G}_{\mathbf{S}}$, 
\begin{equation}
    \label{eq:grp}
    \left\|\hat{\beta}_{\mathbf{S} \cap G g}\right\|_{\infty} > \frac{\beta^*_{\min}}{3} > 0.
\end{equation}

Secondly, we want to show that $\hat{\beta}^R$ solves problem \eqref{eq:op1}. As $\hat{\beta}^R$
is obtained by padding  $\hat{\beta}_{\mathbf{S}}$  with zeros on $\mathbf{S}^c$,

\begin{equation*}
    \begin{aligned}
        J(\hat{\beta}^R) &= [p]\setminus\bigg\{\bigcup\limits_{G_g \cap \textit{supp}(\hat{\beta}^R)=\emptyset} G_g \bigg\}= [p]\setminus\bigg\{\bigcup\limits_{G_g \cap \textit{supp}(\hat{\beta}_{\mathbf{S}})=\emptyset} G_g \bigg\} \\
        & \stackrel{\text{\eqref{eq:active}}}{=} [p]\setminus\bigg\{\bigcup\limits_{G_g \cap \mathbf{S} = \emptyset} G_g \bigg\} =  \mathbf{S}.
    \end{aligned}
\end{equation*}
From Lemma~\ref{lem:jenatton14} we know  that $\hat{\beta}^R$ is the optimal for problem \eqref{eq:op1} if and only if
\begin{equation}
\label{eq:cond1}
\nabla L(\hat{\beta}^R)_{\mathbf{S}}+\lambda_n r(\hat{\beta}^R)_{\mathbf{S}} = \mathbf{0},
\end{equation}
and
\begin{equation}
\label{eq:cond2}
 (\phi_{\mathbf{S}}^{c})^*\left[\nabla L(\hat{\beta}^R)_{\mathbf{S}^c} \right] \leqslant \lambda_n.
\end{equation}

We now verify the condition in Equation \eqref{eq:cond1}. Since
\begin{equation*}
\begin{aligned}
     L(\hat{\beta}^R) &= \frac{1}{2n}\norm{Y -X\hat{\beta}^R}_2^2
   \\  & =  \frac{1}{2n}\left(Y^\top Y - 2(X\beta^*)^\top X\hat{\beta}^R - 2 \epsilon^\top X\hat{\beta}^R + (X\hat{\beta}^R)^\top X\hat{\beta}^R\right),
\end{aligned}
\end{equation*}
we have

\begin{equation}
\label{eq:eq1}
    \begin{aligned}
             \nabla L(\hat{\beta}^R)_{\mathbf{S}} &= 
        \Big[\frac{1}{n}X^\top X\left(\hat{\beta}^R-\beta^*\right) - \frac{1}{n}\epsilon^\top X\Big]_{\mathbf{S}}
        \\& = \left[Q\left(\hat{\beta}^R-\beta^*\right)\right]_{\mathbf{S}} - q_{\mathbf{S}}  = Q_{\mathbf{S}\mathbf{S}}\left(\hat{\beta}^R-\beta^*\right)_\mathbf{S} - q_{\mathbf{S}}
        \\& = Q_{\mathbf{S}\mathbf{S}}\left(\hat{\beta}_\mathbf{S}^R-\beta_\mathbf{S}^*\right) - q_{\mathbf{S}} = Q_{\mathbf{S}\mathbf{S}}\left(\hat{\beta}_\mathbf{S}-\beta_\mathbf{S}^*\right) - q_{\mathbf{S}} \\
        &=  \nabla L_{\mathbf{S}}(\hat{\beta}_{\mathbf{S}}) .
    \end{aligned}
\end{equation}

On the other hand, as $\hat{\beta}^R$
is obtained by padding  $\hat{\beta}_{\mathbf{S}}$  with zeros on $\mathbf{S}^c$, we have

\begin{equation*}
    \begin{aligned}
        \lambda_n r(\hat{\beta}^R)_{\mathbf{S}} = \lambda_n r_{\mathbf{S}}(\hat{\beta}_{\mathbf{S}}). 
    \end{aligned}
\end{equation*}

Because $\hat{\beta}_{\mathbf{S}}$ is the optimal for problem \eqref{eq:redOGL-est}, Equation \eqref{ref1} in Lemma~\ref{lem:jenatton14} implies that 
\begin{equation}
\label{eq:condi11}
    \nabla L_{\mathbf{S}}(\hat{\beta}_{\mathbf{S}})+\lambda_n r_{\mathbf{S}}(\hat{\beta}_{\mathbf{S}})  \stackrel{\text{\eqref{eq:derivation}}}{=}  Q_{\mathbf{S}\mathbf{S}}(\hat{\beta}_{\mathbf{S}}-\beta^*_{\mathbf{S}})-q_{\mathbf{S}} + \lambda_n r_{\mathbf{S}}(\hat{\beta}_{\mathbf{S}}) = \mathbf{0}.
\end{equation}

Thus, Equation \eqref{eq:cond1} holds as
\begin{equation}
\label{eq:inv}
\nabla L(\hat{\beta}^R)_{\mathbf{S}}+\lambda_n r_{\mathbf{S}}(\hat{\beta}^R) =  \nabla L_{\mathbf{S}}(\hat{\beta}_{\mathbf{S}})+\lambda_n r_{\mathbf{S}}(\hat{\beta}_{\mathbf{S}})   \stackrel{\text{\eqref{eq:condi11}}}{=} 
 \mathbf{0}.
\end{equation}

Now we continue to show Equation \eqref{eq:cond2}. Notice that 
\begin{equation}
    \label{eq:invv}
    \left(\hat{\beta}^R-\beta^*\right)_\mathbf{S}\stackrel{\text{\eqref{eq:eq1}}}{=}  \left(\hat{\beta}_\mathbf{S}-\beta_\mathbf{S}^*\right) \stackrel{\text{\eqref{eq:condi11}}}{=} Q_{\mathbf{S}\mathbf{S}}^{-1}( q_{\mathbf{S}} - \lambda_n r_{\mathbf{S}}(\hat{\beta}_{\mathbf{S}}) ).
\end{equation}

Let $q_{\mathbf{S}^c \mid \mathbf{S} } = q_{\mathbf{S}^c} - Q_{\mathbf{S}^c\mathbf{S}}Q_{\mathbf{S}\mathbf{S}}^{-1}q_{\mathbf{S}} $, we have

\begin{equation}
\label{eq:deriveR}
    \begin{aligned}
        \nabla L(\hat{\beta}^R)_{\mathbf{S}^c} &\stackrel{\text{\eqref{eq:eq1}}}{=} \left(Q(\hat{\beta}^R-\beta^*)\right)_{\mathbf{S}^c} - q_{\mathbf{S}^c}  = Q_{\mathbf{S}^c\mathbf{S}}(\hat{\beta}^R-\beta^*)_\mathbf{S} - q_{\mathbf{S}^c}\\
        & \stackrel{\text{\eqref{eq:invv}}}{=} Q_{\mathbf{S}^c\mathbf{S}}Q_{\mathbf{S}\mathbf{S}}^{-1}\left( q_{\mathbf{S}} - \lambda_n r_{\mathbf{S}}(\hat{\beta}_{\mathbf{S}}) \right) - q_{\mathbf{S}^c}\\
        & = -Q_{\mathbf{S}^c\mathbf{S}}Q_{\mathbf{S}\mathbf{S}}^{-1} \lambda_n r_{\mathbf{S}}(\hat{\beta}_{\mathbf{S}})  + Q_{\mathbf{S}^c\mathbf{S}}Q_{\mathbf{S}\mathbf{S}}^{-1}q_{\mathbf{S}} - q_{\mathbf{S}^c}\\
        &= -\lambda_n Q_{\mathbf{S}^c\mathbf{S}}Q_{\mathbf{S}\mathbf{S}}^{-1}\left(r_{\mathbf{S}}(\hat{\beta}_{\mathbf{S}}) -r_{\mathbf{S}}(\beta^*_{\mathbf{S}}) \right) - \lambda_n Q_{\mathbf{S}^c\mathbf{S}}Q_{\mathbf{S}\mathbf{S}}^{-1}r_{\mathbf{S}}(\beta^*_{\mathbf{S}})  - q_{\mathbf{S}^c \mid \mathbf{S} }.
    \end{aligned}.
\end{equation}

The previous expression leads us to study the difference of $r_{\mathbf{S}}(\hat{\beta}_{\mathbf{S}}) -r_{\mathbf{S}}(\beta^*_{\mathbf{S}})$. We now introduce the following lemma.

\begin{lemma}(Lemma 12 of \citet{svsw})
\label{lem:jenatton12}

 For any $J \subset [p]$, let $u_{J}$ and  $v_{J}$ be two nonzero vectors in $\mathbb{R}^{|J|}$, and define the mapping $ r_J : \mathbb{R}^{|J|} \mapsto \mathbb{R}^{|J|}$ such that 
\begin{align*}
    r_J\left(\beta_{J}\right)_j = \beta_j \mathop{\Sigma}\limits_{g \in \mathsf{G}_J , G_g \cap j \neq \phi} \dfrac{\omega_g}{\left\|\beta_{J \cap G_g }\right\|_2}.
\end{align*}
Then there exists $\xi_J=t_0 u_J + (1 - t_0) v_J$ for some $t_0 \in (0,1)$, such that 
\begin{align*}
    \left\| r_J\left(u_J\right)-r_J(v_J) \right\|_1 \leqslant \left\|u_J-v_J\right\|_{\infty} \left(\sum_{j \in J} \sum_{g \in \mathsf{G}_J} \frac{w_g \mathbbm{1}_{\{j \in G_g\}}}{\left\|\xi_{J \cap G_g}\right\|_2} + \sum_{j \in J}\left(\sum_{k\in J}\sum_{g \in \mathsf{G}_J}\frac{|\xi_j| |\xi_k| w_g^4 \mathbbm{1}_{\{j,k \in G_g\}}}{\left\|\xi_{J \cap G_g}\right\|_2^3}\right)\right).
\end{align*}

\end{lemma}

Lemma~\ref{lem:jenatton12} implies that
\begin{equation}
\label{eq:lem4i}
    \begin{aligned}
    \left\|r_{\mathbf{S}}(\hat{\beta}_{\mathbf{S}}) -r_{\mathbf{S}}(\beta^*_{\mathbf{S}}) \right\|_1 \leqslant \left\|\hat{\beta}_{\mathbf{S}}-\beta^*_{\mathbf{S}}\right\|_{\infty} \left(\sum_{j\in \mathbf{S}} \sum_{g \in \mathsf{G}_\mathbf{S}} \frac{w_g \mathbbm{1}_{\{j \in G_g\}}}{\left\|\Tilde{\beta}_{\mathbf{S} \cap G_g}\right\|_2} + \sum_{j\in \mathbf{S}} \sum_{k \in \mathbf{S}} \sum_{g \in \mathsf{G}_\mathbf{S}} \frac{(w_g)^4\mathbbm{1}_{\{j,k \in G_g\}}|\Tilde{\beta}_j||\Tilde{\beta}_k|}{w_g^3\left\|\Tilde{\beta}_{\mathbf{S} \cap G_g}\right\|_2^3}\right),
\end{aligned}
\end{equation}
where $\Tilde{\beta}=t_0 \hat{\beta}_{\mathbf{S}} + \left(1-t_0\right) \beta^*_{\mathbf{S}}$.

To find an upper bound of the right-hand side. Recall that Equation \eqref{eq:active} implies that 
$ \left\|\hat{\beta}_{\mathbf{S}}-\beta_{\mathbf{S}}^*\right\|_{\infty}\leqslant \frac{2}{3} \beta^*_{\min}$, so we have

\begin{equation*}
    \begin{aligned}
    \left\|\Tilde{\beta}_{\mathbf{S} \cap G_g}\right\|_2 &\geqslant \sqrt{|\mathbf{S} \cap G_g|}\min\{|\Tilde{\beta}|_j \mid \Tilde{\beta}_j \neq 0\} \\& \geqslant  \sqrt{|\mathbf{S} \cap G_g|}(\beta_{\min}^* - t_0\left\|\hat{\beta}_{\mathbf{S}}-\beta_{\mathbf{S}}^*\right\|_{\infty})\\&  \geqslant \sqrt{|\mathbf{S} \cap G_g|}(\beta_{\min}^* - \left\|\hat{\beta}_{\mathbf{S}}-\beta_{\mathbf{S}}^*\right\|_{\infty}) \\& \geqslant\sqrt{|\mathbf{S} \cap G_g|}\frac{\beta^*_{\min}}{3}.
    \end{aligned}
\end{equation*}

Consequently, the first term could be upper bounded by
\begin{align*}
  \sum_{j\in \mathbf{S}} \sum_{g \in \mathsf{G}_\mathbf{S}} \frac{w_g\mathbbm{1}_{\{j \in G_g\}}}{\left\|\Tilde{\beta}_{\mathbf{S} \cap G_g}\right\|_2} =  \sum_{g \in \mathsf{G}_\mathbf{S}} \frac{w_g|\mathbf{S} \cap G_g|}{\left\|\Tilde{\beta}_{\mathbf{S} \cap G_g}\right\|_2} \leqslant \frac{3}{\beta^*_{\min}} \sum_{g \in \mathsf{G}_\mathbf{S}}w_g \sqrt{|\mathbf{S} \cap G_g|}
\end{align*}

On the other hand, the Cauchy-Schwarz inequality gives
$$\left\|\Tilde{\beta}_{\mathbf{S} \cap G_g}\right\|_1^2 \leqslant |\mathbf{S} \cap G_g|\left\|\Tilde{\beta}_{\mathbf{S} \cap G_g}\right\|_2^2.$$

Thus,  the second term could also be upper bounded by
\begin{align*}
 \sum_{j\in \mathbf{S}} \sum_{k \in \mathbf{S}} \sum_{g \in \mathsf{G}_\mathbf{S}} \frac{(w_g)^4\mathbbm{1}_{\{j,k \in G_g\}}|\Tilde{\beta}_j||\Tilde{\beta}_k|}{w_g^3\left\|\Tilde{\beta}_{\mathbf{S} \cap G_g}\right\|_2^3} &= \sum_{g \in \mathsf{G}_\mathbf{S}} \frac{w_g^4\left\|\Tilde{\beta}_{\mathbf{S} \cap G_g}\right\|_1^2}{w_g^3\left\|\Tilde{\beta}_{\mathbf{S} \cap G_g}\right\|_2^3} \\&\leqslant 
 \sum_{g \in \mathsf{G}_\mathbf{S}} \frac{w_g|\mathbf{S} \cap G_g|}{\left\|\Tilde{\beta}_{\mathbf{S} \cap G_g}\right\|_2} \\&\leqslant \frac{3}{\beta^*_{\min}} \sum_{g \in \mathsf{G}_\mathbf{S}}w_g \sqrt{|\mathbf{S} \cap G_g|}.
\end{align*}

 Let $c_2 = \frac{6}{\beta^*_{\min}} \sum_{g \in \mathsf{G}_\mathbf{S}}w_g \sqrt{|\mathbf{S} \cap G_g|}$, then Equation \eqref{eq:lem4i} implies $$ \left\|r_{\mathbf{S}}(\hat{\beta}_{\mathbf{S}}) -r_{\mathbf{S}}(\beta^*_{\mathbf{S}}) \right\|_1 \leqslant c_2\left\|\hat{\beta}_{\mathbf{S}}-\beta^*_{\mathbf{S}}\right\|_{\infty}.$$ 

 If 
 \begin{equation}
     \label{eq:fc3}
     \|Q_{\mathbf{S}^c\mathbf{S}}Q_{\mathbf{S}\mathbf{S}}^{-\frac{1}{2}}\|_{2,\infty} \leqslant 3,
 \end{equation}
then we have

\begin{equation*}
    \begin{aligned}
\left\|Q_{\mathbf{S}^c\mathbf{S}}Q_{\mathbf{S}\mathbf{S}}^{-1} \left( r_{\mathbf{S}}(\hat{\beta}_{\mathbf{S}}) -r_{\mathbf{S}}(\beta^*_{\mathbf{S}})  \right)\right\|_{\infty}
         ={}& \left\| Q_{\mathbf{S}^c\mathbf{S}}Q_{\mathbf{S}\mathbf{S}}^{-\frac{1}{2}}Q_{\mathbf{S}\mathbf{S}}^{-\frac{1}{2}} \left( r_{\mathbf{S}}(\hat{\beta}_{\mathbf{S}}) -r_{\mathbf{S}}(\beta^*_{\mathbf{S}})  \right) \right\|_{\infty}\\
\leqslant {}& \left\|Q_{\mathbf{S}^c\mathbf{S}}Q_{\mathbf{S}\mathbf{S}}^{-\frac{1}{2}}\right\|_{\infty,2} \left\|Q_{\mathbf{S}\mathbf{S}}^{-\frac{1}{2}} \right\|_{2} \left\| r_{\mathbf{S}}(\hat{\beta}_{\mathbf{S}}) -r_{\mathbf{S}}(\beta^*_{\mathbf{S}})\right\|_{2}\\
\leqslant {}& 3 \gamma_{\max} (Q_{\mathbf{S}\mathbf{S}}^{-\frac{1}{2}}) \left\| r_{\mathbf{S}}(\hat{\beta}_{\mathbf{S}}) -r_{\mathbf{S}}(\beta^*_{\mathbf{S}})\right\|_{\infty}\\
        \leqslant{}& 3 \gamma^{-\frac{1}{2}}_{\min} (Q_{\mathbf{S}\mathbf{S}}) c_2\left\|\hat{\beta}_{\mathbf{S}}-\beta^*_{\mathbf{S}}\right\|_{\infty}\\
 \stackrel{\text{\eqref{eq:active}}}{\leqslant}& 3c_2 \gamma^{-\frac{1}{2}}_{\min} \left(Q_{\mathbf{S}\mathbf{S}}\right)|\mathbf{S}|^{\frac{1}{2}} \gamma^{-1}_{\min} \left(Q_{\mathbf{S}\mathbf{S}}\right)\left(\lambda_n A_{\mathbf{S}}+\left\|q_{\mathbf{S}}\right\|_{\infty}\right)\\
         ={}& 3\frac{6}{\beta^*_{\min}} \sum_{g \in \mathsf{G}_\mathbf{S}}w_g \sqrt{|\mathbf{S} \cap G_g|} \gamma^{-\frac{3}{2}}_{\min} \left(Q_{\mathbf{S}\mathbf{S}}\right) |\mathbf{S}|^{\frac{1}{2}} \left(\lambda_n A_{\mathbf{S}}+\left\|q_{\mathbf{S}}\right\|_{\infty}\right).
\end{aligned}
\end{equation*}

If the following conditions are satisfied:
\begin{equation}
    \label{eq:fc4}
    a_{\mathbf{S}^c}^{-1}\frac{6}{\beta^*_{\min}} \sum_{g \in \mathsf{G}_\mathbf{S}}w_g \sqrt{|\mathbf{S} \cap G_g|} \gamma^{-\frac{3}{2}}_{\min} \left(Q_{\mathbf{S}\mathbf{S}}\right) |\mathbf{S}|^{\frac{1}{2}} \lambda_n A_{\mathbf{S}}\leqslant \frac{\tau}{12},
\end{equation}

\begin{equation}
    \label{eq:fc5}
  a_{\mathbf{S}^c}^{-1} \frac{6}{\beta^*_{\min}} \sum_{g \in \mathsf{G}_\mathbf{S}}w_g \sqrt{|\mathbf{S} \cap G_g|} \gamma^{-\frac{3}{2}}_{\min} \left(Q_{\mathbf{S}\mathbf{S}}\right) |\mathbf{S}|^{\frac{1}{2}} \left\|q_{\mathbf{S}}\right\|_{\infty} \leqslant \frac{\tau}{12},
\end{equation}

\begin{equation}
    \label{eq:fc6}
(\phi_{\mathbf{S}}^c)^*[Q_{\mathbf{S}^c\mathbf{S}} Q_{\mathbf{S}\mathbf{S}}^{-1} \mathbf{r}_\mathbf{S}] \leqslant 1 - \tau, 
\end{equation}

\begin{equation}
    \label{eq:fc7}
(\phi_{\mathbf{S}}^c)^*[q_{\mathbf{S}^c \mid \mathbf{S} }] \leqslant \frac{\lambda_n\tau}{2},
\end{equation}

then we have

\begin{equation*}
    \begin{aligned}
    (\phi_{\mathbf{S}}^{c})^*\left[\nabla L(\hat{\beta}^R)_{\mathbf{S}^c} \right] & \stackrel{\text{\eqref{eq:deriveR}}}{=} 
 (\phi_{\mathbf{S}}^{c})^*\left[\lambda_n Q_{\mathbf{S}^c\mathbf{S}}Q_{\mathbf{S}\mathbf{S}}^{-1} \left( r_{\mathbf{S}}(\hat{\beta}_{\mathbf{S}}) -r_{\mathbf{S}}(\beta^*_{\mathbf{S}})  \right) + \lambda_n Q_{\mathbf{S}^c\mathbf{S}}Q_{\mathbf{S}\mathbf{S}}^{-1}r_{\mathbf{S}}(\beta^*_{\mathbf{S}})  - q_{\mathbf{S}^c \mid \mathbf{S} }\right] \\
    &\leqslant (\phi_{\mathbf{S}}^{c})^*\left[\lambda_n Q_{\mathbf{S}^c\mathbf{S}}Q_{\mathbf{S}\mathbf{S}}^{-1} \left( r_{\mathbf{S}}(\hat{\beta}_{\mathbf{S}}) -r_{\mathbf{S}}(\beta^*_{\mathbf{S}})  \right)\right]  + (\phi_{\mathbf{S}}^{c})^*\left[\lambda_n Q_{\mathbf{S}^c\mathbf{S}}Q_{\mathbf{S}\mathbf{S}}^{-1}r_{\mathbf{S}}(\beta^*_{\mathbf{S}})  \right]  + (\phi_{\mathbf{S}}^{c})^*\left[- q_{\mathbf{S}^c \mid \mathbf{S} }\right] \\
    &\leqslant  \lambda_n\left(\phi_{\mathbf{S}}^c\right)^{*} \left[ Q_{\mathbf{S}^c\mathbf{S}}Q_{\mathbf{S}\mathbf{S}}^{-1} \left( r_{\mathbf{S}}(\hat{\beta}_{\mathbf{S}}) -r_{\mathbf{S}}(\beta^*_{\mathbf{S}})  \right)\right] + \lambda_n(1-\tau) +\frac{\lambda_n\tau}{2}
    \\
    & \stackrel{\text{\eqref{eq:orderre}}}{\leqslant }  \lambda_n a \left(\mathbf{S}^c\right)^{-1} \left\|Q_{\mathbf{S}^c\mathbf{S}}Q_{\mathbf{S}\mathbf{S}}^{-1} \left( r_{\mathbf{S}}(\hat{\beta}_{\mathbf{S}}) -r_{\mathbf{S}}(\beta^*_{\mathbf{S}})  \right)\right\|_{\infty} + \lambda_n - \frac{\lambda_n\tau}{2} \\
   & \leqslant \frac{\lambda_n\tau}{4} + \frac{\lambda_n\tau}{4} + \lambda_n - \frac{\lambda_n\tau}{2} \leqslant \lambda_n,
    \end{aligned}
\end{equation*}
which is Equation \eqref{eq:cond2}.

Because Equation \eqref{eq:cond1} and Equation \eqref{eq:cond2} are satisfied, Lemma \ref{lem:jenatton14} implies that $\hat{\beta}^R$ is the optimal.  Thus, 
$$\textit{supp}(\hat{\beta}^G) = \textit{supp}(\hat{\beta}^R)  = \mathbf{S}.$$

\subsubsection{Part IV}

The results in Part \rom{3} depend on conditions \eqref{eq:fc1}, \eqref{eq:fc2}, \eqref{eq:fc3}, \eqref{eq:fc4}, \eqref{eq:fc5}, \eqref{eq:fc6}, and \eqref{eq:fc7}, which are summarized as follows:

\begin{equation}
    \label{da0}
   \|Q_{\mathbf{S}^c\mathbf{S}}Q_{\mathbf{S}\mathbf{S}}^{-\frac{1}{2}}\|_{2,\infty} \leqslant 3,
\end{equation}

\begin{equation}
    \label{da1}
    \lambda_n|\mathbf{S}|^{\frac{1}{2}} \leqslant \min\left\{\frac{\gamma_{\min}\left(Q_{\mathbf{S}\mathbf{S}}\right) \beta^*_{\min}}{3 A_{\mathbf{S}}},\frac{\tau  \gamma^{\frac{3}{2}}_{\min}(Q_{\mathbf{S}\mathbf{S}}) a_{\mathbf{S}^c} \beta^*_{\min}}{72 A_{\mathbf{S}} \sum\limits_{g\in \mathsf{G}_{\mathbf{S}}} w_g \sqrt{\left|G_g \cap \mathbf{S}\right|}}\right\},
\end{equation}

\begin{equation}
    \label{da2}
    (\phi_{\mathbf{S}}^c)^*[Q_{\mathbf{S}^c\mathbf{S}} Q_{\mathbf{S}\mathbf{S}} \mathbf{r}_\mathbf{S}] \leqslant 1 - \tau,
\end{equation}

\begin{equation}
    \label{da3}
    (\phi_{\mathbf{S}}^c)^*[q_{\mathbf{S}^c \mid \mathbf{S} }] \leqslant \frac{\lambda_n\tau}{2},
\end{equation}

\begin{equation}
    \label{da4}
    \left\|q_{\mathbf{S}}\right\|_{\infty} \leqslant  \min\left\{\frac{\gamma_{\min}\left(Q_{\mathbf{S}\mathbf{S}}\right) \beta^*_{\min}}{3 A_{\mathbf{S}}},\frac{\tau  \gamma^{\frac{3}{2}}_{\min}(Q_{\mathbf{S}\mathbf{S}})  a_{\mathbf{S}^c} \beta^*_{\min}}{72 A_{\mathbf{S}} \sum\limits_{g\in \mathsf{G}_{\mathbf{S}}} w_g \sqrt{\left|G_g \cap \mathbf{S}\right|}}\right\}.
\end{equation}

In Part \rom{4}, we want to make sure that these conditions hold with high probability. 

\noindent\textbf{Condition \eqref{da0}}

To begin with, for any matrix $A \in \mathbb{R}^{m \times n}$, the Cauchy-Schwarz inequality implies that 
\begin{equation*}
    \begin{aligned}
         \|A\|_{2,\infty} &= \sup\limits_{\|u\|_{2} \leqslant 1} \|Au\|_{\infty} = \sup\limits_{\|u\|_{2} \leqslant 1}\max\limits_{i \in [m]} \Big(\sqrt{\sum\limits_{j \in [n]}A_{ij}u_j}\Big) \\ 
        & \leqslant  \sup\limits_{\|u\|_{2} \leqslant 1}\max\limits_{i \in [m]} \Big(\sqrt{\sum\limits_{j \in [n]}A^2_{ij}}\sqrt{\sum\limits_{j \in [n]}u^2_j}\Big) \\& \leqslant  \max\limits_{i \in [m]} \Big(\sqrt{\sum\limits_{j \in [n]}A^2_{ij}}\Big) \leqslant \max\limits_{i \in [m]} \Big\{\sqrt{\operatorname{diag}(AA^\top)}\Big\}.
    \end{aligned}
\end{equation*}

Recall that $Q=\frac{1}{n} X^{\top}X$. Let $A = Q_{\mathbf{S}^c\mathbf{S}}Q_{\mathbf{S}\mathbf{S}}^{-\frac{1}{2}}$, we have

\begin{equation*}
    \begin{aligned}
        \|Q_{\mathbf{S}^c\mathbf{S}}Q_{\mathbf{S}\mathbf{S}}^{-\frac{1}{2}}\|_{2,\infty}  \leqslant \max \{\sqrt{\operatorname{diag}(Q_{\mathbf{S}^c\mathbf{S}}Q_{\mathbf{S}\mathbf{S}}^{-1}Q_{\mathbf{S}\mathbf{S}^c})}\}.
    \end{aligned}
\end{equation*}

Using the Schur complement of $Q$ on the block matrices $Q_{\mathbf{S}\mathbf{S}}$ and $Q_{\mathbf{S}^c\mathbf{S}^c}$, the positiveness of $Q$ implies the positiveness of $Q_{\mathbf{S}^c\mathbf{S}^c} - Q_{\mathbf{S}^c\mathbf{S}}Q_{\mathbf{S}\mathbf{S}}^{-1}Q_{\mathbf{S}\mathbf{S}^c}$. Thus,
$$\max\operatorname{diag}(Q_{\mathbf{S}^c\mathbf{S}}Q_{\mathbf{S}\mathbf{S}}^{-1}Q_{\mathbf{S}\mathbf{S}^c}) \leqslant \max \operatorname{diag}(Q_{\mathbf{S}^c\mathbf{S}^c}) \leqslant \max\limits_{j \in \mathbf{S}^c} Q_{jj}.$$

\begin{lemma}(Lemma 1 of \citet{laurent2000adaptive})
\label{lem:laurent2000adaptive1}
 
Suppose that the random variable $U$ follows $\chi^2$ distribution with $d$ degrees of freedom, then for any positive $x$,
$$
\begin{array}{r}
 \mathbb{P}(U-d \geq 2 \sqrt{d x}+2 x) \leqslant \exp (-x), \\
 \mathbb{P}(d-U \geq 2 \sqrt{d x}) \leqslant \exp (-x) .
\end{array}
$$
\end{lemma}

As $X$ follows multivariate normal, $\Tilde{Q}_{jj} = \frac{nQ_{jj}}{\Theta_{jj}^2} \sim \chi^2_n$. Then by Lemma~\ref{lem:laurent2000adaptive1}, we have 

\begin{equation}
\label{maxdiag}
        \begin{aligned}
   \mathbb{P}(\max\limits_{j \in \mathbf{S}^c} \sqrt{Q_{jj}} > 3) & \leqslant  \mathbb{P}(\max\limits_{j \in \mathbf{S}^c} Q_{jj} > 5) \leqslant  \mathbb{P}(\bigcup\limits_{j \in \mathbf{S}^c} Q_{jj} > 5) \leqslant  \sum\limits_{j \in \mathbf{S}^c} \mathbb{P}(Q_{jj} > 5) \\
   & \leqslant  \sum\limits_{j \in \mathbf{S}^c} \mathbb{P}(Q_{jj} > 5\Theta^2_{jj}) =  \sum\limits_{j \in \mathbf{S}^c} \mathbb{P}(n\frac{Q_{jj}}{\Theta^2_{jj}} > 5n) \\
   & \leqslant  \sum\limits_{j \in \mathbf{S}^c} \mathbb{P}(\Tilde{Q}_{jj} > n + 2n + 2n) \leqslant (p-|\mathbf{S}|) \exp(-n) \\ &= \exp(-n + \log (p-|\mathbf{S}|))  \\ &\leqslant \exp(-\frac{n}{2}),
    \end{aligned}
\end{equation}
where the last inequality holds as $n >  2\log(p-|\mathbf{S}|)$. Thus,
 $$\mathbb{P}( \|Q_{\mathbf{S}^c\mathbf{S}}Q_{\mathbf{S}\mathbf{S}}^{-\frac{1}{2}}\|_{2,\infty} > 3) \leqslant \mathbb{P}(\max\limits_{j \in \mathbf{S}^c} \sqrt{Q_{jj}} > 3) \leqslant   \exp(-\frac{n}{2}).$$

 Similarly, let $Q_{\mathbf{S}^c \mathbf{S}^c \mid \mathbf{S}}=Q_{\mathbf{S}^c \mathbf{S}^c}-Q_{\mathbf{S}^c \mathbf{S}} Q_{\mathbf{S} \mathbf{S}}^{-1} Q_{\mathbf{S}\mathbf{S}^c}$. The diagonal terms of $Q_{\mathbf{S}^c \mathbf{S}^c \mid \mathbf{S}}$ is less than the diagonal terms of $Q_{\mathbf{S}^c \mathbf{S}^c}$, which implies 
 $$\mathbb{P}(\|Q_{\mathbf{S}^c \mathbf{S}^c  \mid \mathbf{S}}^{1 / 2}\|_{2,\infty} > 3) \leqslant \mathbb{P}(\max\limits_{j \in \mathbf{S}^c} \sqrt{Q_{jj}} > 3) \leqslant   \exp(-\frac{n}{2}).$$\\[0.2in]
 
\noindent\textbf{Condition \eqref{da1}}

\begin{lemma}(Lemma 9 of \citet{wainwright2009sharp})
\label{lem:wainwright9}

 Suppose that $d \leqslant n$ and $X \in \mathbb{R}^{n \times d}$ have i.i.d rows $X_i \sim N(0, \Theta)$, then 
\begin{equation*}
 \mathbb{P}\left(\gamma_{\max}\left(\frac{1}{n}X^\top X\right) \geqslant 9\gamma_{\max}(\Theta)\right) \leqslant 2 \exp(-\frac{n}{2}),
\end{equation*}

\begin{equation*}
 \mathbb{P}\left(\gamma_{\max}\left((\frac{1}{n}X^\top X)^{-1}\right) \geqslant \frac{9}{\gamma_{\min}(\Theta)}\right) \leqslant 2 \exp(-\frac{n}{2}).
\end{equation*}

\end{lemma}

As we assume that $|\mathbf{S}| \leqslant n$ and $X_{\mathbf{S}\mathbf{S}} \sim N\left( 0, \Theta_{\mathbf{S}\mathbf{S}}\right)$,
then  Lemma~\ref{lem:wainwright9} implies
\begin{equation*}
  \mathbb{P}\left(\gamma_{\max}(Q_{\mathbf{S}\mathbf{S}}) \geqslant 9\gamma_{\max}(\Theta_{\mathbf{S}\mathbf{S}})\right) \leqslant 2\exp(-\frac{n}{2}),
\end{equation*}
and also
\begin{equation*}
\mathbb{P}\left(\gamma_{\min}(\Theta_{\mathbf{S}\mathbf{S}}) \geqslant 9\gamma_{\min}(Q_{\mathbf{S}\mathbf{S}})\right) \leqslant 2\exp(-\frac{n}{2}).
\end{equation*}

Thus, by assuming that
\begin{equation*}
    \lambda_n|\mathbf{S}|^{\frac{1}{2}} \leqslant \min\bigg\{\frac{3\gamma_{\min}(\Theta)\beta_{\min}^*}{A_{\mathbf{S}}},\frac{\tau  \gamma^{\frac{3}{2}}_{\min}(\Theta) a_{\mathbf{S}^c} \beta_{\min}^*}{8A_{\mathbf{S}} \sum\limits_{g\in \mathsf{G}_{\mathbf{S}}} w_g \sqrt{\left|G_g \cap \mathbf{S}\right|}}\bigg\},
\end{equation*}
we have 
\begin{equation*}
    \lambda_n|\mathbf{S}|^{\frac{1}{2}} \leqslant \min\bigg\{\frac{\gamma_{\min}\left(Q_{\mathbf{S}\mathbf{S}}\right) \beta^*_{\min}}{3 A_{\mathbf{S}}},\frac{\tau  \gamma^{\frac{3}{2}}_{\min}(Q_{\mathbf{S}\mathbf{S}}) a_{\mathbf{S}^c} \beta^*_{\min}}{72 A_{\mathbf{S}} \sum\limits_{g\in \mathsf{G}_{\mathbf{S}}} w_g \sqrt{\left|G_g \cap \mathbf{S}\right|}}\bigg\}
\end{equation*}
holds with high probability.\\[0.2in]

\noindent\textbf{Condition \eqref{da2}}

For any $j\in \mathbf{S}^c$, $X_j \in \mathbb{R}^n$ is zero-mean Gaussian. Following the decomposition in \citet{wainwright2009sharp}, we have
\begin{equation}
\label{eq:wainde}
    X_j^\top = \Theta_{j\mathbf{S}} \Theta_{\mathbf{S}\mathbf{S}}^{-1}
     X_{\mathbf{S}}^\top
    + E_j^\top,
\end{equation}
where $E_j$ are i.i.d from $N\left(0, \left[\Theta_{\mathbf{S}^c\mathbf{S}^c|\mathbf{S}}\right]_{jj}\right)$ with
$\Theta_{\mathbf{S}^c\mathbf{S}^c|\mathbf{S}}=\Theta_{\mathbf{S}^c\mathbf{S}^c} - \Theta_{\mathbf{S}^c\mathbf{S}}\left(\Theta_{\mathbf{S}\mathbf{S}}\right)^{-1}\Theta_{\mathbf{S}\mathbf{S}^c}$. Let $ E_{\mathbf{S}^c}$ be an $|S^c| \times n$ matrix, with each row representing $E_j$ for an element $j\in \mathbf{S}^c$, then we have

\begin{equation}
\label{eq:xirecall}
    \begin{aligned}
        Q_{\mathbf{S}^c\mathbf{S}} Q_{\mathbf{S}\mathbf{S}}^{-1} \mathbf{r}_\mathbf{S} &= X^\top_{\mathbf{S}^c}  X_{\mathbf{S}}(X^\top_{\mathbf{S}}  X_{\mathbf{S}})^{-1} \mathbf{r}_\mathbf{S}\\
        &\stackrel{\text{\eqref{eq:wainde}}}{=}  \left(\Theta_{\mathbf{S}^c\mathbf{S}} \Theta_{\mathbf{S}\mathbf{S}}^{-1}
     X_{\mathbf{S}}^\top
    +  E^\top_{\mathbf{S}^c} \right)X_{\mathbf{S}}(X^\top_{\mathbf{S}}  X_{\mathbf{S}})^{-1} \mathbf{r}_\mathbf{S}\\&=\Theta_{\mathbf{S}^c\mathbf{S}} \Theta_{\mathbf{S}\mathbf{S}}^{-1}\mathbf{r}_\mathbf{S} +
       E^\top_{\mathbf{S}^c}X_{\mathbf{S}}(X^\top_{\mathbf{S}}  X_{\mathbf{S}})^{-1} \mathbf{r}_\mathbf{S}\\
       & := \Theta_{\mathbf{S}^c\mathbf{S}} \Theta_{\mathbf{S}\mathbf{S}}^{-1}\mathbf{r}_\mathbf{S} + \eta.
    \end{aligned}
\end{equation}

The preceding expression prompts us to establish an upper bound for the dual norm of $\eta$. To achieve this, we begin by examining the scenario in which \underline{$X_{\mathbf{S}}$ is fixed}. Our objective now is to derive the covariance matrix of $\eta$. For any $j\in \mathbf{S}^c$, we have $$\mathbb{E}[\eta_j] = \mathbb{E}\left[  E^\top_{j}X_{\mathbf{S}}(X^\top_{\mathbf{S}}  X_{\mathbf{S}})^{-1} \mathbf{r}_\mathbf{S}\right] = 0 .$$

For any pair of $j, k \in \mathbf{S}^c$, we have
\begin{align*}
\mathbb{E}[\eta_j\eta_k]
={}& \mathbb{E}\left[ E_j^\top X_{\mathbf{S}}(X^\top_{\mathbf{S}}  X_{\mathbf{S}})^{-1}\mathbf{r}_\mathbf{S} E_k^\top X_{\mathbf{S}}(X^\top_{\mathbf{S}}  X_{\mathbf{S}})^{-1}\mathbf{r}_\mathbf{S} \right] \\
={}& \mathbb{E}\left[ \mathbf{r}_\mathbf{S}^\top (X^\top_{\mathbf{S}}  X_{\mathbf{S}})^{-1} X_{\mathbf{S}}^\top E_jE_k^\top  X_{\mathbf{S}} (X^\top_{\mathbf{S}}  X_{\mathbf{S}})^{-1} \mathbf{r}_\mathbf{S}  \right] \\
={}& \mathbf{r}_\mathbf{S}^\top (X^\top_{\mathbf{S}}  X_{\mathbf{S}})^{-1} X_{\mathbf{S}}^\top \mathbb{E}\left[ E_j E_k^\top  \right]  X_{\mathbf{S}}(X^\top_{\mathbf{S}}  X_{\mathbf{S}})^{-1} \mathbf{r}_\mathbf{S},
\end{align*}
where
\begin{align*}
&\mathbb{E}\left[ E_j E_k^\top  \right]
\stackrel{\text{\eqref{eq:wainde}}}{=}{} \mathbb{E} \left[ \left(X_j - X_{\mathbf{S}} \Theta_{\mathbf{S}\mathbf{S}}^{-1}
    \Theta_{j\mathbf{S}}^\top\right)\left(  X_k^\top - \Theta_{k\mathbf{S}} \Theta_{\mathbf{S}\mathbf{S}}^{-1}
     X_{\mathbf{S}}^\top \right)\right]\\
={}& \mathbb{E}\left[X_j X_k^\top \right] - \mathbb{E}\left[ X_{\mathbf{S}}\Theta_{\mathbf{S}\mathbf{S}}^{-1} \Theta_{j\mathbf{S}} X_k^\top \right] - \mathbb{E}\left[ X_j\Theta_{k\mathbf{S}}\Theta_{\mathbf{S}\mathbf{S}}^{-1} X_{\mathbf{S}}^\top\mid  X_{\mathbf{S}}\right] + \mathbb{E}\left[  X_{\mathbf{S}} \Theta_{\mathbf{S}\mathbf{S}}^{-1} \Theta_{j\mathbf{S}}^\top \Theta_{k\mathbf{S}} \Theta_{\mathbf{S}\mathbf{S}}^{-1}  X_{\mathbf{S}}^\top \right] \\
={}& \mathbb{E}\left[X_j X_k^\top \right] -  X_{\mathbf{S}} \Theta_{\mathbf{S}\mathbf{S}}^{-1} \Theta_{j\mathbf{S}} \mathbb{E}\left[ X_k^\top \right] - \mathbb{E}\left[ X_j \right] \Theta_{k\mathbf{S}} \Theta_{\mathbf{S}\mathbf{S}}^{-1}  X_{\mathbf{S}}^\top +  X_{\mathbf{S}} \Theta_{\mathbf{S}\mathbf{S}}^{-1} \Theta_{j\mathbf{S}} \Theta_{k\mathbf{S}} \Theta_{\mathbf{S}\mathbf{S}}^{-1}  X_{\mathbf{S}}^\top \\
={}& \mathbb{E}\left[  X_j X_k^\top \right] - \mathbb{E}\left[ X_j \right] \mathbb{E}\left[ X_k^\top \right] = \operatorname{Cov}\left[X_j, X_k^\top \right] = \left(\Theta_{\mathbf{S}^c\mathbf{S}^c|\mathbf{S}} \right)_{jk} \mathbf{I}_{n\times n}.
\end{align*}
Consequently,
\begin{align*}
\mathbb{E}[\eta_j\eta_k]
={}&\mathbf{r}_\mathbf{S}^\top (X^\top_{\mathbf{S}}  X_{\mathbf{S}})^{-1} X_{\mathbf{S}}^\top \mathbb{E}\left[ E_j E_k^\top  \right]  X_{\mathbf{S}}(X^\top_{\mathbf{S}}  X_{\mathbf{S}})^{-1} \mathbf{r}_\mathbf{S} \\
={}& \mathbf{r}_\mathbf{S}^\top (X^\top_{\mathbf{S}}  X_{\mathbf{S}})^{-1} X_{\mathbf{S}}^\top \left(\Theta_{\mathbf{S}^c\mathbf{S}^c|\mathbf{S}} \right)_{jk} \mathbf{I}_{n\times n} X_{\mathbf{S}}(X^\top_{\mathbf{S}}  X_{\mathbf{S}})^{-1} \mathbf{r}_\mathbf{S}\\
={}& \mathbf{r}_\mathbf{S}^\top (X^\top_{\mathbf{S}}  X_{\mathbf{S}})^{-1}\mathbf{r}_\mathbf{S}\cdot \left(\Theta_{\mathbf{S}^c\mathbf{S}^c|\mathbf{S}} \right)_{jk} = \frac{\mathbf{r}_\mathbf{S}^\top (Q_{\mathbf{S}\mathbf{S}})^{-1}\mathbf{r}_\mathbf{S}}{n} \cdot \left(\Theta_{\mathbf{S}^c\mathbf{S}^c|\mathbf{S}} \right)_{jk}.
\end{align*}

And we have $\operatorname{Cov}(\eta) = \frac{\mathbf{r}_\mathbf{S}^\top (Q_{\mathbf{S}\mathbf{S}})^{-1}\mathbf{r}_\mathbf{S}}{n} \cdot \left(\Theta_{\mathbf{S}^c\mathbf{S}^c|\mathbf{S}} \right) := \Xi$.

\begin{lemma}(Theorem 2.26 in \citet{wainwright2019high})
\label{lem:masset}

Let  $\left(X_{1}, \ldots, X_{n}\right)$  be a vector of i.i.d. standard Gaussian variables, and let  $f: \mathbb{R}^{n} \mapsto \mathbb{R}$  be a Lipschitz function with respect to the Euclidean norm and Lipschitz constant L. Then the variable  $f(X)-\mathbb{E}[f(X)]$  is sub-Gaussian with parameter at most  $L$, and hence

\begin{equation*}
 \mathbb{P}[|f(X)-\mathbb{E}[f(X)]| \geqslant t] \leqslant 2 \exp(-\frac{t^{2}}{2 L^{2}}) \quad \text { for all } t \geqslant 0 . 
\end{equation*} 

\end{lemma}

To apply the concentration bound in Lemma \ref{lem:masset}, we define function $\Psi(u) = \left( \phi_{\mathbf{S}^c}^{*} \right) \left[ \Xi^{\frac{1}{2}} u \right]$. As  $\eta=\Xi^{\frac{1}{2}} W$ where $W \sim N(0,I_{|\mathbf{S}^c|\times |\mathbf{S}^c|})$, $(\phi_{\mathbf{S}}^c)^{*}(\eta)$ has the same distribution  as $\Psi(W)$  . We continue to show that $\Psi$ is a Lipschitz function given fixed $ X_{\mathbf{S}}$.

\begin{equation*}
    \begin{aligned}
\left| \Psi(u)-\Psi(v)\right|
\leqslant {}& \Psi(u-v) = \left( \phi_{\mathbf{S}}^c \right)^{*} \left[\Xi^{\frac{1}{2}} (u-v)\right] \\
\leqslant {}& a_{\mathbf{S}}^{-1} \left\| \Xi^{\frac{1}{2}} (u-v) \right\|_{\infty}\\
=         {}& a_{\mathbf{S}}^{-1} \bigg\|\left[\frac{\mathbf{r}_\mathbf{S}^\top (Q_{\mathbf{S}\mathbf{S}})^{-1}\mathbf{r}_\mathbf{S}}{n} \cdot \left(\Theta_{\mathbf{S}^c\mathbf{S}^c|\mathbf{S}} \right)\right]^{\frac{1}{2}}(u-v) \bigg\|_{\infty} \\
\leqslant {}& a_{\mathbf{S}}^{-1} \left\|\mathbf{r}_\mathbf{S}\right\|_2  n^{-\frac{1}{2}} \gamma^{\frac{1}{2}}_{\max} \left(Q_{\mathbf{S}\mathbf{S}}^{-1}\right) \gamma^{\frac{1}{2}}_{\max} \left(\Theta_{\mathbf{S}^c\mathbf{S}^c|\mathbf{S}}\right) \left\|u-v\right\|_2.
\end{aligned}
\end{equation*}

Thus, the corresponding Lipstichiz constant is
$$L_{\eta} = a_{\mathbf{S}}^{-1} \left\|\mathbf{r}_\mathbf{S}\right\|_2  n^{-\frac{1}{2}} \gamma^{\frac{1}{2}}_{\max} \left(Q_{\mathbf{S}\mathbf{S}}^{-1}\right)\gamma^{\frac{1}{2}}_{\max} \left(\Theta_{\mathbf{S}^c\mathbf{S}^c|\mathbf{S}}\right).$$

On the other hand, suppose that $    \mathbb{E}\left[(\phi_{\mathbf{S}}^c)^{*}(\eta)\right] \leqslant  \frac{\tau}{4}$, since $\Psi$ is a Lipschitiz function, by applying $t = \frac{\tau}{4}$ in  concentration Lemma~\ref{lem:masset} on Lipschitz functions of multivariate standard random variables, we have 

\begin{equation*}
    \begin{aligned}
    \mathbb{P}\left(\left(\phi_{\mathbf{S}}^c\right)^{*} [\eta] > \frac{\tau}{2} \right) &= \mathbb{P} \left(\Psi(W) > \frac{\tau}{2} \right) = \mathbb{P} \left(\Psi(W) - \frac{\tau}{4}  > \frac{\tau}{4} \right)\\& \leqslant  
    \mathbb{P} \left(\Psi(W) - E\left[ \left(\phi_{\mathbf{S}}^c\right)^{*} (\eta) \right]  > \frac{\tau}{4} \right)\\& =
    \mathbb{P} \left(\Psi(W) - E\left[ \Psi(W)\right]  > \frac{\tau}{4} \right) \leqslant  \exp\left(-\frac{\tau^2}{4L^2_{\eta}}\right).
    \end{aligned}
\end{equation*}

Now we further assume that $\{\gamma_{\max}(Q_{\mathbf{S}\mathbf{S}}^{-1}) \leqslant \frac{9}{\gamma_{\min}(\Theta_{\mathbf{S}\mathbf{S}})}\}$. Under this condition, we have
\begin{equation}
\label{eq:expx3}
    L_{\eta} = a_{\mathbf{S}}^{-1} \left\|\mathbf{r}_\mathbf{S}\right\|_2  n^{-\frac{1}{2}} \gamma^{\frac{1}{2}}_{\max} \left(Q_{\mathbf{S}\mathbf{S}}^{-1}\right) \gamma^{\frac{1}{2}}_{\max} \left(\Theta_{\mathbf{S}^c\mathbf{S}^c|\mathbf{S}}\right)
\leqslant  \frac{3a_{\mathbf{S}}^{-1} \left\|\mathbf{r}_\mathbf{S}\right\|_2  \gamma^{\frac{1}{2}}_{\max}\left(\Theta_{\mathbf{S}^c\mathbf{S}^c|\mathbf{S}}\right)}{\left(n\gamma_{\min}(\Theta_{\mathbf{S}\mathbf{S}})\right)^{\frac{1}{2}}} .
\end{equation}

\begin{lemma}(Sudakov inequality, Theorem 5.27 in \citet{wainwright2019high})
\label{lem:suda}

 If $X$ and $Y$ are a.s. bounded, centered Gaussian processes on $T$ such that
$$
\mathbb{E}\left(X_t-X_s\right)^2 \leq \mathbb{E}\left(Y_t-Y_s\right)^2
$$
then
$$
\mathbb{E} \sup _T X_t \leq \mathbb{E} \sup _T Y_t .
$$
\end{lemma}

\begin{lemma} (Exercise 2.12 in \citet{wainwright2019high})
\label{lem:Upper bounds for sub-Gaussian maxima}
    Let \( X_1, \ldots, X_n \) be independent \(\sigma^2\)-subgaussian random variables. Then
\[
\mathbb{E}[\max_{1 \leq i \leq n} |X_i|] \leq 2\sqrt{\sigma^2\log n}.
\]
\end{lemma}
On the other hand, for any $u_t,u_s$, we have 
\begin{equation*}
\begin{aligned}
      &\mathbb{E}(u_t^\top \eta - u_s^\top  \eta )^2 = \mathbb{E}(u_t^\top \Xi^{\frac{1}{2}} W - u_s^\top \Xi^{\frac{1}{2}} W )^2 = (u_t-u_s)^\top \Xi (u_t-u_s) \\
     \leqslant & ||u_t-u_s||^2_2\gamma_{\max} \left(\Xi\right) = \mathbb{E}( \gamma^{\frac{1}{2}} _{\max} \left(\Xi\right)u_t^\top W -  \gamma^{\frac{1}{2}} _{\max} \left(\Xi\right)u_s^\top  W )^2
\end{aligned}
\end{equation*}
By using Sudakov-Fernique inequality in Lemma~\ref{lem:suda},  we have
\begin{equation*}
\mathbb{E}\Big[\sup_{\phi_{\mathbf{S}}^c(u)\leqslant 1} u^\top \Xi^{\frac{1}{2}} W \Big]
\leqslant \mathbb{E}\Big[\sup_{\phi_{\mathbf{S}}^c(u)\leqslant 1}\gamma^{\frac{1}{2}} _{\max} \left(\Xi\right)u^\top W \Big]
\end{equation*}
Consequently,
\begin{equation}
    \label{eq:expxi}
    \begin{aligned}
\mathbb{E}\Big[(\phi_{\mathbf{S}}^c)^{*}(\eta)\Big]
={}&  \mathbb{E}\Big[\sup_{\phi_{\mathbf{S}}^c(u)\leqslant 1}u^\top \eta \Big] = \mathbb{E}\Big[\sup_{\phi_{\mathbf{S}}^c(u)\leqslant 1} u^\top \Xi^{\frac{1}{2}} W \Big]
\\
\leqslant{}& \gamma^{\frac{1}{2}} _{\max} \left(\Xi\right)\mathbb{E}\Big[\sup_{\phi_{\mathbf{S}}^c(u)\leqslant 1}u^\top W \Big] = \gamma_{\max} \left(\Xi\right)^{\frac{1}{2}} \mathbb{E}\Big[(\phi_{\mathbf{S}}^c)^{*}(W)\Big].
\end{aligned}
\end{equation}

 Notice that
\begin{equation}
\label{eq:mapping}
    \begin{aligned}
      \left\|\mathbf{r}_\mathbf{S}\right\|^2_2  & \leqslant |\mathbf{S}|\max_{j \in \mathbf{S}}\mathbf{r}^2_j =  |\mathbf{S}| \Big(\max_{j \in \mathbf{S}}\{\beta^*_j \cdot \sum\limits_{g \in \mathsf{G}^G_{ \mathbf{S}}, G_g \cap j \neq \emptyset} \frac{w_g}{\|\beta^*_{G_g \cap \mathbf{S}}\|_2}\}\Big)^2 \\ & \leqslant  |\mathbf{S}| \Big(\max_{j \in \mathbf{S}}\{|\beta^*_j|\}\cdot \max\{\sum\limits_{g \in \mathsf{G}^G_{ \mathbf{S}}, G_g \cap j \neq \emptyset} \frac{w_g}{\|\beta^*_{G_g \cap \mathbf{S}}\|_2}\}\Big)^2 \\ & \leqslant 
      |\mathbf{S}| \Big(\frac{\max\limits_{j \in \mathbf{S}}\{|\beta^*_j|\}}{\beta_{\min}^*}\cdot \max\{\sum\limits_{g \in \mathsf{G}^G_{ \mathbf{S}}, G_g \cap j \neq \emptyset} \frac{w_g}{\sqrt{|G_g \cap \mathbf{S}}}\}\Big)^2 \\ & \leqslant 
      |\mathbf{S}| \Big(\frac{\max\limits_{j \in \mathbf{S}}\{|\beta^*_j|\}}{\beta_{\min}^*}\cdot \max\{\sum\limits_{g \in \mathsf{G}^G_{ \mathbf{S}}, G_g \cap j \neq \emptyset} w_g\}\Big)^2  \\ & \leqslant 
      |\mathbf{S}| \Big(\frac{\max\limits_{j \in \mathbf{S}}\{|\beta^*_j|\}}{\beta_{\min}^*}\cdot h_{\max}(\mathbf{G_{S}}) \max\limits_{g \in \mathsf{G}^G_{\mathbf{S}}}w_g\}\Big)^2
       \\ & \leqslant   \Big(\frac{\max\limits_{j \in \mathbf{S}}\{|\beta^*_j|\}}{\beta_{\min}^*}\Big)^2  |\mathbf{S}| A^2_{\mathbf{S}} =   \max\limits_{j \in \mathbf{S}}\{(\beta^*_j)^2\}|\mathbf{S}| \Big(\frac{A_{\mathbf{S}}}{\beta_{\min}^*}\Big)^2
       \\& \lesssim \frac{ \max\limits_{j \in \mathbf{S}}\{(\beta^*_j)^2\}}{\lambda^2_n}
    \end{aligned}
\end{equation}

Thus,  if $ X_{\mathbf{S}}$ satisfies $\gamma_{\max}(Q_{\mathbf{S}\mathbf{S}}^{-1}) \leqslant \frac{9}{\gamma_{\min}(\Theta_{\mathbf{S}\mathbf{S}})}$, we have
\begin{equation}
\label{eq:exp4}
   \begin{aligned}
    \mathbb{E}\left[(\phi_{\mathbf{S}}^c)^{*}(\eta)\right] & \stackrel{\text{\eqref{eq:expxi}}}{\leqslant } \gamma_{\max} \left(\Xi\right)^{\frac{1}{2}} \mathbb{E}\left[(\phi_{\mathbf{S}}^c)^{*}(W)\right] \\
    &\leqslant  \frac{\left\|\mathbf{r}_\mathbf{S}\right\|_2 \gamma^{-\frac{1}{2}}_{\min} \left(Q_{\mathbf{S}\mathbf{S}}\right)\gamma^{\frac{1}{2}}_{\max} \left(\Theta_{\mathbf{S}^c\mathbf{S}^c|\mathbf{S}}\right)}{n^{\frac{1}{2}}}
     \mathbb{E}\left[(\phi_{\mathbf{S}}^c)^{*}(W)\right]\\
    &\leqslant  \frac{\left\|\mathbf{r}_\mathbf{S}\right\|_2 3 \gamma^{\frac{1}{2}}_{\max} \left(\Theta_{\mathbf{S}^c\mathbf{S}^c|\mathbf{S}}\right)}{(n\gamma_{\min}(\Theta_{\mathbf{S}\mathbf{S}}))^{\frac{1}{2}}}\mathbb{E}\left[(\phi_{\mathbf{S}}^c)^{*}(W)\right]\\
    &\stackrel{\text{\eqref{eq:orderre}}}{\leqslant } \frac{\left\|\mathbf{r}_\mathbf{S}\right\|_2 3 \gamma^{\frac{1}{2}}_{\max} \left(\Theta_{\mathbf{S}^c\mathbf{S}^c|\mathbf{S}}\right)}{(n\gamma_{\min}(\Theta_{\mathbf{S}\mathbf{S}}))^{\frac{1}{2}}}\mathbb{E}\left[ a_{\mathbf{S}^c}^{-1}\|W\|_{\infty}\right]\\
    &\leqslant   \frac{\left\|\mathbf{r}_\mathbf{S}\right\|_2 3 \gamma^{\frac{1}{2}}_{\max} \left(\Theta_{\mathbf{S}^c\mathbf{S}^c|\mathbf{S}}\right)}{ a_{\mathbf{S}^c}(n\gamma_{\min}(\Theta_{\mathbf{S}\mathbf{S}}))^{\frac{1}{2}}}\mathbb{E}\left[\|W\|_{\infty}\right]\\
    &\stackrel{\text{Lemma~\ref{lem:Upper bounds for sub-Gaussian maxima}}}{\leqslant } \frac{6\left\|\mathbf{r}_\mathbf{S}\right\|_2  \gamma^{\frac{1}{2}}_{\max} \left(\Theta_{\mathbf{S}^c\mathbf{S}^c|\mathbf{S}}\right)}{ a_{\mathbf{S}^c}(n\gamma_{\min}(\Theta_{\mathbf{S}\mathbf{S}}))^{\frac{1}{2}}}\sqrt{\log(p-|\mathbf{S}|)}\leqslant  \frac{\tau}{4},
\end{aligned}
\end{equation}
where the last inequality holds as Assumption \ref{ass:Irrepresentable_condition} implies that
$$ n \gtrsim \frac{ \max\limits_{j \in \mathbf{S}}\{(\beta^*_j)^2\}\log(p-|\mathbf{S}|)}{a^2_{\mathbf{S}^c}\lambda^2_n} \stackrel{\text{\eqref{eq:mapping}}}{\gtrsim }\frac{\left\|\mathbf{r}_\mathbf{S}\right\|^2_2\log(p-|\mathbf{S}|)}{a^2_{\mathbf{S}^c}}  \geqslant \frac{576\left\|\mathbf{r}_\mathbf{S}\right\|^2_2\log(p-|\mathbf{S}|)\gamma_{\max} \left(\Theta_{\mathbf{S}^c\mathbf{S}^c|\mathbf{S}}\right)}{a^2_{\mathbf{S}^c}\gamma_{\min}(\Theta_{\mathbf{S}\mathbf{S}})\tau^2}.$$
Consequently, Equation \eqref{eq:expx3} and \eqref{eq:exp4} together implies

\begin{equation}
\label{eq:condip}
    \begin{aligned}
      & \mathbb{P}\Big(\left(\phi_{\mathbf{S}}^c\right)^{*} [\eta] > \frac{\tau}{2} \mid X_{\mathbf{S}}, \gamma_{\max}(Q_{\mathbf{S}\mathbf{S}}^{-1}) \leqslant \frac{9}{\gamma_{\min}(\Theta_{\mathbf{S}\mathbf{S}})}\Big) \\&\leqslant \exp\Big(-\frac{\tau^2}{4L^2_{\eta}}\Big) \leqslant \exp\Big( - \frac{\tau^2 n a_{\mathbf{S}}^2 \gamma_{\min}(\Theta_{\mathbf{S}\mathbf{S}})}{12\left\|\mathbf{r}_\mathbf{S}\right\|^2_2  \gamma_{\max}\left(\Theta_{\mathbf{S}^c\mathbf{S}^c|\mathbf{S}}\right)} \Big).
    \end{aligned}
\end{equation}

Thus, let $\mathcal{A}$ be the event $\{X_{\mathbf{S}} \mid \gamma_{\max}(Q_{\mathbf{S}\mathbf{S}}^{-1}) \leqslant \frac{9}{\gamma_{\min}(\Theta_{\mathbf{S}\mathbf{S}})}\}$. We have

\begin{equation*}
    \begin{aligned}
        \mathbb{P}\left(\left(\phi_{\mathbf{S}}^c\right)^{*} [\eta] > \frac{\tau}{2}\mid X_{\mathbf{S}}  \right) &= 
       \mathbb{P}\Big(\left(\phi_{\mathbf{S}}^c\right)^{*} [\eta] > \frac{\tau}{2} \mid X_{\mathbf{S}}, \gamma_{\max}(Q_{\mathbf{S}\mathbf{S}}^{-1}) \leqslant \frac{9}{\gamma_{\min}(\Theta_{\mathbf{S}\mathbf{S}})}\Big) \\ +&  \mathbb{P}\Big(\left(\phi_{\mathbf{S}}^c\right)^{*} [\eta] > \frac{\tau}{2} \mid X_{\mathbf{S}}, \gamma_{\max}(Q_{\mathbf{S}\mathbf{S}}^{-1}) > \frac{9}{\gamma_{\min}(\Theta_{\mathbf{S}\mathbf{S}})}\Big) \\
        & \leqslant \exp\left( - \frac{\tau^2 n a_{\mathbf{S}}^2 \gamma_{\min}(\Theta_{\mathbf{S}\mathbf{S}})}{4\left\|\mathbf{r}_\mathbf{S}\right\|^2_2  \gamma_{\max}\left(\Theta_{\mathbf{S}^c\mathbf{S}^c|\mathbf{S}}\right)} \right)  +   \mathbb{P}\left( \mathcal{A}^c \right) \\
        &\leqslant \exp\left( - \frac{\tau^2 n a_{\mathbf{S}}^2 \gamma_{\min}(\Theta_{\mathbf{S}\mathbf{S}})}{4\left\|\mathbf{r}_\mathbf{S}\right\|^2_2  \gamma_{\max}\left(\Theta_{\mathbf{S}^c\mathbf{S}^c|\mathbf{S}}\right)} \right) + 2\exp(-\frac{n}{2}).
    \end{aligned}
\end{equation*}
\newpage
\noindent\textbf{Condition \eqref{da3}}

Now we are going to study condition \eqref{da3}. Recall that
$q_{\mathbf{S}^c \mid \mathbf{S} } = q_{\mathbf{S}^c} - Q_{\mathbf{S}^c\mathbf{S}}Q_{\mathbf{S}\mathbf{S}}^{-1}q_{\mathbf{S}} $ and $Q_{\mathbf{S}^c \mathbf{S}^c \mid \mathbf{S}}=Q_{\mathbf{S}^c \mathbf{S}^c}-Q_{\mathbf{S}^c \mathbf{S}} Q_{\mathbf{S} \mathbf{S}}^{-1} Q_{\mathbf{S}\mathbf{S}^c}$.
Given $X$, $q_{\mathbf{S}^c \mid \mathbf{S}}$ is a centered Gaussian random vector with covariance matrix
\begin{equation*}
   \begin{aligned}
 \mathbb{E}\left[q_{\mathbf{S}^c \mid \mathbf{S}} q_{\mathbf{S}^c \mid \mathbf{S}}^{\top}\right] &=\mathbb{E}\left[q_{\mathbf{S}^c} q_{\mathbf{S}^c}^{\top}-q_{\mathbf{S}^c} q_{\mathbf{S}}^{\top} Q_{\mathbf{S}\mathbf{S}}^{-1} Q_{\mathbf{SS}^c}-Q_{\mathbf{S}^c \mathbf{S}} Q_{\mathbf{S} \mathbf{S}}^{-1} q_{\mathbf{S}} q_{\mathbf{S}^c}^{\top}+Q_{\mathbf{S}^c \mathbf{S}} Q_{\mathbf{S}\mathbf{S}}^{-1} q_{\mathbf{S}} q_{\mathbf{S}}^{\top} Q_{\mathbf{SS}}^{-1} Q_{\mathbf{S}\mathbf{S}^c}\right] \\& =\mathbb{E}\left[q_{\mathbf{S}^c} q_{\mathbf{S}^c}^{\top}-Q_{\mathbf{S}^c \mathbf{S}} Q_{\mathbf{S}\mathbf{S}}^{-1} q_{\mathbf{S}} q_{\mathbf{S}}^{\top} Q_{\mathbf{SS}}^{-1} Q_{\mathbf{S}\mathbf{S}^c}\right] \\&= \mathbb{E}\left[q_{\mathbf{S}^c} q_{\mathbf{S}^c}^{\top}\right]-\mathbb{E}\left[Q_{\mathbf{S}^c \mathbf{S}} Q_{\mathbf{S}\mathbf{S}}^{-1} q_{\mathbf{S}} q_{\mathbf{S}}^{\top} Q_{\mathbf{SS}}^{-1} Q_{\mathbf{S}\mathbf{S}^c}\right] \\
 &=\frac{\sigma^2}{n}Q_{\mathbf{S}^c \mathbf{S}^c} - \frac{\sigma^2}{n} Q_{\mathbf{S}^c \mathbf{S}} Q_{\mathbf{S} \mathbf{S}}^{-1} Q_{\mathbf{S}\mathbf{S}^c} :=\frac{\sigma^2}{n} Q_{\mathbf{S}^c \mathbf{S}^c \mid \mathbf{S}}.
\end{aligned} 
\end{equation*}

Next, we define $\psi(u) = \left(\phi_{\mathbf{S}}^c\right)^*\left(\sigma n^{-1 / 2} Q_{\mathbf{S}^c \mathbf{S}^c \mid \mathbf{S}}^{1 / 2} u\right)$ so that $\left(\phi_{\mathbf{S}^c}^c\right)^*\left[q_{\mathbf{S}^c \mid \mathbf{S}}\right]$ has the same distribution as $\psi(W)$. Now we want to show that $\psi$ is a Lipschitz function

\begin{equation*}
    \begin{aligned}
        |\psi(u)-\psi(v)| &\leqslant \psi(u-v)   = \left(\phi_{\mathbf{S}}^c\right)^*\left(\sigma n^{-1 / 2} Q_{\mathbf{S}^c \mathbf{S}^c \mid \mathbf{S}}^{1 / 2} (u - v)\right) \\ & \leqslant \sigma n^{-1 / 2} a_{\mathbf{S}^c}^{-1}\left\|Q_{\mathbf{S}^c \mathbf{S}^c \mid \mathbf{S}}^{\frac{1}{2}}(u-v)\right\|_{\infty} \\
        & \leqslant \sigma n^{-1 / 2} a_{\mathbf{S}^c}^{-1}\left\|Q_{\mathbf{S}^c \mathbf{S}^c \mid \mathbf{S}}^{\frac{1}{2}}\right\|_{2,\infty }\left\|(u-v)\right\|_{\infty}
         \\
        & \leqslant \sigma n^{-1 / 2} a_{\mathbf{S}^c}^{-1}\left\|Q_{\mathbf{S}^c \mathbf{S}^c \mid \mathbf{S}}^{\frac{1}{2}}\right\|_{2,\infty }\left\|(u-v)\right\|_{2}
    \end{aligned}
\end{equation*}

Suppose that $\left\|Q_{\mathbf{S}^c \mathbf{S}^c  \mid \mathbf{S}}^{1 / 2}\right\|_{2,\infty} \leqslant 3$, then $\psi$ is a Lipschitz function with  Lipschitz  constant $3\sigma n^{-1 / 2} a_{\mathbf{S}^c}^{-1}$. In addition, if 
$\mathbb{E} [(\phi_{\mathbf{S}}^c)^*(q_{\mathbf{S}^c \mid \mathbf{S}})] \leqslant \frac{\lambda_n \tau}{4} $, then  by Lemma~\ref{lem:masset} , we have for $t = \frac{\lambda_n \tau}{4}$,

\begin{equation*}
    \begin{aligned}
    \mathbb{P}\left(\left(\phi_{\mathbf{S}}^c\right)^*\left[q_{\mathbf{S}^c \mid \mathbf{S}}\right] \geqslant \frac{\lambda_n\tau}{2}  \right) &= \mathbb{P} \left(\psi(W) >\frac{\lambda_n\tau}{2}  \right) = \mathbb{P} \left(\psi(W) - \frac{\lambda_n\tau}{4}  > \frac{\lambda_n\tau}{4}  \right)\\& \leqslant  
    \mathbb{P} \left(\psi(W) -\mathbb{E} [(\phi_{\mathbf{S}}^c)^*(q_{\mathbf{S}^c \mid \mathbf{S}})]  > \frac{\lambda_n\tau}{4}  \right)\\& =
    \mathbb{P} \left(\psi(W) - \mathbb{E}\left[ \psi(W)\right]  > \frac{\lambda_n\tau}{4}  \right) \leqslant  \exp \left(-\frac{\tau^2   \lambda_n^2 na_{\mathbf{S}^c}^2}{144 \sigma^2}\right)  .
    \end{aligned}
\end{equation*}

Now, we consider random $X$. For any  $u_t, u_s$, we have 
\begin{equation*}
\begin{aligned}
     \mathbb{E}\big[(u_t-u_s)^{\top}  q_{\mathbf{S}^c \mid \mathbf{S}} \big]^2 &=\frac{\sigma^2}{n}(u_t-u_s)^{\top} Q_{\mathbf{S}^c \mathbf{S}^c \mid \mathbf{S}}  (u_t-u_s) \leqslant \frac{\sigma^2}{n}\big\|Q_{\mathbf{S}^c \mathbf{S}^c \mid \mathbf{S}} ^{\frac{1}{2}}\big\|_2^2\big\| (u_t-u_s)\big\|_2^2  \\& = \mathbb{E} \big[\sigma n^{-\frac{1}{2}}\|Q_{\mathbf{S}^c \mathbf{S}^c \mid \mathbf{S}} \|_2^{\frac{1}{2}} (u_t-u_s)^{\top} W\big]^2
\end{aligned}
\end{equation*}

By using Sudakov-Fernique inequality, if $\|Q_{\mathbf{S}^c \mathbf{S}^c \mid \mathbf{S}}\|_2 \leqslant 9$, we get
\begin{equation}
    \begin{aligned}
       \mathbb{E} [(\phi_{\mathbf{S}}^c)^*(q_{\mathbf{S}^c \mid \mathbf{S}})] & =\mathbb{E} \sup_{\phi_{\mathbf{S}}^c(u) \leq 1} u^{\top} q_{\mathbf{S}^c \mid \mathbf{S}} \\
       &\leqslant \sigma n^{-1 / 2}\|Q_{\mathbf{S}^c \mathbf{S}^c \mid \mathbf{S}}\|_2^{\frac{1}{2}} \mathbb{E} \sup _{\phi_{\mathbf{S}}^c(u) \leq 1} u^{\top} W \\
       & \leqslant  \sigma n^{-\frac{1}{2}}\|Q_{\mathbf{S}^c \mathbf{S}^c \mid \mathbf{S}}\|_2^{\frac{1}{2}} \mathbb{E}\left[\left(\phi_{\mathbf{S}}^c\right)^*(W)\right]\\
       & \leqslant  3\sigma n^{-\frac{1}{2}}\mathbb{E}\left[\left(\phi_{\mathbf{S}}^c\right)^*(W)\right]\\
       & \leqslant \frac{\lambda_n \tau}{4} .
    \end{aligned}
\end{equation}

On the other hand, Assumption \ref{ass:distribution_normal noise} and \ref{ass:Irrepresentable_condition} imply that
$$\frac{9\sigma^2\mathbb{E}^2\left[\left(\phi_{\mathbf{S}}^c\right)^*(W)\right]}{n} \leqslant \frac{9\sigma^2\log(p-|\mathbf{S}|)}{a^2_{\mathbf{S}^c}n} \leqslant \frac{\lambda_n^2\tau^2}{16}.$$

Therefore, we have 
\begin{equation*}
    \begin{aligned}
    \mathbb{P}\left(\left(\phi_{\mathbf{S}}^c\right)^*\left[q_{\mathbf{S}^c \mid \mathbf{S}}\right]  \geqslant \frac{\lambda_n\tau}{2}  \mid X, \left\|Q_{\mathbf{S}^c \mathbf{S}^c \mid \mathbf{S}}^{1 / 2}\right\|_{2,\infty} \leqslant 3\right) \leqslant  \exp \left(-\frac{\tau^2  n \lambda_n^2a_{\mathbf{S}^c}^2}{144 \sigma^2}\right)  .
    \end{aligned}
\end{equation*}

Let $\mathcal{B}$ be the event $\{X \mid \left\|Q_{\mathbf{S}^c \mathbf{S}^c \mid \mathbf{S}}^{1 / 2}\right\|_{2,\infty} \leqslant 3\}$. We have
\begin{equation*}
    \begin{aligned}        \mathbb{P}\left(\left(\phi_{\mathbf{S}}^c\right)^*\left[q_{\mathbf{S}^c \mid \mathbf{S}}\right]  \geqslant \frac{\lambda_n\tau}{2}  \mid X\right) &=  \mathbb{P}\left(\left(\phi_{\mathbf{S}}^c\right)^*\left[q_{\mathbf{S}^c \mid \mathbf{S}}\right]  \geqslant \frac{\lambda_n\tau}{2}  \mid X, \left\|Q_{\mathbf{S}^c \mathbf{S}^c \mid \mathbf{S}}^{1 / 2}\right\|_{2,\infty} \leqslant 3\right) \\+&  \mathbb{P}\left(\left(\phi_{\mathbf{S}}^c\right)^*\left[q_{\mathbf{S}^c \mid \mathbf{S}}\right]  \geqslant \frac{\lambda_n\tau}{2}  \mid X, \left\|Q_{\mathbf{S}^c \mathbf{S}^c \mid \mathbf{S}}^{1 / 2}\right\|_{2,\infty} > 3\right)\\
        & \leqslant \exp \left(-\frac{\tau^2  n \lambda_n^2a_{\mathbf{S}^c}^2}{144 \sigma^2}\right)+   \mathbb{P}\left( \mathcal{B}^c \right)\\
        &\stackrel{\text{\eqref{da0}}}{\leqslant }  \exp \left(-\frac{\tau^2  n \lambda_n^2a_{\mathbf{S}^c}^2}{144 \sigma^2}\right)  + \exp(-\frac{n}{2}).
    \end{aligned}
\end{equation*}

\noindent\textbf{Condition \eqref{da4}}

The last condition \eqref{da4} lead us to control the term $\mathbb{P}\left(\left\|q_{\mathbf{S}}\right\|_{\infty} \geqslant c'(\mathbf{S},G)\right)$, with
$$
c'(\mathbf{S},G) =  \min\left\{\frac{\gamma_{\min}\left(Q_{\mathbf{S}\mathbf{S}}\right) \beta^*_{\min}}{3 A_{\mathbf{S}}},\frac{\tau  \gamma^{\frac{3}{2}}_{\min}(Q_{\mathbf{S}\mathbf{S}}) a_{\mathbf{S}^c} \beta^*_{\min}}{72 A_{\mathbf{S}} \sum\limits_{g\in \mathsf{G}_{\mathbf{S}}} w_g \sqrt{\left|G_g \cap \mathbf{S}\right|}}\right\}.
$$

For any given $X$, \cite{svsw} showed that for any $\delta > 0$,
 
$$
\mathbb{P}\left(\left\|q_{\mathbf{S}}\right\|_{\infty} \geqslant \delta  \right) \leqslant 2|\mathbf{S}| \exp \left(-\frac{n \delta^2}{2 \sigma^2}\right)  .
$$

Recall under the event $\mathcal{A}$, we have 

$$ \frac{\gamma_{\min}(\Theta_{\mathbf{S}\mathbf{S}})}{9} \leqslant \gamma_{\min}(Q_{\mathbf{S}\mathbf{S}}).$$

Which implies that 
\begin{equation*}
    \begin{aligned}
        c'(\mathbf{S},G) &\geqslant  \min\left\{\frac{\gamma_{\min}\left(\Theta_{\mathbf{S}\mathbf{S}}\right) \beta^*_{\min}}{27 A_{\mathbf{S}}},\frac{\tau  \gamma_{\min}(\Theta_{\mathbf{S}\mathbf{S}})^{\frac{3}{2}} a_{\mathbf{S}^c} \beta^*_{\min}}{648 A_{\mathbf{S}} \sum\limits_{g\in \mathsf{G}_{\mathbf{S}}} w_g \sqrt{\left|G_g \cap \mathbf{S}\right|}}\right\} \\&\geqslant  \min\left\{\frac{\beta^*_{\min}}{27c_1 A_{\mathbf{S}}},\frac{\tau   a_{\mathbf{S}^c} \beta^*_{\min}}{648c_1^{\frac{3}{2}} A_{\mathbf{S}} \sum\limits_{g\in \mathsf{G}_{\mathbf{S}}} w_g \sqrt{\left|G_g \cap \mathbf{S}\right|}}\right\}
        := c(\mathbf{S},G)  .
    \end{aligned}
\end{equation*}

Thus, consider random $X$, we have 
\begin{equation*}
      \mathbb{P}\left(\left\|q_{\mathbf{S}}\right\|_{\infty} \geqslant c'(\mathbf{S},G) \mid \mathcal{A}\right) \leqslant \mathbb{P}\left(\left\|q_{\mathbf{S}}\right\|_{\infty} \geqslant c(\mathbf{S},G) \mid \mathcal{A} \right) \leqslant 2|\mathbf{S}| \exp \left(-\frac{n c^2(\mathbf{S},G)}{2 \sigma^2}\right)
\end{equation*}

Thus,

\begin{equation*}
    \begin{aligned}
        \mathbb{P}\left(\left\|q_{\mathbf{S}}\right\|_{\infty} \geqslant c'(\mathbf{S},G) \right) &= 
        \mathbb{P}\left(\left\|q_{\mathbf{S}}\right\|_{\infty} \geqslant c'(\mathbf{S},G)  \cap \mathcal{A} \right) +   \mathbb{P}\left(\left\|q_{\mathbf{S}}\right\|_{\infty} \geqslant c'(\mathbf{S},G) \cap \mathcal{A}^c \right) \\
        & \leqslant \mathbb{P}\left(\left\|q_{\mathbf{S}}\right\|_{\infty} \geqslant c'(\mathbf{S},G)  \cap \mathcal{A} \right) +   \mathbb{P}\left( \mathcal{A}^c \right) \\
        & = \mathbb{P}\left(\left\|q_{\mathbf{S}}\right\|_{\infty} \geqslant c'(\mathbf{S},G) \mid \mathcal{A} \right)\mathbb{P}\left( \mathcal{A} \right) +   \mathbb{P}\left( \mathcal{A}^c \right) \\
        & \leqslant \mathbb{P}\left(\left\|q_{\mathbf{S}}\right\|_{\infty} \geqslant c'(\mathbf{S},G)  \mid \mathcal{A} \right) +   \mathbb{P}\left( \mathcal{A}^c \right)\\
        &\leqslant 2|\mathbf{S}| \exp \left(-\frac{n c^2(\mathbf{S},G)}{2 \sigma^2}\right) + 2\exp(-n/2).
    \end{aligned}
\end{equation*}

In summary, the probability of one of the conditions being violated is upper bound by 
$$8\exp(-\frac{n}{2}) +  \exp\left( - \frac{n a_{\mathbf{S}}^2\tau^2 \gamma_{\max}(\Theta_{\mathbf{S}\mathbf{S}})}{4\left\|\mathbf{r}_\mathbf{S}\right\|^2_2  \gamma_{\max}\left(\Theta_{\mathbf{S}^c\mathbf{S}^c|\mathbf{S}}\right)} \right)
   +\exp \left(-\frac{n \lambda_n^2\tau^2 a_{\mathbf{S}^c}^2}{32 \sigma^2c_2^4}\right) + 2|\mathbf{S}| \exp \left(-\frac{n c^2(\mathbf{S},G)}{2 \sigma^2}\right).$$

\quad

\subsubsection{Part V}

First, given the original group structure $G$ and its induced counterpart $\Gcal$, along with their respective weights  $w$ and $\wcal$, we consider the scenario where  $\mathbf{J} = \mathbf{S}$. For all $ \beta \in \mathbb{R}^{p}$, we have
\begin{equation}
\label{eq:ineqnorm1}
    \begin{aligned}
        \phi^G_{\mathbf{S}}(\beta_{\mathbf{S}}) & = \sum_{g \in \mathsf{G}^G_{\mathbf{S}}} w_g \|\beta_{\mathbf{S} \cap G_g  }\|_2 \leqslant
\sum_{g \in \mathsf{G}_{\mathbf{S}}} w_g  \big(\sum\limits_{\gcal : \gcal \in  F^{-1}(g), \Gcal_\gcal \subset \mathbf{S}  }\|\beta_{\mathbf{S} \cap \Gcal_\gcal  }\|_2\big) \\& =
    \sum\limits_{\gcal :  \Gcal_\gcal \subset \mathbf{S}  } \big( \sum\limits_{g : g \in  F(\gcal),g \in \mathsf{G}_{\mathbf{S}} } w_g \big)\|\beta_{\mathbf{S} \cap \Gcal_\gcal  }\|_2 \\& =
       \sum\limits_{\gcal :  \Gcal_\gcal \subset \mathbf{S}  } \big( \sum\limits_{g : g \in  F(\gcal) } w_g  \big)\|\beta_{\mathbf{S} \cap \Gcal_\gcal  }\|_2 \\& =
       \sum\limits_{\gcal \in \mathsf{G}^\Gcal_{\mathbf{S}} }  w_\gcal \|\beta_{\mathbf{S} \cap \Gcal_\gcal  }\|_2 = \phi^\Gcal_{\mathbf{S}}(\beta).
    \end{aligned}
\end{equation}

Since  $\phi^G_{\mathbf{S}}(\beta) \leqslant \phi^\Gcal_{\mathbf{S}}(\beta)$, we can set 
$  a^\Gcal_{\mathbf{S}} = a^G_{\mathbf{S}} = \min\limits_{g \in \mathsf{G}^G_{\mathbf{S}}}\frac{w_g}{\sqrt{d_g}}$.  Since
$$  \max\limits_{\gcal \in \mathsf{G}^\Gcal_{\mathbf{S}} }w_\gcal = \max\limits_{\gcal: \Gcal_\gcal \cap \mathbf{S} \neq \emptyset }\sum_{g\in F(\mathcal{g}) }w_g \leqslant h_{\max}(\mathbf{G_{S}})\max\limits_{g \in \mathsf{G}^G_{\mathbf{S}}}w_g,$$
we can set $ A^\Gcal_{\mathbf{S}} =  A^G_{\mathbf{S}}$. On the other hand, for all $ \beta \in \mathbb{R}^{p}$, we have

\begin{equation}
\label{eq:ineqnorm2}
    \begin{aligned}
        (\phi^G_{\mathbf{S}})^c(\beta_{\mathbf{S}}^c) & = \sum_{g \in [m] \setminus \mathsf{G}^G_{\mathbf{S}}} w_g \|\beta_{\mathbf{S}^c \cap G_g  }\|_2 \leqslant
\sum_{g \in [m] \setminus \mathsf{G}_{\mathbf{S}}} w_g  \big(\sum\limits_{\gcal : \gcal \in  F^{-1}(g), \Gcal_\gcal \subset \mathbf{S}^c  }\|\beta_{\mathbf{S}^c \cap \Gcal_\gcal  }\|_2\big) \\& =
    \sum\limits_{\gcal :  \Gcal_\gcal \subset \mathbf{S}^c  } \big( \sum\limits_{g : g \in  F(\gcal),g \in [m] \setminus \mathsf{G}^G_{\mathbf{S}} } w_g \big)\|\beta_{\mathbf{S}^c \cap \Gcal_\gcal  }\|_2 \\& =
       \sum\limits_{\gcal :  \Gcal_\gcal \subset \mathbf{S}^c  } \big( \sum\limits_{g : g \in  F(\gcal) } w_g  \big)\|\beta_{\mathbf{S}^c \cap \Gcal_\gcal  }\|_2 \\& =
       \sum\limits_{\gcal \in  [m] \setminus \mathsf{G}^\Gcal_{\mathbf{S}} }  w_\gcal \|\beta_{\mathbf{S}^c \cap \Gcal_\gcal  }\|_2 = (\phi^\Gcal_{\mathbf{S}})^c(\beta).
    \end{aligned}
\end{equation}

Consequently, with an trivial extension, we can set $  a^\Gcal_{\mathbf{S}^c} = a^G_{\mathbf{S}^c} \leqslant\min\limits_{g \in \mathsf{G}^G_{\mathbf{S^c}}}w_g/\sqrt{d_g}$.

Based on the result of Theorem~\ref{pattern}.\ref{the6part1}, Equation \eqref{pattern2}
holds if
    \begin{equation*}
       \lambda_n|\mathbf{S}|^{\frac{1}{2}} \lesssim \min\Big\{\frac{\beta_{\min}^* }{A_{\mathbf{S}}},\frac{\beta_{\min}^*a_{\mathbf{S^c}} }{A_{\mathbf{S}}\sum\limits_{\gcal \in \mathsf{\Gcal}_{\mathbf{S}}} w_\gcal \sqrt{\left|\Gcal_\gcal \cap \mathbf{S}\right|}}\Big\}.
\end{equation*}

By the Cauchy–Schwarz inequality, we have
\begin{equation*}
\begin{aligned}
     \sum\limits_{g\in \mathsf{G}_{\mathbf{S}}} w_g \sqrt{\left|G_g \cap \mathbf{S}\right|} &\leqslant  \sum\limits_{g\in \mathsf{G}_{\mathbf{S}}} w_g \sum\limits_{{\gcal \in F^{-1}(g) } }\sqrt{\left|\Gcal_\gcal \cap \mathbf{S}\right|}  \\&= \sum\limits_{\gcal \in F^{-1}(g), g\in \mathsf{G}_{\mathbf{S}}} \sqrt{\left|\Gcal_\gcal \cap \mathbf{S}\right|}  \big(\sum\limits_{g\in F(\mathcal{\gcal}) }w_g\big)
     \\&= \sum\limits_{\gcal \in \mathsf{\Gcal}_{\mathbf{S}}} w_\gcal \sqrt{\left|\Gcal_\gcal \cap \mathbf{S}\right|} 
\end{aligned}
\end{equation*}

If $ F^{-1}(g) = O(1)$ for every $g \in \mathsf{G}_{\mathbf{S}}$, we have
$$|G_g \cap \mathbf{S}| =  \sum\limits_{{\gcal \in F^{-1}(g) } }|\Gcal_\gcal \cap \mathbf{S}| \asymp  \big( \sum\limits_{{\gcal \in F^{-1}(g) } }\sqrt{|\Gcal_\gcal \cap \mathbf{S}|}\big)^2.$$
Consequent, we have $\sqrt{|G_g \cap \mathbf{S}|} \asymp \sum\limits_{{\gcal \in F^{-1}(g) } }\sqrt{|\Gcal_\gcal \cap \mathbf{S}|} $,
$$\sum\limits_{g\in \mathsf{G}_{\mathbf{S}}} w_g \sqrt{\left|G_g \cap \mathbf{S}\right|} \asymp  \sum\limits_{\gcal \in \mathsf{\Gcal}_{\mathbf{S}}} w_\gcal \sqrt{\left|\Gcal_\gcal \cap \mathbf{S}\right|},$$
and
$$\min\bigg\{\frac{\beta_{\min}^* }{A^G_{\mathbf{S}}},\frac{\beta_{\min}^*a^G_{\mathbf{S^c}} }{A^G_{\mathbf{S}}\sum\limits_{g\in \mathsf{G}_{\mathbf{S}}} w_g \sqrt{\left|G_g \cap \mathbf{S}\right|}}\bigg\} \asymp \min\bigg\{\frac{\beta_{\min}^* }{A^\Gcal_{\mathbf{S}}},\frac{\beta_{\min}^*a^\Gcal_{\mathbf{S^c}} }{A^\Gcal_{\mathbf{S}}\sum\limits_{\gcal \in \mathsf{\Gcal}_{\mathbf{S}}} w_\gcal \sqrt{\left|\Gcal_\gcal \cap \mathbf{S}\right|}}\bigg\}.$$

\end{document}